%% file: RandomizedQ_arxiv.tex
\documentclass[10pt,english]{article}
\usepackage[margin=1in]{geometry} 
\usepackage[T1]{fontenc}
\usepackage[utf8]{inputenc}
\usepackage{hyperref}       % hyperlinks
\hypersetup{colorlinks,citecolor=blue,filecolor=blue,linkcolor=blue,urlcolor=blue}

\usepackage{url}            % simple URL typesetting
\usepackage{booktabs}       % professional-quality tables
\usepackage{amsfonts,mathtools}       
\usepackage{nicefrac}       
\usepackage{microtype}      
\usepackage{lipsum}		
\usepackage{graphicx}
\usepackage{subcaption}
\usepackage{natbib}
\usepackage{doi}
\usepackage{algorithm}
\usepackage{algorithmic}
\usepackage{amsmath,amsthm,bbm} 
\usepackage{xcolor,colortbl}
\usepackage{xspace}
\usepackage{enumitem}
\usepackage{makecell}
\usepackage{tikz}
\usepackage{amsthm,amsmath}
\usepackage{natbib} % in the preamble

\definecolor{hew}{RGB}{0,47,167}

\definecolor{yc}{RGB}{200,47,0}

\definecolor{modify}{RGB}{0,47,167}

\usepackage{amssymb}       % for \checkmark
\usepackage{pifont}        % for \ding{55}
\usepackage{lineno}

 \makeatletter
 \input{macro}

 \newcommand{\multiline}[1]{%
  \begin{tabularx}{\dimexpr\linewidth-\ALG@thistlm}[t]{@{}X@{}}
    #1
  \end{tabularx}
}

\title{Provably Efficient and Agile Randomized Q-Learning}

\author{
    He Wang\thanks{Department of Electrical and Computer Engineering, Carnegie Mellon University. Email: \texttt{\{hew2,xingyuxu\}@andrew.cmu.edu}. }\\
    CMU\\
    \and
    Xingyu Xu\footnotemark[1] \\
    CMU\\
    \and
    Yuejie Chi\thanks{Department of Statistics and Data Science, Yale University. Email: \texttt{yuejie.chi@yale.edu}. } \\ 	 
     Yale   \\
    } 

\date{\today}

\begin{document}

\maketitle
\begin{abstract}
While Bayesian-based exploration often demonstrates superior empirical performance compared to bonus-based methods in model-based reinforcement learning (RL), its theoretical understanding remains limited for model-free settings. Existing provable algorithms either suffer from computational intractability or rely on stage-wise policy updates which reduce responsiveness and slow down the learning process. In this paper, we propose a novel variant of Q-learning algorithm, referred to as \myalg, which integrates \textit{sampling-based exploration with agile, step-wise, policy updates}, for episodic tabular RL. We establish a sublinear regret bound $\widetilde{O}(\sqrt{H^5SAT})$, where $S$ is the number of states, $A$ is the number of actions, $H$ is the episode length, and $T$ is the total number of episodes. In addition, we present a logarithmic regret bound $ O\left(\frac{H^6SA}{\Delta_{\min}}\log^5(SAHT)\right)$ when the optimal Q-function has a positive sub-optimality $\Delta_{\min}$. Empirically, \myalg exhibits outstanding performance compared to existing Q-learning variants with both bonus-based and Bayesian-based exploration on standard benchmarks. 
\end{abstract}


\allowdisplaybreaks
\setcounter{tocdepth}{2}
\tableofcontents

\input{section/intro.tex}

\input{section/problem_formulation.tex}

\input{section/algorithm.tex}

\input{section/theoretical_analysis.tex}

\input{section/simulation.tex}

\section{Conclusion}
 
In this work, we study the performance of Q-learning without exploration bonuses for episodic tabular MDPs in the online setting. We identify two key challenges in existing approaches: the additional statistical dependency introduced by randomizing learning rates, and the inefficiency of slow, stage-wise policy updates, as the bottlenecks of theoretical analysis and algorithm design. To address these challenges, we develop a novel randomized Q-learning algorithm with agile updates called \myalg, which efficiently adapts the policy to newly observed data. Theoretically, we establish a sublinear worst-case and a logarithmic gap-dependent regret bounds. Empirically, our experiments show that \myalg significantly outperforms than baseline algorithms in terms of total regret, due to the effective exploration and agile updates.

There are several promising directions for future research. For example, extending our analysis to function approximation settings—such as linear or neural representations—would significantly broaden the applicability of \myalg\citep{melo2007q,song2019efficient}. In addition, incorporating variance reduction techniques could further improve the regret bounds and potentially match the theoretical lower bounds \citep{zhang2020almost,zheng2024gap}.

\section*{Acknowledgement}
This work is supported in part by grants ONR N00014-19-1-2404, NSF CCF-2106778,  CNS-2148212, AFRL FA8750-20-2-0504, and by funds from federal agency and industry partners as specified in the Resilient \& Intelligent NextG Systems (RINGS) program.

\bibliographystyle{apalike}
\bibliography{references,bibfileRL}
\appendix
\input{appendix/notation.tex}
\input{appendix/Q_reformulation.tex}

\input{appendix/concentration.tex}

\input{appendix/proof_thm1.tex}
\input{appendix/anytime.tex}

\input{appendix/log_regret.tex}

\end{document}

%% file: macro.tex
\newcommand{\cf}{\text{cf.}\xspace}
\newcommand{\ie}{\text{i.e.}}

\newcommand{\Var}{\operatorname{Var}}

\renewcommand{\S}{\mathcal{S}}
\newcommand{\A}{\mathcal{A}}

\newcommand{\mE}{\mathbb{E}}

\newcommand{\M}{\mathcal{M}}
\newcommand{\F}{\mathcal{F}}
\newcommand{\E}{\mathbb{E}}
\newcommand{\B}{\operatorname{Beta}}
\newcommand{\bP}{\mathbb{P}}

\newcommand{\1}{\mathbbm{1}}

\newcommand{\defeq}{\vcentcolon=}
\newcommand{\resp}{\text{resp.}}

\newcommand{\myalg}{\textsc{RandomizedQ}\xspace}
\newcommand{\myalgany}{\textsc{RandomizedQ-Anytime}\xspace}

\newtheorem{lemma}{\textbf{Lemma}}

\newtheorem{proposition}{\textbf{Proposition}}
\newtheorem{theorem}{\textbf{Theorem}} 
\newtheorem{definition}{Definition}
\newtheorem{remark}{Remark}

\newtheorem{assumption}{Assumption}

\newcommand{\tildeQ}{\widetilde{Q}}
\newcommand{\tildeV}{\widetilde{V}}
\newcommand{\aux}{\flat}
\newcommand{\clip}{\operatorname{clip}}

%% file: section/intro.tex
\section{Introduction}
In reinforcement learning (RL) \citep{sutton1988learning}, an agent aims to learn an optimal policy that maximizes its cumulative rewards through interactions with an unknown environment. 
Broadly speaking, RL algorithms can be categorized into two main approaches---model-based and model-free methods---depending on whether they first learn a model of the environment and plan within it, or directly learn the optimal policy from experience. While model-based approaches offer advantages in sample efficiency, model-free algorithms tend to be more computationally efficient and take lower space complexity, making them more attractive for deployment in many real-world applications, such as games \citep{mnih2015human}, robotics control \citep{tang2025deep} and language model training \citep{hong2024q}.

As one of fundamental challenges in RL, the \textit{exploitation-exploration dilemma} remains particularly difficult to address in the model-free paradigm, i.e., the learned policy needs to carefully balance between exploiting current observations and exploring unseen state-action pairs to maximize total rewards in the long term. To manage the trade-off, most provably efficient model-free algorithms adopt the principle of \textit{optimism in the face of uncertainty}, incentivizing exploration by assigning bonuses to uncertain outcomes, constructed from their upper confidence bound (UCB) \citep{lai1985asymptotically}. In particular, prior works \citep{jin2018q,zhang2020almost,li2021breaking} showed that Q-learning augmented with tailored bonus functions  achieve comparable sample complexity to their model-based counterparts.

In contrast to bonus-based exploration methods aforementioned, Bayesian-based approaches have gained increasing attention for their superior empirical performance \citep{osband2016deep,fortunato2017noisy}. These approaches enhance efficient exploration by leveraging the inherent randomness in sampling from posteriors that are updated based on prior observations. However, theoretical understandings have been limited, where the majority of prior work has focused on model-based RL \citep{osband2013more,agrawal2017optimistic,tiapkin2022optimistic}.  

When it comes to model-free RL, research is even more limited in several aspects. \citet{dann2021provably} proposed a \textit{sample-efficient} algorithm that draws Q-functions directly from the posterior distribution. Nevertheless, this approach suffers from \textit{computational inefficiency}. More recently, \citet{tiapkin2024model} introduced posterior sampling via randomized learning rates, but unfortunately they only provided theoretical guarantees\footnote{A careful examination of their proof reveals a critical technical gap in their analysis. We provide a novel fix with substantial new analyses, which fortunately preserves their claimed theoretical guarantee. We discuss this in more detail in Section~\ref{sec:gap_independent_bound}.} for \textit{stage-wise policy updates}, which are known to be inefficient in practice as this staging approach does not allow agents to respond agilely to the environment.  
To this end, it is natural to ask:
\begin{center}
    \textit{Is it possible to design a model-free RL algorithm with Bayesian-based exploration, achieving \textbf{sample efficiency}, \textbf{computational efficiency}, and \textbf{agile policy updates}?}  
\end{center}   

\subsection{Main contribution}
To answer this question, we focus on learning a near-optimal policy through sampling-based Q-learning, in a provably sample- and computation-efficient manner. As in \citet{jin2018q,dann2021provably,tiapkin2024model}, throughout this paper, we consider tabular, finite-horizon Markov Decision Processes (MDPs) in the online setting. Below we summarize the highlights of this work:
\begin{itemize}%[leftmargin=1em, topsep=0em, itemsep=0.1em]
    \item We propose \myalg, a sampling-based Q-learning algorithm which leverages tailored randomized learning rates to enable both efficient exploration and agile policy updates.
    \item We establish a gap-independent regret bound on the order of $\widetilde{O}(\sqrt{H^5SAT})$, where $S$ is the number of states, $A$ is the number of actions, $H$ is the episode length, and $T$ is the number of episodes.
    \item Under a strictly positive sub-optimality gap $\Delta_{\min}$ of the optimal Q-function, we further prove a logarithmic regret bound of $O\left(H^6SA\log^5(SAHT)/\Delta_{\min}\right)$. To the best of our knowledge, this is the first result showing model-free algorithms can achieve logarithmic regret via sampling-based exploration.
    \item Empirically, \myalg consistently outperforms existing bonus-based and sampling-based model-free algorithms on standard exploration benchmarks, validating its efficacy.
\end{itemize} 
A detailed comparison with pertinent  works is provided in Table \ref{tab:comparison}.
    {
    \renewcommand{\arraystretch}{1.5}  
    \begin{table*}[t]
        \centering
        \resizebox{\textwidth}{!}{
        \begin{tabular}{|c|c|c|c|c|}
        \hline
        \rowcolor[HTML]{E6F7FF} 
        \textbf{Key Property} & 
        \makecell{\textbf{Conditional-PS} \\ \citep{dann2021provably}} & 
        \makecell{\textbf{Staged-RandQL} \\ \citep{tiapkin2024model}} & 
        \makecell{\textbf{RandQL} \\ \citep{tiapkin2024model}} & 
        \makecell{\textbf{\myalg} \\ (This Work) }  \\ \hline
    
        Computational tractability & \ding{55} & \checkmark & \checkmark & \checkmark \\ \hline
        Agile policy update & \checkmark & \ding{55} & \checkmark & \checkmark \\ \hline
        Gap-independent regret guarantee & \checkmark & \checkmark & \ding{55} & \checkmark \\ \hline
        Gap-dependent regret guarantee & \ding{55} & \ding{55} & \ding{55} & \checkmark \\ \hline
        \end{tabular}}
        %\vspace{0.1in}
        \caption{Comparison with the most relevant model-free RL methods with Bayesian-based exploration in tabular settings. A \checkmark\ indicates the method possesses the corresponding property, while a \ding{55} denotes its absence. We identify and fix a technical gap in \citet{tiapkin2024model}, which preserves the gap-independent regret guarantee of Staged-RandQL. Notably, our method uniquely achieves \emph{computational tractability}, \emph{agile policy updates}, and \emph{provable regret guarantees}, distinguishing it from prior work.}  
        \label{tab:comparison}
 
    \end{table*}
    }

\subsection{Related works}
In this section, we discuss closely-related prior works on optimistic Q-learning and online RL with Bayesian-based exploration, focusing on the tabular setting.
\paragraph{Q-learning with bonus-based exploration.} Q-learning and its variants \citep{watkins1989learning,strehl2006pac,mnih2013playing} are among the most widely studied model-free RL algorithms. To understand its theoretical guarantees, several works have equipped Q-learning with UCB bonuses derived from the principle of optimism in the face of uncertainty \citep{jin2018q,bai2019provably,zhang2020almost,wei2020model,li2021breaking,yang2021q,zheng2024gap}. Notably, \citet{jin2018q} first introduced UCB-Q, which augments Q-learning with Hoeffding-type or Bernstein-type bonuses and established a nearly optimal regret bound. Building upon this, \citet{zhang2020almost} proposed a variance-reduced version of UCB-Q, achieving an optimal sample complexity, and \citet{li2021breaking} further improved the performance by reducing the burn-in cost.  

In addition to the worst-case regret bound, gap-dependent regret bounds often leverage benign properties of the environment and enjoy logarithmic regret bounds \citep{yang2021q,zheng2024gap} For instance, \citet{yang2021q} showed that UCB-Q has a logarithmic regret bound under the positive sub-optimality gap assumption, and \citet{zheng2024gap} incorporated error decomposition to establish a gap-dependent bound for Q-learning with variance reduction techniques \citep{zhang2020almost,li2021breaking}.

\paragraph{Model-based RL with Bayesian-based exploration.} 
Extensive works have investigated the theoretical and empirical performance of Bayesian-based exploration in model-based RL.  One popular approach is posterior sampling for reinforcement learning \citep{strens2000bayesian,osband2013more,agrawal2017optimistic,zhang2022feel,hao2022regret,moradipari2023improved}, where the policy is iteratively learned by sampling a model from its posterior distribution over MDP models. The approach has been shown to achieve the optimal regret bound when UCB on Q-functions are also incorporated \citep{tiapkin2022optimistic}. In addition, several works \citep{osband2016generalization,agrawal2021improved,zanette2020frequentist,ishfaq2024more} have investigated posterior sampling with linear function approximation.
 
\paragraph{Model-free RL with Bayesian-based exploration.} To overcome the computational inefficiency of model-based methods, several algorithms have been developed in model-free RL with Bayesian-based exploration, showing promising empirical results \citep{osband2016deep,osband2018randomized,fortunato2017noisy} but lacking theoretical guarantees. Recently, \citet{dann2021provably} sampled Q-functions directly from the posterior, but such an approach is computationally intractable. To address this, \citet{tiapkin2024model} introduced RandQL, the first tractable model-free posterior sampling-based algorithm, which encourages exploration through using randomized learning rates and achieves a regret bound of $\widetilde{O}(\sqrt{SAH^5T})$ when RandQL is staged. However, the slow policy update empirically leads to a significantly degraded performance. This leaves a gap between the theoretical efficiency and practical performance in model-free RL with Bayesian-based exploration.

\paragraph{Notation.} Throughout this paper, we define $\Delta(\mathcal{S})$ as the probability simplex over a set $\mathcal{S}$, and use $[H] \defeq {1, \ldots, H}$ and $[T] \defeq {1, \ldots, T}$ for positive integers $H, T > 0$. We denote $\1$ as the indicator function, which equals $1$ if the specified condition holds and $0$ otherwise. For any set $\mathcal{D}$, we write $|\mathcal{D}|$ to represent its cardinality (i.e., the number of elements in $\mathcal{D}$).  The beta distribution with parameters $\alpha$ and $\beta$ is denoted by $\mathrm{Beta}(\alpha, \beta)$. Finally, we use the notations $\widetilde{O}(\cdot)$ and $O(\cdot)$ to describe the order-wise non-asymptotic behavior, where the former omits logarithmic factors.

%% file: section/problem_formulation.tex
\section{Problem Setup}
\paragraph{Finite-horizon MDPs.} Consider a {tabular finite-horizon MDP} $\M(\S,\A,\{P_h\}_{h=1}^H,\{r_h\}_{h=1}^H,H)$, where $\S$ is the finite state space of cardinality $S$, $\A$ is the action space of cardinality $A$, $P_h: \S\times\A\to\Delta(\S)$ is the transition kernel and $r_h:\S\times\A\to[0,1]$ is the reward function at time step $h\in[H]$, and $H$ is the number of steps within each episode. In each episode, the agent starts from an initial state $s_1\in\S$ and then interacts with the environment for $H$ steps. In each step $h\in[H]$, the agent observes the current state $s_h\in \S$, selects an action $a_h\in\A$, receives a reward $r_h(s_h,a_h)$, and transitions to the next state $s_{h+1}\sim P_h(\ \cdot\ |s_h,a_h)$.

\paragraph{Policy, value function and Q-function.} We denote $\pi = \{\pi_h\}_{h=1}^H$ as the \textit{policy} of the agent within an episode of $H$ steps, where each $\pi_h:\S\to\Delta(\A)$ specifies the action selection probability over the action space $\A$ at the step $h\in[H]$. Given any finite-horizon MDP $\M$, we use the value function $V_{h}^{\pi}$ (\resp Q-function) to denote the expected accumulative rewards starting from the state $s$ (\resp the state-action pair $(s,a)$) at step $h$ and following the policy $\pi$ until the end of the episode: for any $(h,s,a)\in[H]\times\S\times\A$,
\begin{align*}
     &V_{h}^{\pi}(s) = \E_{\pi} \left[\sum_{h'=h}^H r_{h'}(s_{h'},a_{h'})|s_h = s\right], \\
     &Q_{h}^{\pi}(s,a) = \E_{\pi} \left[\sum_{h'=h}^H r_{h'}(s_{h'},a_{h'})|s_h = s, a_h=a\right].
\end{align*}
By convention, we set $V_{H+1}^{\pi}(s) = 0$ and $Q_{H+1}^{\pi}(s,a) = 0$ for any $(s,a)\in\S\times\A$ and policy $\pi$. In addition, we denote $\pi^{\star} = \{\pi_h^{\star}\}_{h=1}^H$ as the \textit{deterministic optimal policy}, which maximizes the value function (\resp Q-function) for  all states (\resp~state-action pairs) among all possible policies, i.e.
\begin{equation} \label{eq:value_func_optimal}
\begin{aligned}
    & V_{h}^{\star}(s)\defeq V_{h}^{\pi^\star}(s) = \max_{\pi} V_{h}^{\pi}(s),\\
    & Q_{h}^{\star}(s,a)\defeq Q_{h}^{\pi^\star}(s,a) = \max_{\pi} Q_{h}^{\pi}(s,a),
\end{aligned}
\end{equation}
where the existence of the optimal policy is well-established \citep{puterman2014markov}.

\paragraph{Bellman equations.}
As the pivotal property of MDPs, the value function and Q-function satisfy the following Bellman consistency equations: for any policy $\pi$ and any $(h,s,a)\in[H]\times\S\times\A$,
\begin{align} \label{eq:bellman_consistency}
    Q_{h}^{\pi}(s,a)= r_h(s,a)+P_{h,s,a}V^{\pi}_{h+1},  
\end{align}
where we use $P_{h,s,a} \defeq P(\ \cdot\ |s,a)\in[0,1]^{1\times S}$ to represent the transition probability (row) vector for the state-action pair $(s,a)$ at $h$-th step. Similarly, we also have the following Bellman optimality equation regarding the optimal policy $\pi^{\star}$:
\begin{equation}\label{eq:bellman_optimality}
     Q_{h}^{\star}(s,a) = r_h(s,a) + P_{h,s,a}V_{h+1}^{\star}.
\end{equation}

\paragraph{Learning goal.} 
In this work, our goal is to learn a policy that minimizes the total regret during $T$ episodes, defined as  
\begin{equation}\label{eq:regret}
    \text{Regret}_T =  \sum_{t=1}^T \left(V_1^{\star}(s_1) - V_1^{\pi^t}(s_1) \right),
\end{equation}
in a computationally efficient and scalable fashion. Here, $\pi^t$ denotes the learned policy executed in the $t$-th episode, for every $t\in[T]$.

%% file: section/algorithm.tex
\section{Efficient and Agile Randomized Q-learning}

In this section, we introduce a sampling-based variant of Q-learning, referred to as \myalg, which ensures the policy to be updated in an agile manner and enhances efficient exploration based on random sampling rather than adding explicit bonuses.
 
\subsection{Motivation}\label{sec:motivation}
Before describing the proposed algorithm in detail, we first review the critical role of learning rates in Q-learning to achieve polynomial sample complexity guarantees, and why such choices cannot be directly extended to the context of Bayesian-based exploration.

\paragraph{The effect of learning rates in Q-learning with UCB bonuses.}
Upon observing a sample transition $(s_h,a_h,s_{h+1})$ at the $h$-th step, the celebrated UCB-Q algorithm \citep{jin2018q}  updates the corresponding entry of the Q-values as:
\begin{align*}
    Q_h(s_h,a_h) \gets &(1-w_m)Q_h(s_h,a_h) + w_{m}\left[r_h(s_h,a_h)+V_{h+1}(s_{h+1})+b_m\right],
\end{align*}
where $m$ is the number of visits to the state-action pair $(s_h,a_h)$ at the $h$-th step,  $w_m = \frac{H+1}{H+m}$ is the learning rate and $b_m$ is the UCB-style bonus term to drive efficient exploration. As detailed in \cite{jin2018q}, such a learning rate of $O(H/m)$ is essential to ensure that the first earlier observations have negligible influence on the most recent Q-value updates. 
 
\paragraph{Challenges in randomized Q-Learning. } In the absence of bonus terms, the exploration is guided by assigning higher weights to important states and leveraging the inherent randomness of sampling from the posterior. As directly sampling Q-functions is computationally intractable~\citep{dann2021provably}, recent work \citep{tiapkin2024model} encourages exploration by randomizing the learning rates according to Beta distribution with an expected order of $O(1/m)$. However, such a learning rate treats all episodes as equally informative, which can result in high bias and an exponential dependency of the sample complexity on the horizon $H$. To overcome the resulting exponential dependence on the horizon $H$, \citet{tiapkin2024model} resort to split the learning process, update the policy at exponentially slower frequencies, and reset the temporary Q values to ensure enough optimism. While this strategy mitigates the sample complexity issue, it suffers from practical inefficiencies due to discarding valuable data across stages, and is unsuitable to deploy in time-varying environments. This inefficiency is empirically demonstrated in \citet[Appendix I]{tiapkin2024model} and Section \ref{section:experiment}.

Thus, it naturally raises the question: can we simply randomize learning rates with an expected order of $O(H/m)$, as used in UCB-Q \citep{jin2018q}? Unfortunately, randomizing learning rates with an expected magnitude of $O(H/m)$ rapidly forgets the earlier episodes that includes the initialization and thus fails to maintain sufficient optimism.

\subsection{Algorithm description}
Motivated by these limitations, we propose an agile, bonus-free Q-learning algorithm for episodic tabular RL in the online setting, referred to as \myalg and summarized in Algorithm \ref{alg:samplingq}. 

The main idea behind \myalg is to update the policy based on an \textit{optimistic mixture of two Q-function ensembles}---each trained with a tailored distribution of learning rates---to balance between agile exploitation of recent observations and sufficiently optimistic exploration. Specifically, after initializing the counters, value and Q-functions (cf. Line \ref{line:init} in Algorithm~\ref{alg:samplingq}), if at the step $h$ of episode $t$, the current state is $s_h\in\S$, the action $a_h$ is selected greedily with respect to the \textit{current} policy Q-function $Q_h^t(s_h,\cdot)$ (cf. Line \ref{line:action} in Algorithm \ref{alg:samplingq}). Upon observing the next state $s_{h+1}$, the updates can be boiled down to the following key components.

\begin{algorithm}[!t]
    \caption{\myalg}\label{alg:samplingq}
    \begin{algorithmic}[1]
    \REQUIRE Initial state $s_1$, optimistically-initial value $\{V^0_h\}\!$, inflation coefficient $\kappa,\kappa^{\aux}\!>\!0$, ensemble size $J$, the number of prior transitions $n_0,n_0^{\aux}\!>\!0$,  and mixing rates $\{\eta_{t,h}\}$.
    \STATE \textbf{Initialize:} $n_h(s, a), n_h^{\aux}(s,a), q_h(s,a) \gets 0$; $\tildeV_h(s), \tildeV^{\aux}_h(s) \gets V^0_h$; $\widetilde{Q}^{j}_{h}(s, a), \widetilde{Q}^{\aux}_{h}(s, a),\allowbreak\widetilde{Q}^{\aux,j}_{h}(s, a) \gets r_h(s, a) + V^0_{h+1}$,  for any $(j,h,s,a)\in [J]\times[H]\times\S\times\A$.\label{line:init}
    \FOR{$t \in [T]$}{
        \FOR{$h = 1,\ldots, H$}{
            \STATE Play $a_{h} = \arg\max_{a\in\A} Q_h(s_h,a)$ and observe the next state $s_{h+1} \sim P_h(\ \cdot\ \vert s_{h}, a_{h})$.\label{line:action}
            \STATE Set $m\gets n_h(s_{h}, a_{h})$ and $m^{\aux}\!\gets\! n_h^{\aux}(s_h,a_h)$.\label{line:visit_counter}
            \STATE \textcolor{blue}{\texttt{/* Update temporary Q-ensembles via randomized learning rates. \hfill*/}}
            \FOR{$j = 1,\ldots, J$}{
            \STATE Sample $w^{j}_{m} \sim \B\left(\frac{H+1}{\kappa}, \frac{m + n_0}{\kappa} \right)$ and $w^{\aux,j}_{m} \sim \B\left(\frac{1}{\kappa^{\aux}}, \frac{m^{\aux}+n_0^{\aux}}{\kappa^{\aux}}\right)$.
            \STATE Update $\tildeQ_{h}^{j}$ and $\tildeQ_{h}^{\aux,j}$ via \eqref{eq:Q_circ_update} and \eqref{eq:Q_dagger_update}.
            }
            \ENDFOR
            \STATE \textcolor{blue}{\texttt{/* Update the agile policy Q-function by optimistic mixing.\hfill */}}
            \STATE Update the policy Q-function $Q_h$ via \eqref{eq:policy_Q_update}. \label{line:policy_Q_update}
            \STATE \textcolor{blue}{\texttt{/* Update the policy with step-wise agility. \hfill */}}
            \STATE Update policy $\pi_h(s_h) \gets \arg\max_{a\in\A} Q_h(s_h,a)$. \label{line:policy_update}
            \STATE \textcolor{blue}{\texttt{/* Update $\tildeV_{h}$ optimistically. \hfill */}}
            \STATE Update $\tildeV_{h}(s_{h})\gets \max_{j\in[J]}\tildeQ^{j}_{h}(s_{h}, \pi_h(s_h))$. \label{line:tilde_V_update}
            \STATE \textcolor{blue}{\texttt{/* Update visit counters. \hfill */}}            
            \STATE Update counter $n_h\!(s_{h}, a_{h})\!\!\gets\! n_h(s_{h}, a_{h})\!+\!1$ and $n_h^{\aux}(s_h,a_h) \gets n_h^{\aux}(s_h,a_h) + 1$.
            \STATE \textcolor{blue}{\texttt{/* At the end of the stage: update $\widetilde{Q}_h^{\aux}$, $\pi^{\aux}_{h}$, $\widetilde{V}_h^{\aux}$ and reset $n_h^{\aux}$, $\{\widetilde{Q}_h^{b,j}\}$.\hfill */}}
            \IF{$n_h^{\aux}(s_h,a_h) = \lfloor(1+1/H)^q H\rfloor$ for the stage $q=q_h(s_h,a_h)$}{
                \STATE Update $\tildeQ_{h}^{\aux}(s_{h}, a_{h}) \gets \max_{j \in [J]} \widetilde{Q}^{\aux,j}_{h}(s_{h}, a_{h})$, $\pi^{\aux}_{h}(s_{h})\!\gets\arg\max_{a\in\A}\widetilde{Q}^{\aux}_{h}(s_{h}, a)$,\\ and $\tilde{V}_h^{\aux}(s_h)\gets \widetilde{Q}^{\aux}_{h}(s_{h}, \pi^{\aux}_{h}(s_{h}))$. \label{line:stage_value_update}
                \STATE Reset  $\widetilde{Q}_h^{\aux,j}(s_h,a_h) \gets r_h(s_h,a_h) + V^0_{h+1}$ for $j\in[J]$ and $n_h^{\aux}(s_h,a_h) \gets 0 $. \label{line:stage_init}
                \STATE Move to the next stage: $q_h\!(s_h,a_h)\!\gets q_h\!(s_h,a_h) + 1$.
            }
            \ENDIF
        }
        \ENDFOR
    }
    \ENDFOR
    \end{algorithmic}
    \end{algorithm}

\paragraph{Two Q-ensembles for adaptation and exploration.}
To ensure the mixed Q-function with the learning rate scaled as $O(H/m)$, we tailor the probability distribution of the randomized learning rate as: 
$$w_m^{j} \sim \text{Beta}\left(\tfrac{H+1}{\kappa},\tfrac{m + n_0}{\kappa}\right),\quad  \forall j\in[J],$$
 where $m$ records the total number of visits to the state-action pair $(s_h,a_h)$ just before current visit (\cf Line \ref{line:visit_counter} in Algorithm \ref{alg:samplingq}), $n_0$ introduces pseudo-transitions to induce optimism, and $\kappa > 0$ controls the concentration level of the distribution and $J$ is the size of temporary Q-ensembles. With these randomized learning rates in hand, the corresponding entry of temporary Q-ensembles is updated in parallel as: 
\begin{equation}\label{eq:Q_circ_update}
    \begin{aligned}
        \tildeQ_{h}^{j}(s_h, a_h) \gets & (1 - w^{j}_{m}) \tildeQ_{h}^{j}(s_h, a_h) + w^{j}_{m} \big(r_h(s_h, a_h)+\tildeV_{h+1}(s_{h+1})\big),
    \end{aligned}
\end{equation}
for any $j\in[J]$, where $\tildeV_{h+1}$ is the optimistic value estimate  at the next step, computed \textit{before} processing the current transition and held fixed within this visit.
As discussed in Section \ref{sec:motivation}, such a learning rate could guarantee a polynomial dependency on the horizon $H$, but lead to rapidly forgetting the earlier episodes and assigning exponentially decreasing weights on the optimistic initialization. To further emphasize the optimistic initialization, we also introduce another sequence of Q-ensembles as follows: for any $j\in[J]$,
\begin{equation}\label{eq:Q_dagger_update}
    \begin{aligned}
         \widetilde{Q}_{h}^{\aux,j}(s_{h}, a_{h}) \gets &(1 - w^{\aux,j}_{m}) \widetilde{Q}_{h}^{\aux,j}(s_h, a_h)  + w^{\aux,j}_{m} \big(r_h(s_{h}, a_{h})+\tildeV^{\aux}_{h+1}(s_{h+1})\big),
    \end{aligned}
\end{equation}
where the randomized learning rates are sampled from $w_m^{\aux,j} \sim \text{Beta}\left(\tfrac{1}{\kappa^{\aux}}, \tfrac{m^{\aux} + n_0^{\aux}}{\kappa^{\aux}}\right)$, $m^{\aux}$ represents the number of visits during the current stage,  and $\widetilde{V}^{\aux}_{h+1}$ is the value estimate updated in a stage-wise manner and frozen for the current visit (cf. Line \ref{line:stage_value_update} in Algorithm \ref{alg:samplingq}).

\paragraph{Agile policy Q-function via optimistic mixing.}
Then, to promote optimism, the policy Q-function is computed via optimistic mixing (cf. Line \ref{line:policy_Q_update} in Algorithm \ref{alg:samplingq}). For the current state--action pair $(s_h,a_h)$,
\begin{equation}\label{eq:policy_Q_update} 
    \begin{aligned}
            Q_{h}^{t+1}(s_h, a_h) = &\eta_{t,h} \max_{j\in[J]}\{\widetilde{Q}^{j,t+1}_{h}(s_h, a_h)\} + (1-\eta_{t,h}) \cdot\widetilde{Q}_{h}^{\aux,t+1}(s_h, a_h),
    \end{aligned}
\end{equation}
and for all $ (s,a)\ne (s_h,a_h)$, we keep $Q^{t+1}_{h}(s, a)\! = \! Q^{t}_{h}(s, a)$. Here, $\widetilde{Q}^{j,t+1}_{h}(s_h,a_h)$ is the $j$-th temporary Q-value, updated at the current visit of episode $t$ via \eqref{eq:Q_circ_update}, while the staged $\widetilde{Q}^{\aux,t+1}_{h}$ remains fixed throughout the current stage and is optimistically refreshed to the ensemble maximum only at the end of the stage (cf. Line \ref{line:stage_value_update} in Algorithm~\ref{alg:samplingq}).
Note that the first term---corresponding to the maximum over $J$ temporary Q-values---is updated every step, which allows \myalg to perform {\em agile policy updates} rather than the {\em exponentially slower} schedule that updates only at stage boundaries.
Such optimistic mixing allows \myalg to remain responsive to new data and adapt the policy efficiently without requiring periodic updates.

\paragraph{Reset for bias mitigation and optimism restoration.}  
To mitigate outdated data and ensure optimism, we reset the temporary Q-ensembles $\widetilde{Q}_{h}^{\aux,j}$ according to the optimistic initialization $V_{h+1}^0$ and the visit counter (\cf Line \ref{line:stage_init} in Algorithm \ref{alg:samplingq}), when the number of visits in current stage exceeds a predefined threshold---specifically, when $n_h^{\aux}(s_h, a_h) = \lfloor(1 + 1/H)^q H \rfloor$ for the $q$-th stage. Meanwhile, the staged value estimate $\tildeV^{\aux}_h$ is greedily updated with respect to $\widetilde{Q}^{\aux}_{h}$ at state $s_h$ (cf. Line \ref{line:stage_value_update} in Algorithm \ref{alg:samplingq}), which will be reused to update the temporary Q-values in the subsequent visits within the new stage.

%% file: section/theoretical_analysis.tex
\section{Theoretical Guarantee}
In this section we provide both gap-independent and gap-dependent regret bounds for \myalg, considering the worst-case scenario and favorable structural MDPs, respectively. 
\subsection{Gap-independent sublinear regret guarantee}\label{sec:gap_independent_bound}
To begin with, the following theorem shows that \myalg has a $\sqrt{T}$-type regret bound, where the full proof is deferred to Appendix \ref{appendix:regret_proof}.
\begin{theorem}\label{thm:regret_main}
    Consider $\delta\in (0,1)$. Assume that $J = \lceil{c\cdot\log(SAHT/\delta)}\rceil$, $\kappa^{\aux} = c\cdot(\log(SAH/\delta) + \log(T))$, and $n_0^{\aux} = \lceil c\cdot\log(T)\cdot \kappa \rceil$, where $c$ is some universal constant. Let the initialized value function $V_h^0 = 2(H-h+1)$ for any $h\in[H+1]$, and the mixing rate $\eta_{t,h} = \frac{1}{\sqrt{(1+1/H)^{q}H}+1}$ where $q = q_h^t(s_h^t,a_h^t)$ is the stage index for any $(t,h)\in[T]\times[H]$. Then, with probability at least $1-\delta$, Algorithm \ref{alg:samplingq} guarantees that
    \begin{align*}
        \mathrm{Regret}_T \le \widetilde{O}\left(\sqrt{H^5SAT}\right).
    \end{align*}
\end{theorem}

Theorem~\ref{thm:regret_main} shows that \myalg achieves a gap-independent regret bound of $\widetilde{O}(\sqrt{H^5SAT})$, matching the guarantees of UCB-Q with Hoeffding-type bonuses~\citep{jin2018q}  in episodic tabular MDPs. This bound is minimax-optimal up to polynomial factors of $H$ when compared to the known lower bound of $\Omega(\sqrt{H^3SAT})$~\citep{jin2018q,domingues2021episodic}.  

\begin{remark}[Anytime convergence guarantee]
    To obtain an anytime convergence, we can update the ensemble size $J$, inflation parameter $\kappa^{\aux}$ and $n_0^{\aux}$ adaptively at every stage (see the detailed description of Algorithm \ref{alg:samplingq_anytime} in Appendix \ref{appendix:anytime}). Then without requiring the number of episodes $T$ in advance, Algorithm \ref{alg:samplingq_anytime} attains the same regret bound $\widetilde{O}(\sqrt{H^5SAT})$ (see the proof in Appendix \ref{appendix:proof_anytime}), whereas the prior study \cite{tiapkin2024model} is not anytime.
\end{remark}

\paragraph{Technical challenges.} 
The primary challenge in analyzing \myalg arises from several subtle requirements on the randomized learning rates. Specifically, these rates must:
\begin{itemize}%[leftmargin=1em, topsep=0.1em, itemsep=0em]
\item be sufficiently randomized to induce necessary optimism;
\item avoid excessive randomness that could incur undesirable fluctuations;
\item support efficient exploitation of the most recent observations to avoid introducing exponential dependence on the horizon $H$. 
\end{itemize}

The subtle interplay among these conditions precludes the straightforward application of existing analytical techniques from the literature. For instance, the optimism may decay exponentially and be insufficient for sparse reward scenarios, so we re-inject the weighted optimistic values into the Q-ensembles at every stage to ensure necessary optimism at every step. In addition, to bound undesirable fluctuations of randomized learning rates, prior work \citep{tiapkin2024model} attempted to prove a concentration inequality based on Rosenthal's inequality \citep[Theorem 6]{tiapkin2024model}, which in turn requires a martingale property of the so-called \emph{aggregated} learning rates. However, the martingale property in fact does not hold (detailed below), revealing a gap in their proof. We propose a new proof strategy to bridge this gap and to extend the concentration inequality to our setting. 

These challenges jointly necessitate a carefully constructed mixing scheme that balances the efficient exploration and agile responses to latest observations, refined control of fluctuation and favorable properties of learning rates to ensure that \myalg attains near-optimal sample complexity with agile updates.% \vspace{-0.5em}

\paragraph{Identifying and fixing a technical gap in the proof of \cite{tiapkin2024model}.} While  \citet{tiapkin2024model} established a comparable regret bound for Staged-RandQL, it turns out that analysis has a crucial technical gap. Specifically, central to the analysis is to study the concentration of the weighted sum of the \emph{aggregated randomized learning rates}, defined as %\vspace{-0.4em}
\begin{align*}
    W_{j,m}^{0} &= \prod_{k=0}^{m-1} (1-w_{k}^{j})\\
    W_{j,m}^{i} &= w_{i-1}^{j}\prod_{k=i}^{m-1} (1-w_{k}^{j}), \quad\forall i\in[m],
\end{align*}
which involves bounding the sum %\vspace{-1em}
\begin{equation}\label{eq:concentration_sum}
    \left| \sum_{i=0}^{m} \lambda_i (W_{j, m}^i - \E[W_{j, m}^i]) \right|,
\end{equation}
for fixed real numbers $\lambda_i \in [-1, 1]$; see the proof of Lemma~4 in \citet{tiapkin2024model}. To this end, \citet{tiapkin2024model} asserted that the partial sums $S_i = \sum_{k=0}^{i} \lambda_k (W_{j, m}^k - \E[W_{j, m}^k])$ form a martingale with respect to some filtration $\F_i$ (cf. Proposition~7 there, which was invoked in the proof of Lemma~4 therein). The proof went on by a standard application of Rosenthal's inequality (\ie, Theorem~6 therein). To prove the martingale property, it was claimed that $W_{j, m}^i$ is adapted to $\F_i$ (\cf assumption of their Theorem~6), and $W_{j, m}^i$ is independent of $\F_{i-1}$ (\cf assumption of their Proposition~7).  Unfortunately, it is \emph{impossible} to achieve both adaptedness and independence except for trivial cases (e.g., deterministic learning rates), regardless of choice of $\F_i$.\footnote{\citet{tiapkin2024model} chose a filtration $\{\F_i\}$ to which $W_{j, m}^i$ is actually not adapted. } Indeed, if $\{\F_i\}$ is such a filtration, then $W_{j, m}^i$ being adapted means that all the randomness of $W_{j, m}^0, W_{j, m}^1 \cdots, W_{j, m}^{i-1}$ is contained in $\F_{i-1}$. As $W_{j, m}^i$ is independent of $\F_{i-1}$, we see that $W_{j, m}^i$ is independent with all of $W_{j, m}^0, \cdots, W_{j, m}^{i-1}$. By induction, we readily see that $W_{j, m}^0, \cdots, W_{j, m}^m$ are jointly independent. However, since $\sum_{i=0}^m W_{j, m}^i = 1$ \citep[Lemma 3]{tiapkin2024model}, such independence is not possible unless all aggregated learning rates $W_{j, m}^i$, $0\le i \le m$, are deterministic. Thus, Lemma 4 in \citet{tiapkin2024model} does not hold, thereby leaving a gap in the analysis. We fix this gap by introducing a reverse filtration that is tailored to the form of the aggregated weights, and study  \eqref{eq:concentration_sum} using a backward martingale construction in contrast with partial sums with substantial new analyses. With this approach, we established the correct concentration inequality not only for their setting but also for ours.  
 
\paragraph{Memory and computation complexity.} As the number of ensembles is $J = \widetilde{O}(1)$, the computational complexity is $O(H)$ per episode and the space complexity is $O(HSA)$, same as \cite{tiapkin2024model}. However, we note that due to the use of optimistic mixing, \myalg requires maintaining two ensembles, which effectively doubles the memory and computational cost compared to Staged-RandQL \citep{tiapkin2024model}.

\subsection{Gap-dependent logarithmic regret guarantee}
Note that such a $\sqrt{T}$-type regret bound holds for \textit{any} episodic tabular MDPs, which might not be tight for environment with some benign structural properties. To this end, we further develop a gap-dependent regret bound, which improves the regret bound from sublinear to logarithmic under a strictly positive suboptimality gap condition, as follows.
\begin{assumption}[Positive suboptimality gap]\label{assump:positive_gap}
    For any $(s,a,h)\in\S\times\A\times[H]$, we denote the sub-optimality gap as $\Delta_h(s,a) \defeq V_h^{\star}(s)-Q_h^{\star}(s,a)$ and assume that the minimal gap 
    $$\Delta_{\min} \triangleq \underset{{\substack{(s,a,h) \in \mathcal{S} \times \mathcal{A} \times [H]}}}{\min} \Delta_h(s,a) \mathbbm{1}\{\Delta_h(s,a) \ne 0\} > 0. $$
\end{assumption}
Note that this assumption implies that there exist some strictly better actions (i.e., the optimal actions) outperform the others for every state. This assumption is mild, as the minimal suboptimality gap $\Delta_{\min} = 0$ only when the MDPs degenerates. Consequently, it is commonly fulfilled in environments with finite action spaces, such as Atari-games and control tasks, and it is also widely adopted in prior literature \citep{yang2021q,zheng2024gap}.

Under this mild assumption, we have the following logarithmic gap-dependent regret bound, whose proof is deferred to Appendix \ref{appendix:log_regret}. To the best of our knowledge, this is the \textit{first} guarantee that shows sampling-based Q-learning can also achieve the logarithmic regret in episodic tabular RL.  
\begin{theorem}\label{thm:log_regret}
    Consider $\delta\in (0,1)$. Suppose all conditions in Theorem \ref{thm:regret_main} and Assumption \ref{assump:positive_gap} hold. Let the mixing rate $\eta_{t,h} = \frac{1}{H(1+\frac{1}{H})^q}$, where $q = q_h^t(s_h^t,a_h^t)$ is the stage index for every $(t,h)\in[T]\times[H]$. Then,  Algorithm \ref{alg:samplingq} guarantees that
    \begin{equation*}
       \mE[\mathrm{Regret}_T] \le O\left(\frac{H^6SA}{\Delta_{\min}}\log^5(SAHT)\right).
    \end{equation*}
\end{theorem}
Note that the above logarithmic regret bound holds for \myalg without assuming any prior knowledge on the minimal suboptimality gap $\Delta_{\min}$ during implementation. As shown in \citet{simchowitz2019non},  any algorithm with an $\Omega(\sqrt{T})$ regret bound in the worst case, has a $\log T$-type lower bound of the expected gap-dependent regret. Also, our bound matches the expected regret for UCB-Q \citep{jin2018q} under the same condition (i.e. Assumption \ref{assump:positive_gap}) for episodic tabular MDPs \citep{yang2021q}, which is nearly tight in $S,A,T$ up to the $\log(SAT)$ and $H$ factors.

%% file: section/simulation.tex
\section{Experiments}\label{section:experiment}
In this section, we present the experimental results of \myalg compared to baseline algorithms, using \textsc{rlberry} \citep{rlberry}, in the following two environments. 
All the experiments are conducted on a machine equipped with 2 CPUs (Intel(R) Xeon(R) Gold 6244 CPU), running Red Hat Enterprise Linux 9.4, without GPU acceleration.
The corresponding codes can be found at  
\begin{center}
    \url{https://github.com/IrisHeWANG/RandomizedQ}
\end{center}
which is built upon the implementation of \cite{tiapkin2024model}, using the \textsc{rlberry} library \citep{rlberry}.

\paragraph{A grid-world environment.}
We first evaluate performance in a $10 \times 10$ grid-world environment as used in \citet{tiapkin2024model}, where each state is represented as a tuple $(i, j) \in [10] \times [10]$, and the agent can choose from four actions: left, right, up, and down. The episode horizon is set to $H = 50$. At each step, the agent moves in the planed direction with probability $1 - \epsilon$ and to a random neighboring state with probability $\epsilon = 0.2$. The agent starts at position $(1,1)$, and the reward is 1 only at state $(10,10)$, with all other states yielding zero reward. We also examine the performance in a larger $25 \times 25$ grid-world environment, where the agent receives a reward of 1 only at state $(25,25)$, with the episode horizon set to $H = 200$. Compared to the $10 \times 10$ setting, the reward is significantly sparser. The corresponding results are shown in Figures~\ref{fig:grid} and~\ref{fig:grid_large}, respectively.

\paragraph{A chain MDP.}
We also consider a chain MDP environment as described in \citet{osband2016deep}, which consists of $L=20$ states (i.e., the length of the chain) and two actions: left and right. The episode horizon is set to $H = 50$. With each action, the agent transits in the intended direction with probability $0.9$, and in the opposite direction with probability $0.1$. The agent starts in the leftmost state, which provides a reward of 0.05, while the rightmost state yields the highest reward of 1.  Additionally, we evaluate performance in a longer chain MDP with $L = 50$ states and a horizon of $H = 100$. The corresponding results are shown in Figures~\ref{fig:chain} and~\ref{fig:chain_large}, respectively.
 
\begin{figure}[t]
    \centering
    \begin{subfigure}{0.48\linewidth}
        \centering
        \includegraphics[height=2in,trim=9 0 5 0, clip]{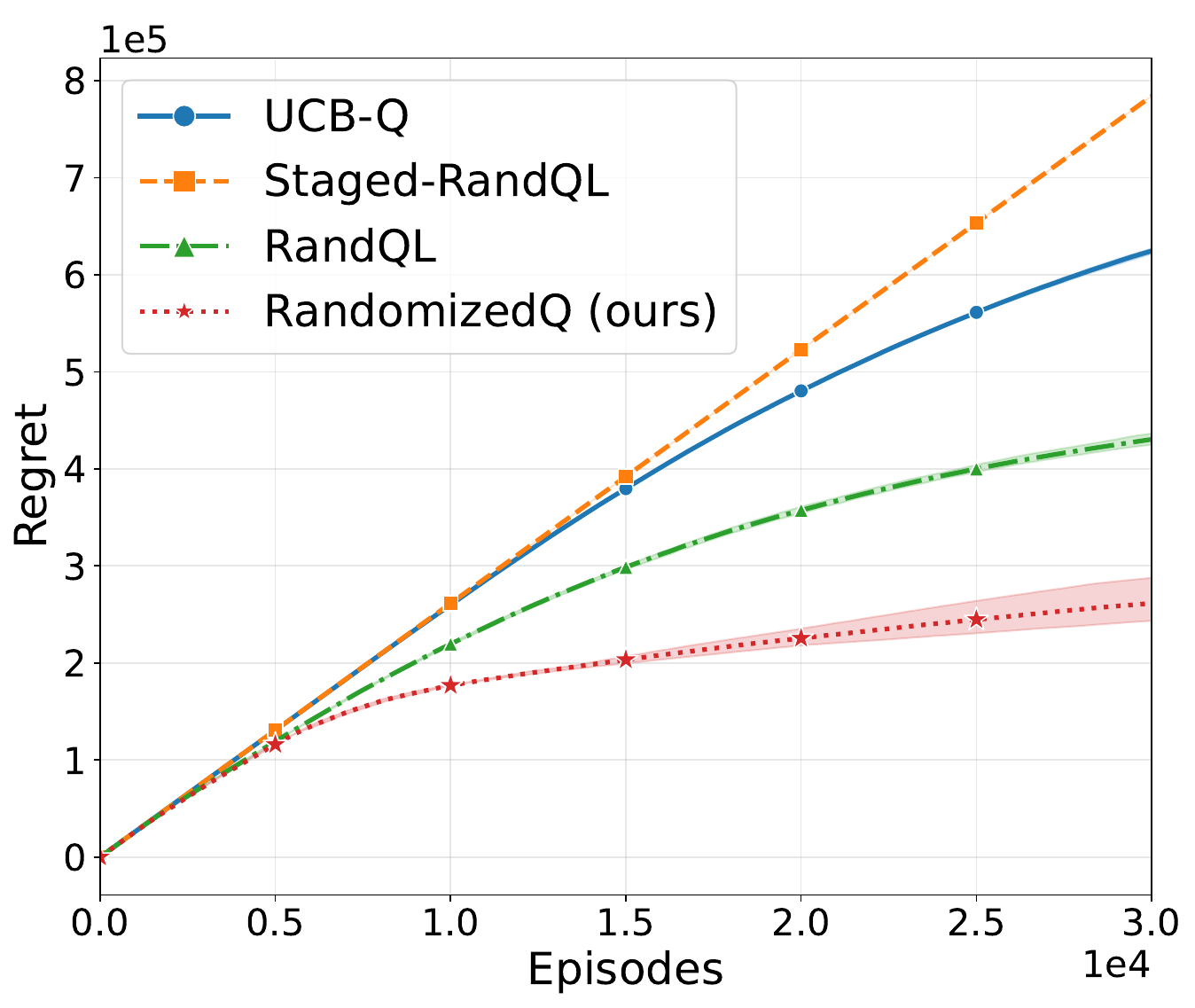}
        \caption{A $10 \times 10$ grid-world environment.}
        \label{fig:grid}
    \end{subfigure}
    \begin{subfigure}{0.48\linewidth}
        \centering
        \includegraphics[height=2in,trim=9 0 5 0, clip]{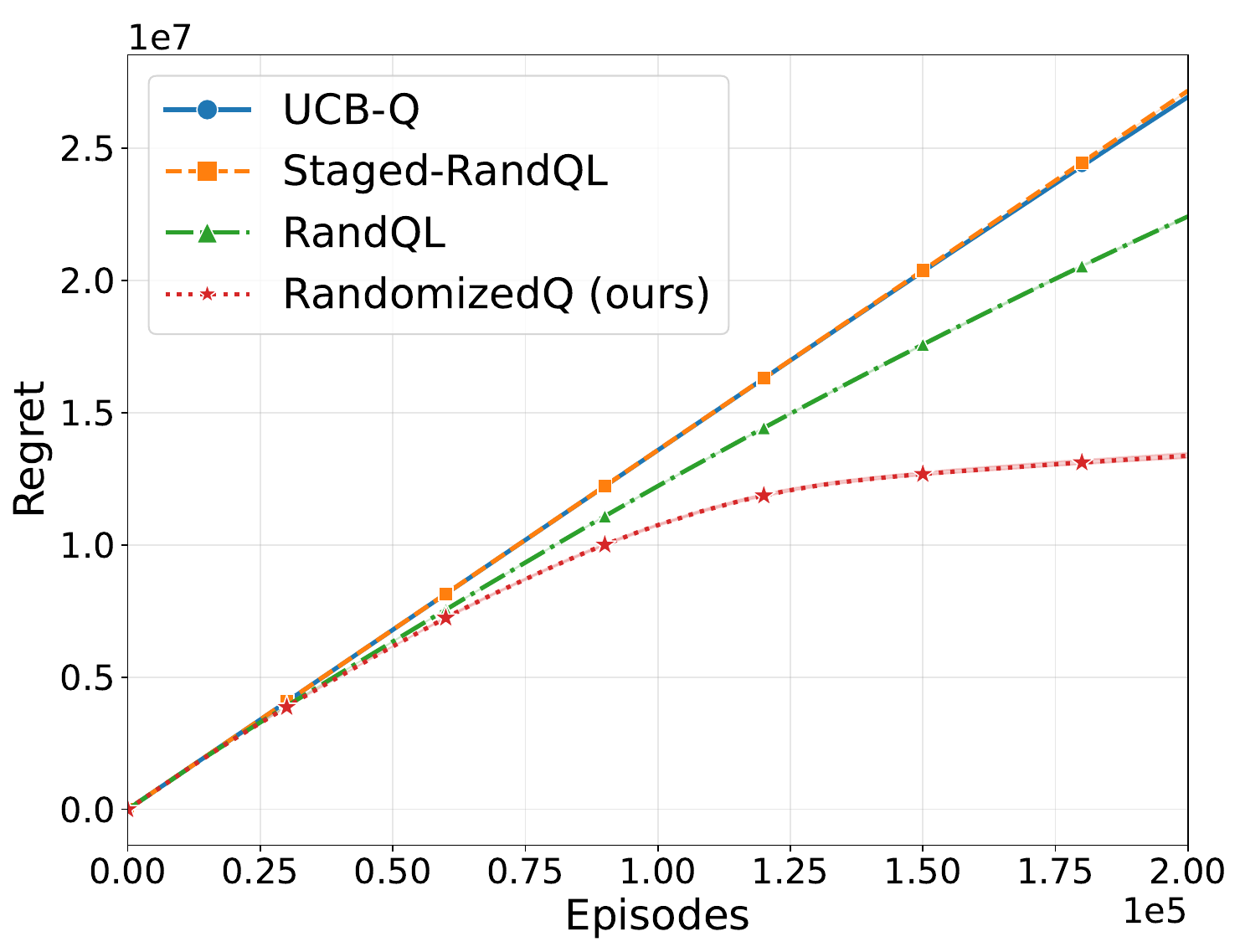}
        \caption{A $25 \times 25$ grid-world environment.}
        \label{fig:grid_large}
    \end{subfigure}
    \begin{subfigure}{0.48\linewidth}
        \centering
        \includegraphics[height=2in,trim=9 0 5 0, clip]{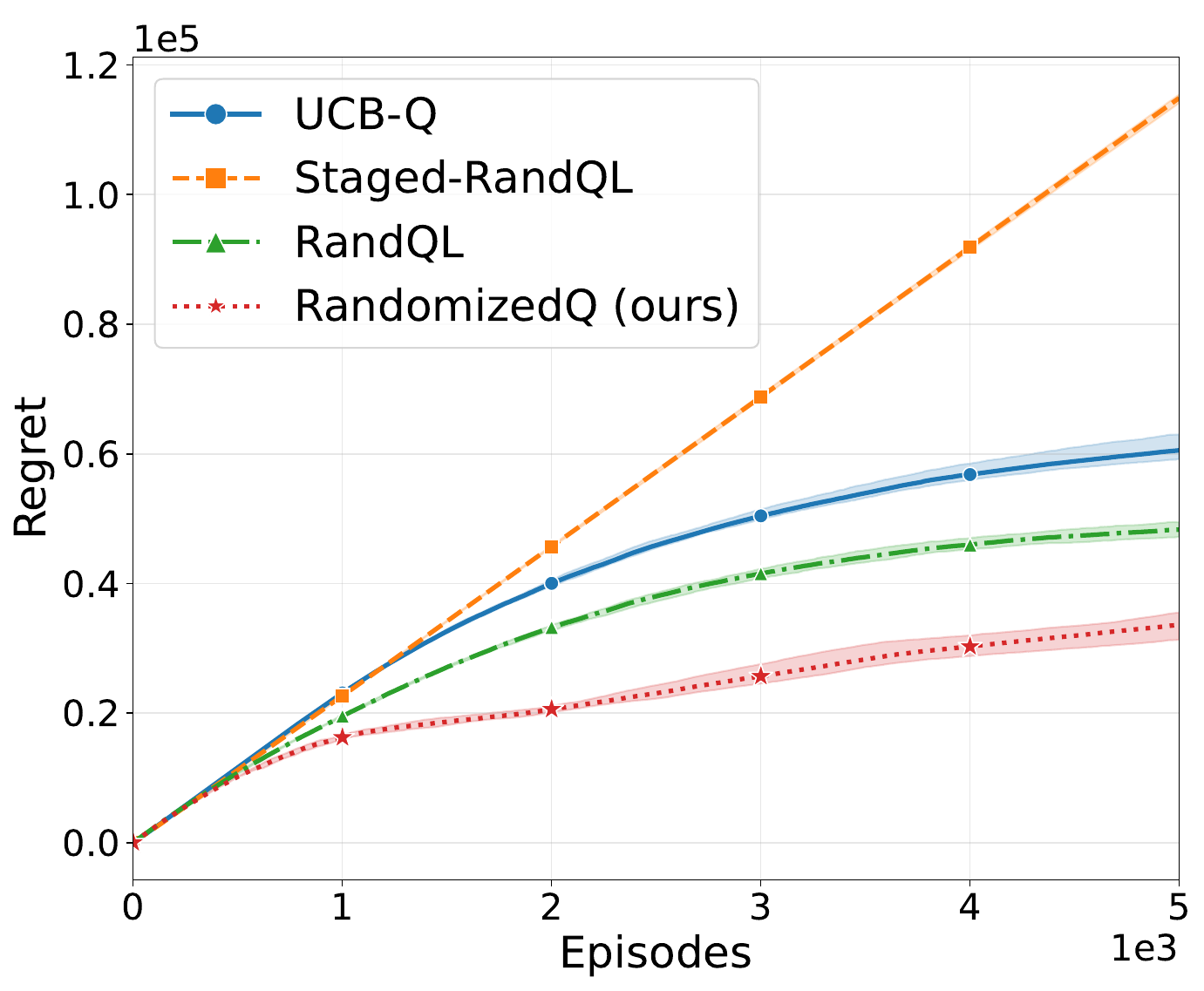}
        \caption{A chain MDP with length $L=20$.}
        \label{fig:chain}
    \end{subfigure}
    \begin{subfigure}{0.48\linewidth}
        \centering
        \includegraphics[height=2in,trim=9 0 5 0, clip]{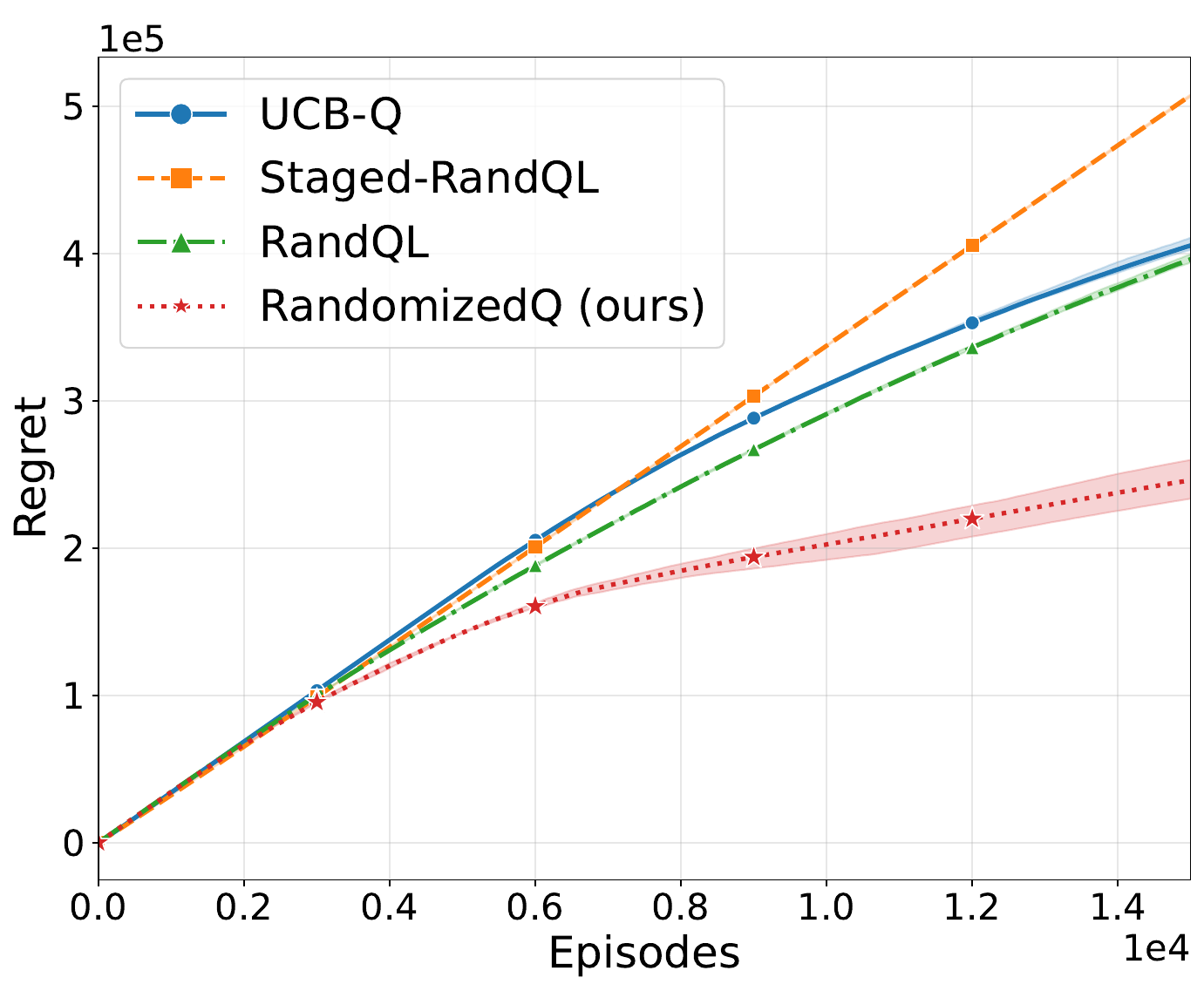}
        \caption{A chain MDP with length $L=50$.}
        \label{fig:chain_large}
    \end{subfigure}
    \caption{Comparison between \myalg and baseline algorithms in the grid-world environment (\cf the first row) and the chain MDP (\cf the second row), where total regret is plotted against the number of episodes. \myalg consistently achieves lower regret than UCB-Q, as well as both the standard randomized Q-learning (\ie, RandQL) and its stage-wise variant (\ie, Staged-RandQL), demonstrating superior sample efficiency and faster learning processes.  }
    \label{fig:mainfig}
\end{figure}

\paragraph{Baselines and experiment setups.}
We compare \myalg with (1) UCB-Q: model-free Q-learning with bonuses \citep{jin2018q} (2) Staged-RandQL \citep{tiapkin2024model}: the staged version of RandQL with theoretical guarantees (3) RandQL \citep{tiapkin2024model}: the randomized version of UCB-Q, without provable guarantees. For all algorithms with randomized learning rates in both environments, we let the number of ensembles $J=20$, the inflation coefficient $\kappa=1$, and the number of pseudo-transitions $n_0 = 1/S$, where $S$ corresponds to the size of the state space in different environments. For \myalg, we set the mixing rate $\eta_{t,h} = (\sqrt{(1+1/H)^{q}H}+1)^{-1}$, where $q = q_h^t(s_h^t,a_h^t)$, $\forall (t,h)\in[T]\times[H]$.
For fair comparison, we run $4$ trials per algorithm and show the average along with the 90\% confidence interval in the figures above.

\paragraph{Results.}
From Figure \ref{fig:mainfig}, \myalg exhibits significantly improved performance across all environment sizes, achieving substantially lower total regret. Unlike UCB-Q, which suffers from excessive over-exploration, and the Staged-RandQL that adapts the policy only at the end of each stage, \myalg effectively balances exploration and exploitation through randomized learning rates and agile policy updates. We also observe that \myalg performs even better than the empirical RandQL that lacks theoretical guarantees in prior work \citep{tiapkin2024model}, especially for larger environments with more sparse rewards. These results validate the effectiveness of sampling-driven updates without explicit bonus terms and highlight the benefit of avoiding stage-wise policy updates in model-free reinforcement learning.

%% file: appendix/notation.tex
\section{Notation and Preliminaries}
Before proceeding, we first introduce the following notation with the dependency on the episode index $t$ and its short-hand notation whenever it is clear from the context.
\begin{itemize}%[leftmargin=2em,itemsep=6pt]
    \item \(n_h^t(s,a)\), or the shorthand \(n^t_h\): the number of previous visits to \((s,a)\) at step \(h\) before episode \(t\).
    \item \(n_h^{\aux,t}(s,a)\), or the shorthand \(n^{\aux,t}_h\): the number of previous visits to \((s,a)\) at step \(h\) during the current stage that the episode $t$ belongs to.
    \item $q_h^t(s,a)$, or the shorthand $q_h^t$: the stage index of the $i$-th visit to $(s,a)$ at step \(h\) and episode $t$.
    \item \(\ell_h^i(s,a)\), or the shorthand \(\ell^{i}\):
          the episode index of the \(i\)-th visit to \((s,a)\) at step \(h\);
          by convention \(\ell^{0}=0\).
    \item \(\ell_{q,h}^{\aux,i}(s,a)\), or the shorthand \(\ell^{i}_{q}\):
          the episode index of the \(i\)-th visit to \((s,a)\) at step \(h\) during the stage $q$; by convention, \(\ell^{0}_{0}=0\) and $\ell^{\aux,0}_{q,h}(s,a)$ represents the episode when the $q$-th stage starts for $(h,s,a)$.
    \item $e_q = \lceil (1+1/H)^q H\rceil$: the length of the $q$-th stage; by convention, $e_{-1}=0$.
    \item \(J\):
          number of ensemble heads (temporary \(Q\)-functions) per episode.   
    \item \(\widetilde Q_h^{j,t}(s,a)\):
          \(j\)-th temporary (ensemble) estimate of the optimal Q-value
          at the \emph{beginning} of episode \(t\), where the randomized learning rate follows $\B(\frac{H+1}{\kappa},\frac{n_h^t+n_0 - 1}{\kappa})$.

    \item \(\widetilde Q_h^{\aux,j,t}(s,a)\):
          \(j\)-th temporary (ensemble) estimate of the optimal Q-value
          at the \emph{beginning} of episode \(t\), where the randomized learning rate follows $\B(\frac{1}{\kappa^{\aux}},\frac{n_h^{\aux,t}+n_0^{\aux} - 1}{\kappa^{\aux}})$.
    
    \item \(\widetilde Q_h^{\aux,t}(s,a)\):
          the optimistic approximation of the optimal Q-function updated at the end of each stage.
    
    \item \(Q_h^{t}(s,a)\):
          the policy \(Q\)-function at the start of episode \(t\);
          its update at the visited pair \((s_h^t,a_h^t)\) is
          \begin{equation}\label{eq:policy_Q_with_dependency}
          Q_h^{t}(s_h^t,a_h^t)
          = \eta_{t-1,h}\max_{j\in[J]}\!\widetilde Q_h^{j,t}(s_h^t,a_h^t)
            + (1-\eta_{t-1,h}) \widetilde Q_h^{\aux,t}(s_h^t,a_h^t).
          \end{equation}
    
    \item \(\widetilde V_{h}^{\ell^{i}}\),
          \(\widetilde V_{h}^{\aux,\ell_{q}^i}\):
          optimistic value estimation of the optimal value function at episode
          \(\ell^{i}\) and \(\ell_{q}^i\).    
    \item \(\pi^t_h\): the learned policy used at step \(h\) in episode \(t\); the action $a_h^t$ is drawn from the learned policy $\pi_h^t$ at state $s_h^t$ for any $(h,t)\in[H]\times[T]$.
    \end{itemize}
For analysis, we also introduce the following notation. 
\begin{itemize}%[leftmargin=2em,itemsep=6pt]
        \item $s_0$: the optimistic pseudo-state $s_0$ with
        \[
            r_h(s_0,a)=r_0 \;>\; 1,
            \quad
            p_h(s_0\mid s,a)=\1\{s=s_0\}.
        \]
    
        \item $V^{\star}_{h}(s_0)$: the cumulative return obtained by always staying at the optimistic state $s_0$ from step $h$, \ie, $V^{\star}_{h}(s_0)=r_0\bigl(H-h+1\bigr)$.
    
        \item $n_0,n_0^{\aux}$: prior pseudo-transition counts; thereby, each state-action pair $(s,a)$ starts with $n_0$ prior pseudo-transitions, leading to  
        \[
            w_{0}^{\,j}\sim\mathrm{Beta}\bigl((H+1)/\kappa,\;n_0/\kappa\bigr),
            \quad
            w_{0}^{\aux,j}\sim\mathrm{Beta}\bigl(1/\kappa^{\aux},\;n_0^{\aux}/\kappa^{\aux}\bigr),
            \qquad j\in[J].
        \]
    
        \item $\mathcal K_{\inf}(p,\mu)$: the information-theoretic distance between some measure $p\in\mathcal P[0,b]$ and $\mu\in[0,b]$, defined as 
        \[\mathcal{K}_{\inf}(p,\mu) = \inf \{\text{KL}(p,q): q\in \mathcal{P}[0,b], p \ll q, \E_{X\sim q}[X]\ge \mu\},\]
        where $\mathcal P[0,b]$ denotes all probability measures supported on $[0,b]$.

        \item  $\delta_x$: Dirac measure concentrated at a single point $x$.
    
        \item $[n]$: the indexing shorthand; for a positive integer $n$, we write $[n]\coloneqq\{1,2,\ldots,n\}$.
    
        \item $\|X\|_p$: the $\ell_p$-norm of a vector $X\in\mathbb R^n$ where $p\ge1$; formally defined as
        \[
            \|X\|_p=\Bigl(\sum_{i=1}^n |X_i|^p\Bigr)^{1/p}.
        \]
    
        \item $(x)_k$: Pochhammer symbol, \ie,  for $k\in\mathbb N$,
        \begin{equation}\label{eq:weird_symbol}
            (x)_k = x(x+1)\cdots(x+k-1).
        \end{equation} 
        \item $\1 \{x\ge c\}$: an indicator function that equals $1$ when $x\ge c$, and 0 otherwise.
        \item $|X|$: the cardinality of the set $X$.
         \item \( A \lesssim B \) : means \( A \le c B \) for some universal constant \( c > 0 \).
         \item $\mathbb{N}^{\star}$: the set consists of the positive integers, i.e., $\{1,2,3,\ldots\}$.
    \end{itemize}

\paragraph{Beta distribution.}
We recall the definition and important properties of Beta distribution \citep[Section 2]{gupta2004handbook}, which is used in the follow-up analysis.

\begin{definition}[Beta distribution]\label{def:beta}
    A continuous random variable \(X\) is said to follow a \emph{Beta} distribution with shape parameters \(\alpha>0\) and \(\beta>0\), written as
    \[
    X \sim \operatorname{Beta}(\alpha,\beta),
    \]
    if its probability density function is  
    \[
    f_{X}(x \mid \alpha,\beta)
        \;=\;
       \frac{1}{B(\alpha,\beta)}
       \,x^{\alpha-1}\,(1-x)^{\beta-1},
    \qquad 0 < x < 1,
    \]
    and \(f_{X}(x)=0\) otherwise,
    where the \emph{Beta function}
    \(B(\alpha,\beta)=\dfrac{\Gamma(\alpha)\Gamma(\beta)}{\Gamma(\alpha+\beta)}\)
    serves as the normalizing constant.
\end{definition}

\begin{lemma}[Moments of the Beta distribution]\label{lem:beta_moments}
    Let \(X \sim \mathrm{Beta}(a,b)\) with \(a,b>0\), and recall the Pochhammer symbol \((x)_k\) defined in \eqref{eq:weird_symbol}.  Then, for any positive integer \(r\),
    \begin{equation}\label{eq:beta_moment}
        \E\left[X^{r}\right] = \frac{(a)_r}{(a+b)_r}.
    \end{equation}
    In particular, the expectation and variance of \(X\) are
    \begin{equation}\label{eq:beta_expectation_variance}
        \E[X] = \frac{a}{a+b},
        \qquad
        \Var(X) = \frac{ab}{(a+b)^2(a+b+1)}.
    \end{equation}
\end{lemma}

We next collect a few auxiliary lemmas used in the proofs.

\begin{lemma} \label{lemma:fraction_plus_k}
    For any $a,b\ge 1,\kappa\ge 0$, we have
\begin{equation}\label{eq:fraction_plus_k_1}
    \frac{a}{a+b}\le \frac{a+\kappa}{a+b+\kappa}\le (\kappa+1) \cdot\frac{a}{a+b}.
\end{equation}
In addition, for any $r\in \mathbb{N}_{+}$,
\begin{equation}\label{eq:fraction_plus_k_2}
    \frac{a+r\kappa}{a+b+r\kappa}\le r \cdot\frac{a+\kappa}{a+b+\kappa}.
\end{equation}
\end{lemma}
\begin{proof}
The following proves \eqref{eq:fraction_plus_k_1} and \eqref{eq:fraction_plus_k_2}, respectively.
We start with the lower bound of \eqref{eq:fraction_plus_k_1}. Define $f(t)=\dfrac{a+t}{a+b+t}$ for $t\ge 0$. Then $f'(t)=\dfrac{b}{(a+b+t)^2}>0$, so $f$ is increasing. Hence $f(0)=\dfrac{a}{a+b}\le f(\kappa)=\dfrac{a+\kappa}{a+b+\kappa}$. Moving to the upper bound, it is equivalent to show
\begin{equation*}
    (a+\kappa)(a+b)\;\le\;(\kappa+1)a(a+b+\kappa).
\end{equation*}
Expanding and rearranging leads to
\begin{equation*}
    (\kappa+1)a(a+b+\kappa)-(a+\kappa)(a+b)=\kappa\bigl(a^2+ab+a\kappa-b\bigr) \ge 0,
\end{equation*}
which holds if $a,b\ge 1$ and $\kappa\ge 0$. 

We now show \eqref{eq:fraction_plus_k_2}. Let $r\in\mathbb N_{+}$ and
    \[
    L=(a+r\kappa)(a+b+\kappa),\quad
    R=r(a+\kappa)(a+b+r\kappa).
    \]
    We prove $L\le R$ by first computing
    \[
    \begin{aligned}
    R-L
    &=r(a+\kappa)(a+b+r\kappa)-(a+r\kappa)(a+b+\kappa)\\
    &=(r-1)\Bigl[a(a+b)+\kappa\bigl(a(r+1)+\kappa r\bigr)\Bigr]\;\ge\;0,
    \end{aligned}
    \]
    because each bracketed term is non-negative and $r-1\ge 0$.
    Dividing by the common positive factor
    $(a+b+r\kappa)(a+b+\kappa)$ yields
    \(
    \dfrac{a+r\kappa}{a+b+r\kappa}\le r\dfrac{a+\kappa}{a+b+\kappa}.
    \)
All statements are therefore proved.
\end{proof}

\begin{lemma}[Rosenthal inequality, Theorem 4.1 in \cite{pinelis1994optimum}]\label{lemma:rosenthal_ineq}
    Let $X_1, \ldots, X_n$ be a martingale-difference sequence adapted to a filtration $\{\mathcal{F}_i\}_{i=1,\ldots,n}$:
    \[
    \mathbb{E}[X_i \mid \mathcal{F}_{i-1}] = 0.
    \]
    Define $\mathcal{V}_i = \mathbb{E}[X_i^2 \mid \mathcal{F}_{i-1}]$. Then there exist universal constants $c_1$ and $c_2$ such that for any $p \geq 2$ the following holds
    \[
    \mathbb{E}^{1/p} \left[ \left| \sum_{i=1}^n X_i \right|^p \right]
    \leq
    C_1 p^{1/2} \mathbb{E}^{1/p} \left[ \left| \sum_{i=1}^n \mathcal{V}_i \right|^{p/2} \right]
    + 2C_2 p \cdot \mathbb{E}^{1/p} \left[ \max_{i \in [n]} |X_i|^p \right].
    \]
\end{lemma}

\begin{lemma}[Corrected version of Lemma 12 in \citet{tiapkin2024model}.\footnote{We clarify an earlier oversight in \citet[Lemma 12]{tiapkin2024model} by properly accounting for the Dirac measure's contribution, which was previously incorrectly separated after applying the variational formula—--specifically, Lemma 9 in \cite{tiapkin2024model}. Consequently, the right-hand side of \eqref{eq:corrected_lemma12} is now scaled by $\alpha$, instead of the $1-\alpha$ factor used in \cite{tiapkin2024model}.}]\label{lemma:randql_lemma12}
    Let $\nu \in \mathcal{P}([0, b])$ be a probability measure over the segment $[0, b]$ and let $\bar{\nu} = (1 - \alpha)\delta_{b_0} + \alpha \cdot \nu$ be a mixture between $\nu$ and a Dirac measure on $b_0 > b$, where $\alpha\in(0,1)$. Then for any $\mu \in (0, b)$,
    \begin{equation}\label{eq:corrected_lemma12}
        \mathcal{K}_{\inf}(\bar{\nu}, \mu) \leq \alpha \mathcal{K}_{\inf}(\nu, \mu).
    \end{equation}
    \end{lemma} 
    \begin{proof}
        From \citet[Lemma 9]{tiapkin2024model}, for any probability measure $\nu\in\mathcal{P}[0,b]$ and 
        $\mu\in(0,b)$,
        \begin{equation*}
            \mathcal{K}_{\inf}(\bar{\nu},\mu)
           =\max_{\lambda\in[0,\,1/(b_0-\mu)]}
                \E_{X\sim\bar{\nu}}\!\bigl[\log\!\bigl(1-\lambda\,(X-\mu)\bigr)\bigr].
        \end{equation*}
        The support of $\bar{\nu}$ is contained in $[0,b_0]$, so for any $\lambda\in[0,1/(b_0-\mu)]$,
        \[
            \E_{X\sim\bar{\nu}} \bigl[\log\!\bigl(1-\lambda\,(X-\mu)\bigr)\bigr] =
              (1-\alpha)\log \bigl(1-\lambda(b_0-\mu)\bigr)
                    +\alpha
                      \E_{X\sim\nu} \bigl[\log \bigl(1-\lambda(X-\mu)\bigr)\bigr].
        \]
        For every admissible $\lambda$ we have
        $0\le\lambda(b_0-\mu)\le 1$, so
        $\log\bigl(1-\lambda(b_0-\mu)\bigr)\le 0$.
        Hence
        \[
        \mathcal{K}_{\inf}(\bar{\nu},\mu)\le
        \alpha 
        \E_{X\sim\nu} \bigl[\log\!\bigl(1-\lambda(X-\mu)\bigr)\bigr].
        \]
        Because $b_0>b$, the interval
        $[0,\,1/(b_0-\mu)]$ is a subset of $[0,\,1/(b-\mu)]$.
        Taking the maximum over the smaller interval leads to
        \[
        \mathcal{K}_{\inf}(\bar\nu,\mu)\le 
          \alpha 
          \max_{0\le\lambda\le 1/(b-\mu)}
               \E_{X\sim\nu} \bigl[\log\!\bigl(1-\lambda(X-\mu)\bigr)\bigr]
           =\alpha\,\mathcal{K}_{\inf}(\nu,\mu).
        \]
        \end{proof}

%% file: appendix/Q_reformulation.tex
\section{Reformulation of the Update Equation and Aggregated Weights}
 In this section, we rewrite the update of each temporary $Q$-function for every trajectory $t\in[T]$ and $j\in[J]$ by recursively unrolling the update \eqref{eq:Q_circ_update} and \eqref{eq:Q_dagger_update}, for each $(h,s,a)\in [H]\times S\times A$. 
 
For the ease of notation, we denote $m \defeq {n}^t_h(s,a)$.  

 \paragraph{Unrolled update for $\widetilde{Q}_h^{j,t}$.} For each $(j,t,h)\in[J]\times[T]\times[H]$, we can unroll \eqref{eq:Q_circ_update} by
\begin{align*}
    \widetilde{Q}_h^{j,t}(s,a) =  \widetilde{Q}_h^{j,\ell^m +1}(s,a) &= (1 - w_{m-1}^{j}) \cdot \widetilde{Q}_h^{j,\ell^{m- 1}+1}(s,a) + w_{m-1}^{j} \left[r_h(s,a) + \tildeV_{h+1}^{\ell^{m}}(s_{h+1}^{{\ell^{m}}})\right],\\
    \widetilde{Q}_h^{j,\ell^{m -1}+1}(s,a) &= (1 - w_{m-2}^{j}) \cdot \widetilde{Q}_h^{j,\ell^{m-2}+1}(s,a) + w_{m-2}^{j} \left[r_h(s,a) + \tildeV_{h+1}^{\ell^{m-1}}(s_{h+1}^{\ell^{m-1}})\right],\\
    &\vdots\\
    \widetilde{Q}_h^{j,\ell^1+1}(s,a) &= (1-w_{0}^{j})\cdot \widetilde{Q}_h^{j,\ell^0+1}(s,a) +  w_{0}^{j} \left[r_h(s,a) + \tildeV_{h+1}^{\ell^1}(s_{h+1}^{\ell^1})\right],\\
    \widetilde{Q}_h^{j,\ell^0+1}(s,a) &= r_h(s,a) + \tildeV_{h+1}^{\ell^0}(s_{h+1}^{\ell^0}),
\end{align*}
where we define $\tildeV_{h+1}^{\ell^0}(s_{h+1}^{\ell^0})= V_{h+1}^{0}$.  For $m \ge 1$, let $W_{j,m} = \left( W_{j,m}^m,\ldots, W_{j,m}^1,W_{j,m}^0\right)$ be the aggregated weights defined as
\begin{equation}\label{eq:W_def}
    W_{j,m}^{0} = \prod_{k=0}^{m-1} (1-w_{k}^{j}) \quad \text{and}\quad W_{j,m}^{i} = w_{i-1}^{j}\prod_{k=i}^{m-1} (1-w_{k}^{j}), \quad\forall i\in[m],
\end{equation}
where $w_k^j$ is sampled from Beta $(\frac{H+1}{\kappa}, \frac{k+n_0}{\kappa})$.
Then, we have
\begin{equation}\label{eq:Q_circ_recursion_update}
    \widetilde{Q}_h^{j,t}(s,a) = r_h(s,a) + \sum_{i=0}^{m} W_{j,m}^{i} \left[\tildeV_{h+1}^{\ell^{i}}(s_{h+1}^{\ell^{i}})\right].
\end{equation}

\paragraph{Unrolled update for $\widetilde{Q}_h^{\aux,t}$.} Suppose that $q$ is the stage index of the episode $t$ on $(h,s,a)$. We let $e_q$ be the length of the $q$-th stage.

Similar to the above unrolling steps for $\widetilde{Q}_h^{j,t}$, we define the corresponding aggregated weights as:
for the $q$-th stage, let $W_{j,q}^{\aux} = \left( W_{j,e_{q}}^{\aux,0}, W_{j,e_{q}}^{\aux,1},\ldots,W_{j,e_{q}}^{\aux,e_{q}}\right)$ be the aggregated weights at the defined as
\begin{equation}\label{eq:W_aux_def}
    W_{j,q}^{\aux,0} = \prod_{k=0}^{e_q-1} (1-w_{k,q}^{\aux,j}) \quad \text{and}\quad W_{j,q}^{\aux,i} = w_{i-1}^{\aux,j}\prod_{k=i}^{e_q-1} (1-w_{k,q}^{\aux,j}), \quad\forall i\in[e_q],
\end{equation}
where $w_{k,q}^{\aux,j}$ is sampled from $\B(\frac{1}{\kappa^{\aux}},\frac{k+n_{0}^{\aux}}{\kappa^{\aux}})$.

For each $(t,h)\in [T]\times[H]$, as shown in \eqref{eq:Q_dagger_update}, $\widetilde{Q}_h^{\aux,t}(s,a)$ is updated via the most recent $e_{q-1}$ samples before the $q$-th stage. Thus, for any $q\ge 1$, we have
\begin{equation}\label{eq:Q_dagger_recursion_update}
    \widetilde{Q}_h^{\aux,t}(s,a) = 
    r_h(s,a)+\max_{j\in[J]} \left\{ \sum_{i=0}^{e_{q-1}} W_{j,q-1}^{\aux,i} \left[\tildeV_{h+1}^{\aux,\ell^{\aux,i}_{q-1}}(s_{h+1}^{\ell^{\aux,i}_{q-1}})\right]\right\}. 
\end{equation}

\subsection{Properties of aggregated weights} \label{appendix:weight_property}
 
In this section, we mainly discuss the properties of the aggregated weights $W_{j,m}$ for any $m\ge 1$. 

To begin with, from \eqref{eq:W_def}, we can verify that the sum of the aggregated weights is equal to 1, i.e.,
\begin{equation}
\label{eqn:agg-sum-one}
    \sum_{i=0}^{m} W_{j,m}^i =  \prod_{k=0}^{m-1} (1-w_{k}^j) + \sum_{i=1}^m w_{i-1}^j\prod_{k=i}^{m-1} (1-w_{k}^j) =  1.
\end{equation}
In the following proposition, we further show some desirable properties regarding the expectation and variance of the aggregated weights $W_{j,m}$.
\begin{proposition}\label{prop:aggregated_weight}
    The following properties hold for the aggregated weights $W_{j,m}$, $\forall j\in[J], m\ge 1$:
    \begin{enumerate}[label=(\roman*), leftmargin=2em]
        \item The moment of the aggregated weights is given by:
        \begin{equation}\label{eq:prop_moment}
            \mathbb{E}[(W_{j,m}^i)^d] = \left(\prod_{j=i+1}^m\frac{(\frac{n_0+j-1}{\kappa})_d}{(\frac{H+n_0+j}{\kappa})_d}\right)\cdot\frac{(\frac{H+1}{\kappa})_d}{(\frac{H+n_0+i}{\kappa})_d};
        \end{equation}
        \item The upper bound of expectations and the sum of variances:
            \begin{align}\label{eq:prop_max_exp_and_var}
                \max_{i\in[m]} \mE[W_{j,m}^i] &\le \frac{H+1}{H+n_0+m},\qquad \sum_{i=1}^m\Var[W_{j,m}^i] \le \frac{(H+1)\kappa}{H+n_0+m};
            \end{align}
        \item For every $i\ge 1$, we have 
        \begin{equation}\label{eq:prop_sum_expectation}
            \sum_{t=i}^{\infty} \E[W_{j,t}^i] \le 1+\frac{1}{H}.
        \end{equation} 
    \end{enumerate}
\end{proposition}

\begin{proof}
    We prove each item separately.
    \begin{enumerate}[label=(\roman*), leftmargin=2em]
    \item Directly follows from \citet[Section 2]{wong1998generalized}.
    \item Similar to \citet[Lemma 4.1]{jin2018q}, we have that for $i\in[m]$
    \begin{align*}
        \mE[W_{j,m}^i] &= \frac{H+1}{H+n_0+i} \left(\frac{n_0+i}{H+n_0+i+1}\frac{n_0+i+1}{H+n_0+i+2}\cdots \frac{n_0+m-1}{H+n_0+m}\right)\\
        &=\frac{H+1}{H+n_0+m} \left(\frac{n_0+i}{H+n_0+i}\frac{n_0+i+1}{H+n_0+i+1}\cdots \frac{n_0+m-1}{H+n_0+m-1}\right)\\
        &\le \frac{H+1}{H+n_0+m}.
    \end{align*}
    Thus, $\max_{i\in[m]} \mE[W_{j,m}^i] \le \frac{H+1}{H+n_0+m}$ holds. From Lemma \ref{lemma:fraction_plus_k} and \eqref{eq:prop_moment},
   \begin{align*}
        \Var[W_{j,m}^i] &= \E[(W_{j,m}^i)^2] - \E[W_{j,m}^i]^2\\
        &= \E[W_{j,m}^i] \left(\left(\prod_{j=i+1}^m\frac{n_0+j-1+\kappa}{H+n_0+j+\kappa}\right)\cdot\frac{H+1+\kappa}{H+n_0+i+\kappa} - \E[W_{j,m}^i] \right)\\
        & \le \E[W_{j,m}^i]\left(\frac{H+1+\kappa}{H+n_0+m+\kappa} - \E[W_{j,m}^i]\right)\\
        &\le \kappa \mE[W_{j,m}^i] \cdot \frac{H+1}{H+n_0+m},
    \end{align*}
    and thus
    \begin{align*}
        \sum_{i=1}^m\Var[W_{j,m}^i] & \le \kappa \cdot \frac{H+1}{H+n_0+m} \cdot \left(\sum_{i=1}^m \mE[W_{j,m}^i]\right) \le \frac{(H+1)\kappa}{H+n_0+m},
    \end{align*}
    where the last inequality used $\sum_{i=1}^m \mE[W_{j,m}^i] \le \sum_{i=0}^m \mE[W_{j,m}^i] = 1$ (recalling \eqref{eqn:agg-sum-one}).
    \item Following \citet[equation (B.1)]{jin2018q}, it also holds for any positive integer $n,k$ and $n\ge k$
    \begin{align}\label{eq:modified_n_divided}
    \frac{n+n_0}{k} = 1 + \frac{n+n_0 - k}{n+n_0 + 1}  + \frac{(n+n_0 - k)(n+n_0 - k + 1)}{(n+n_0 + 1)(n+n_0 + 2)} + \cdots
    \end{align}
    which can be verified by induction. Specifically, we let the terms of the right-hand side be $x_0 = 1$, $x_1 = \frac{n+n_0 - k}{n+n_0 + 1}$, \ldots. Then, we will show $\frac{n+n_0}{k} - \sum_{j=0}^i x_j = \frac{n+n_0-k}{k}\prod_{j=1}^i \frac{n+n_0-k+j}{n+n_0+j}$ by induction. 
    \begin{itemize}
        \item Base case when $i=1$: It can easily verified that 
        \begin{equation*}
            \frac{n+n_0}{k} - 1 - \frac{n+n_0 - k}{n+n_0 + 1} = \frac{(n+n_0-k)(n+n_0+1-k)}{k(n+n_0 + 1)}
        \end{equation*}
        \item Suppose $i=r$, the claim holds, i.e., $\frac{n+n_0}{k} - \sum_{j=0}^r x_j = \frac{n+n_0-k}{k}\prod_{j=1}^r \frac{n+n_0-k+j}{n+n_0+j}$. Then, for $i=r+1$, we have
        \begin{align*}
            & \frac{n+n_0}{k} - \sum_{j=0}^r x_j -x_{r+1} = \frac{n+n_0-k}{k}\prod_{j=1}^r \frac{n+n_0-k+j}{n+n_0+j} - \prod_{j=1}^{r+1}\frac{n+n_0 -k+j-1}{n+n_0 +j}\\
            & = \frac{n+n_0-k}{k}\;
            \prod_{j=1}^{r}\frac{n+n_0-k+j}{\,n+n_0+j\,}
          \Bigl[
               1-\frac{k}{n+n_0-k} \cdot \frac{n+n_0-k}{n+n_0+r+1}
          \Bigr]\\
          &=
          \frac{n+n_0-k}{k}\;
            \prod_{j=1}^{r}\frac{n+n_0-k+j}{\,n+n_0+j\,}
            \frac{n+n_0-k+r+1}{\,n+n_0+r+1\,}\\
          &=
          \frac{n+n_0-k}{k}\;
            \prod_{j=1}^{r+1}\frac{n+n_0-k+j}{\,n+n_0+j\,}.
        \end{align*}
    \end{itemize}
    By letting $i\to\infty$, we obtain \eqref{eq:modified_n_divided}.
    Thus, we have
    \begin{align*}
        \sum_{t=i}^{\infty} \E[W_{j,t}^i] & = \frac{H+1}{H+n_0+i}\left( 1+ \frac{i+n_0}{H+n_0+i+1} + \frac{i+n_0}{H+n_0+i+1}\frac{i+n_0+1}{H+n_0+i+2} + \ldots\right)\\
        & = \frac{H+1}{H+n_0+i}\frac{H+n_0+i}{H} = 1+\frac{1}{H},
    \end{align*}
    where the second equality uses \eqref{eq:modified_n_divided} with $n = i+H$ and $k=H$. 
    \end{enumerate}
\end{proof}

%% file: appendix/concentration.tex
\subsection{Concentration of aggregated weights}\label{appendix:concentration}
For notational convenience and clearness, we slightly abuse the notation by ignoring the dependency on the ensemble index $j$, while the following concentration lemma of the aggregated weights holds for any $j\in[J]$. 

\begin{lemma}\label{lemma:concentration_aggregated_GD}
    Let $\lambda_0, \lambda_1,\lambda_2, \cdots, \lambda_m$ be nonnegative real numbers such that $\lambda_i \le 1$, $i=1,2,\cdots,m$. Then, with probability at least $1-\frac{\delta}{T}$, we have 
    \begin{equation*}
        \left\vert \sum_{i=0}^m \lambda_i \left(W^i_m - \mathbb{E}[W^i_m]\right) \right\vert\le c_1\sqrt{\frac{(H+1)\kappa^2\log^3(T/\delta)}{H + n_0 + m}} + c_2 \frac{(H+1)\kappa\log^2(T/\delta)}{H+n_0+m},
    \end{equation*}
    where $c_1$ and $c_2$ are some positive universal constants.
\end{lemma}

\begin{proof}
    By definition, we can find independent random variables $w_0, \cdots, w_{m-1}$ such that $w_i \sim \B(\frac{H + 1}{\kappa}, \frac{n_0 + i}{\kappa})$, and 
    \begin{equation}\label{eq:concentration_W_def}
        W^0_m = \prod_{k=1}^{m-1} (1 - w_k), \quad W^i_m = w_{i-1} \prod_{k=i}^{m-1} (1 - w_k), \quad 1 \le i \le m.
    \end{equation}

    Let $\F_{-i}$ be the $\sigma$-algebra generated by $w_{m-1}, w_{m-2}, \cdots, w_{m-i}$, with the convention that $\F_{-0} = \F_0$ is the trivial $\sigma$-algebra. In the same spirit, denote
    \begin{equation}\label{eq:S_i_def}
        S_{-i} = \sum_{k=m-i+1}^{m} \lambda_k (W_m^{k} - \E[W_m^{k}]), \quad i=0, 1, \cdots, m.
    \end{equation}
    For conceptual reason, we will use the notation $S_{-\infty}$ to denote the sum $\sum_{k=0}^m \lambda_k (W_m^{k} - \E[W_m^{k}])$, which corresponds to $S_{-(m+1)}$ in the above notation.

    We view $\F_0 \subset \F_{-1} \subset \cdots \subset \F_{-m}$ as a reverse filtration, and consider the backward martingale
    \begin{equation}\label{eq:M_i_def}
        M_{-i} \coloneqq \E[S_{-\infty} | \F_{-i}], \quad i=0, 1, \cdots, m. 
    \end{equation}
    It is clear that $M_0 = \E[S_{-\infty}]$, while $M_{-m} = S_{-\infty}$. Therefore
    \[
    \sum_{i=0}^m \lambda_i \left(W_m^{i} - \mathbb{E}[W_m^{i}]\right) =  S_{-\infty} - \E[S_{-\infty}] = M_{-m} - M_0.
    \]

    We may then apply the Rosenthal's inequality (\ie, Lemma~\ref{lemma:rosenthal_ineq}) to obtain
    \begin{align}
    \left( \E [ \left| M_{-m} - M_0 \right|^p ] \right)^{1/p} 
    & \le C \sqrt{p} \left( \E [\langle M \rangle_{-m}^{p/2}] \right)^{1/p} + C  p \left( \E [\max_{1 \le i \le m} |M_{-i} - M_{-(i-1)}|^p] \right)^{1/p} 
    \nonumber\\
    & \le C \sqrt{p} \left( \E [\langle M \rangle_{-m}^{p/2}] \right)^{1/p} + C  p \left( \sum_{i=1}^m \E [|M_{-i} - M_{-(i-1)}|^p] \right)^{1/p},\label{eq:rosenthal_ineq}
    \end{align}
    where 
    \[
    \langle M \rangle_{-m} \coloneqq \sum_{i=1}^m \E[(M_{-i} - M_{-(i-1)})^2 | \F_{-(i-1)}].
    \]

    To proceed, we calculate the martingale difference $M_{-i} - M_{-(i-1)}$ for any $i\in[m]$. To simplify the resulting expressions, we will denote
    \[
    T_{-i} \coloneqq \lambda_0 \prod_{j=0}^{m-i-1} (1 - w_j) +  \sum_{k=1}^{m-i} \left( \lambda_k w_{k-1} \prod_{j=k}^{m-i-1} (1 - w_j) \right).
    \]
    With this notation in hand, we can decompose $S_{-\infty}$ by
    \begin{align*}
        S_{-\infty} &= \lambda_{m-i+1}\left(W_m^{m-i+1} - \mE[W_m^{m-i+1}]\right) + \sum_{k=0}^{m-i}\lambda_k(W_m^{k} - \mE[W_m^{k}])+ \!\!\!\sum_{k=m-i+2}^{m} \!\!\!\lambda_k\left(W_m^{k}-\mE[W_m^{k}]\right)\nonumber\\
        &= \lambda_{m-i+1}\left(w_{m-i}\prod_{k=m-(i-1)}^{m-1} (1 - w_k)-\E\left[w_{m-i}\prod_{k=m-(i-1)}^{m-1} (1 - w_k)\right]\right)\nonumber\\ 
        & \quad+ \left(T_{-i} \prod_{k=m-i}^{m-1} (1 - w_k) -\mE\left[T_{-i}\prod_{k=m-i}^{m-1} (1 - w_k)\right]\right)\nonumber\\
        & \quad+ \sum_{k=m-i+2}^{m} \lambda_k\left(W_m^{k}-\mE[W_m^{k}]\right) ,
    \end{align*}
    where only the first two terms involve the observation $w_{m-i} = \F_{-i} \setminus \F_{-(i-1)}$.
    Then for $i\in[m]$, 
    \begin{align}
        M_{-i} - M_{-(i-1)} 
        & = \E[S_{-\infty} | \F_{-i}] - \E[S_{-\infty} | \F_{-(i-1)}]\\
        & = \lambda_{m-i+1}\left(w_{m-i}-\E\left[w_{m-i}\right]\right)\prod_{k=m-(i-1)}^{m-1} (1 - w_k) \nonumber\\
        & + \E[T_{-i}]\prod_{k=m-(i-1)}^{m-1} (1 - w_k)\left((1 - w_{m-i}) - \mE[1 - w_{m-i}]\right)\nonumber\\
        &= \left( \lambda_{m - i + 1} - \E [T_{-i}] \right) \cdot (w_{m-i} - \E[w_{m-i}]) \prod_{k=m-(i-1)}^{m-1} (1 - w_k), \label{eq:martingale-diff}
    \end{align}
    where the last line is from $\mE[1-w_{m-i}] = 1-\mE[w_{m-i}]$.
    
    By our assumption $|\lambda_i| \le 1$, $i=1,\cdots, m$, it can be readily checked that
    \begin{equation*}
        |T_{-i}| \le  \prod_{j=0}^{m-i-1} (1 - w_j) +  \sum_{k=1}^{m-i} \left( w_{k-1} \prod_{j=k}^{m-i-1} (1 - w_j) \right) = 1, \quad \forall i=\{1,\ldots,m\}.
    \end{equation*}
    Consequently, the absolute value of the martingale difference can be bounded above by
    \begin{equation}
    \label{ineq:martingale-diff}
        | M_{-i} - M_{-(i-1)} | \le 2 |w_{m-i} - \E[w_{m-i}] | \prod_{k=m-(i-1)}^{m-1} (1 - w_k).
    \end{equation}

    Recall that $\langle M \rangle_{-m} \coloneqq \sum_{i=1}^m \E[(M_{-i} - M_{-(i-1)})^2 | \F_{-(i-1)}]$. Together with \eqref{ineq:martingale-diff}, we have
    \begin{align}\label{eq:M_m_quadratic}
        \langle M \rangle_{-m} 
        & \le 2 \sum_{i=1}^m \mE \left(|w_{m-i} - \E[w_{m-i}] | \prod_{k=m-(i-1)}^{m-1} (1 - w_k) ~\Big\vert~ \F_{-(i-1)}  \right)^2 
        \\
        & = 2 \sum_{i=1}^m \Var[w_{m-i}] \prod_{k=m-(i-1)}^{m-1} (1 - w_k)^2 \nonumber\\
        & = 2 \sum_{i=0}^{m-1} \Var[w_i] \prod_{k=i+1}^{m-1} (1 - w_k)^2 \nonumber\\
        & \le 2 \sum_{i=0}^{m-1} \frac{\kappa (H + 1) (n_0 + i)}{(H + 1 + n_0 + i)^2 (H + 1 + n_0 + i + \kappa)} \prod_{k=i+1}^{m-1} (1 - w_k)^2,
    \end{align}
    where the second equality is due to the change of index, and the last inequality follows from \eqref{eq:beta_expectation_variance}.
    We may then apply triangle inequality to obtain 
    \begin{align*}
        \E^{2/p} [\langle M \rangle_{-m}^{p/2}]  &\le  \sum_{i=0}^{m-1} 2 \frac{\kappa (H + 1) (n_0 + i)}{(H + 1 + n_0 + i)^2 (H + 1 + n_0 + i + \kappa)} \E^{2/p} \left[ \left( \prod_{k=i+1}^{m-1} (1 - w_k)^2 \right)^{p/2} \right] \\
        &= 2 \sum_{i=0}^{m-1} \frac{\kappa (H + 1) (n_0 + i)}{(H + 1 + n_0 + i)^2 (H + 1 + n_0 + i + \kappa)} \left( \prod_{k=i+1}^{m-1} \E \left[ (1 - w_k)^p \right] \right)^{2/p},
    \end{align*}
     for $p\ge2$.
    Note that $1-w_k \sim \mathrm{Beta}(\frac{n_0+k}{\kappa}, \frac{H+1}{\kappa})$, directly from the definition \ref{def:beta}. Thus,
     By \eqref{eq:beta_moment}, we have
    \begin{equation*}
        \E \prod_{k=i+1}^{m-1} (1 - w_k)^p = \prod_{k=i+2}^m\frac{(\frac{n_0+k-1}{\kappa})_{p}}{(\frac{H+n_0+k}{\kappa})_{p}} \le \prod_{r=0}^{p-1} \frac{n_0 + i + 1 + r\kappa}{H + n_0 + m + r\kappa} \le p! \left( \frac{n_0 + i + 1 + \kappa}{H + n_0 + m + \kappa} \right)^{p},\label{eq:prod_residual}
    \end{equation*}
    where the last inequality is from Lemma \ref{lemma:fraction_plus_k}. Therefore, we can obtain
    \begin{align*}
        \left( \E[\langle M \rangle_{-m}^{p/2}] \right)^{2/p} 
        & \le 2 \sum_{i=0}^{m-1} \frac{\kappa (H + 1) (n_0 + i)}{(H + 1 + n_0 + i)^2 (H + 1 + n_0 + i + \kappa)} (p!)^{2/p} \left( \frac{n_0 + i + 1 + \kappa}{H + n_0 + m + \kappa} \right)^2
        \\
        &\le 2\kappa p^2\frac{H+1+\kappa}{(H+n_0+m+\kappa)^2} \sum_{i=0}^{m-1}\frac{n_0+i+\kappa}{H+1+n_0+i+\kappa}\\
        & \le 2 \kappa p^2 \frac{H+1+\kappa}{H + n_0 + m +\kappa} \le 2(\kappa+1)^2 p^2 \frac{H+1}{H + n_0 + m}
    \end{align*}
    where the first and last inequality are due to Lemma \ref{lemma:fraction_plus_k} (i.e., \eqref{eq:fraction_plus_k_1}), the second inequality uses the facts that $p! \le p^p$ and $\frac{n_0 + i + 1 + \kappa}{H+1+n_0+i+\kappa}\le 1$ for every $i=0,\ldots,m-1$, and the third inequality is from the fact that the summation term is less than $m$.

    Therefore, 
    \begin{equation}\label{eq:rosenthal_first_term}
        \left( \E [\langle M \rangle_{-m}^{p/2}] \right)^{1/p} \le 2p(\kappa+1)\sqrt{\frac{ (H+1)}{H + n_0 + m}}.
    \end{equation}

    We turn to bound $\E [|M_{-i} - M_{-(i-1)}|^p]$ for $i\in[m]$. It is clear from \eqref{ineq:martingale-diff} that 
    \begin{align*}
    \E [|M_{-i} - M_{-(i-1)}|^p] &\le \E\left[\left|2 |w_{m-i} - \E[w_{m-i}] | \prod_{k=m-(i-1)}^{m-1} (1 - w_k)\right|^p\right]\\
    &\le 8^p \E[w_{m-i}^p] \prod_{k=m-(i-1)}^{m-1} \E [(1 - w_k)^p]\\
    &= 8^p \E[(W_m^{m-i+1})^p] 
    \end{align*}
    where the second inequality uses $|w_{m-i} - \E[w_{m-i}]|\le 2^{p+1}w_{m-i}^p$, and the equality is from \eqref{eq:concentration_W_def} where $W_m^{m-i+1} = w_{m-i}\prod_{k=m-(i-1)}^{m-1} (1 - w_k)$ for $i\in[m]$.
    
     From Proposition \ref{prop:aggregated_weight}, we have that for any $i\in[m]$,
    \begin{align*}
        \mE \left[ (W_m^i)^p\right] &= \left(\prod_{j=i+1}^m\frac{(\frac{n_0+j-1}{\kappa})_{p}}{(\frac{H+n_0+j}{\kappa})_{p}}\right)\cdot\frac{(\frac{H+1}{\kappa})_{p}}{(\frac{H+n_0+i}{\kappa})_{p}}\\
        &\le \prod_{r=0}^{ p-1}\left(\frac{H+1+r\kappa}{H+n_0+i+r\kappa} \cdot \prod_{j=i+1}^m\frac{n_0+j-1+r\kappa}{H+n_0+j+r\kappa}\right)\\
        & \le \prod_{r=0}^{p-1} \frac{r\kappa(H+1)}{H+n_0+m} \le p! \left(\frac{(H+1)\kappa}{H+n_0+m}\right)^p
    \end{align*}
    Thus, 
    \begin{equation}\label{eq:rosenthal_second_term}
        \left( \sum_{i=1}^m \E [|M_{-i} - M_{-(i-1)}|^p] \right)^{1/p} \le \left( 8^p \sum_{i=1}^m\E[(W_m^{m-i+1})^p] \right)^{1/p}\le 8 m^{1/p} p \frac{(H+1)\kappa}{H+n_0+m}.
    \end{equation}

    Substituting \eqref{eq:rosenthal_first_term} and \eqref{eq:rosenthal_second_term} into \eqref{eq:rosenthal_ineq} leads to 
    \begin{align*}
         \E^{1/p} [ \left| M_{-m} - M_0 \right|^p ]  
        &\le 2C \sqrt{\frac{p^3 (\kappa+1)^2(H+1)}{H + n_0 + m}} + 8C\cdot m^{1/p} p^2 \frac{(H+1)\kappa}{H+n_0+m}.
    \end{align*}
    Note that by definition, $m = n_h^t(s,a)\le T$ always holds for any $(h,t,s,a)\in[H]\times[T]\times\S\times\A$. Let $p = \lceil\log(T/\delta)\rceil\ge 2$ and thus $m^{1/p} \le \mathrm e$ since $m^{1/p}\le \mathrm e^{(\log T)/p}$.  
    Finally, by Markov inequality with $t = 2\mathrm e C\sqrt{\frac{(H+1)(\kappa+1)^2\log^3(T/\delta)}{H + n_0 + m}} + 8\mathrm e^2 C \frac{(H+1)\kappa\log^2(T/\delta)}{H+n_0+m}$,  we have
    \begin{align*}
        \bP\!\!\left[\vert\sum_{i=0}^m \lambda_i \left(W_m^{i} - \mathbb{E}[W_m^{i}]\right) \vert \ge t\right] &\!\le\! \left(\frac{\mathbb{E}^{1/p} \left[ \left| \sum_{i=0}^m \lambda_i \left(W_m^{i} - \mathbb{E}[W_m^{i}]\right) \right|^p \right]}{t}\right)^p \!\!\le (\frac{1}{\mathrm e})^{\lceil\log(T/\delta)\rceil}\le \frac{\delta}{T},
    \end{align*}
    which finishes the proof.
\end{proof}
In addition, we also have the following lemma regarding the concentration of $W_{q}^{\aux}$ for any stage $q\ge 0$. We omit its proof since it is similar to Lemma~\ref{lemma:concentration_aggregated_GD} .
\begin{lemma}\label{lemma:concentration_standard_Dir}
    Let $\lambda_0, \lambda_1,\lambda_2, \cdots, \lambda_{e_q}$ be nonnegative real numbers such that $\lambda_i \le 1$, $i=1,2,\cdots,e_q$. Then, with probability at least $1-\delta/T$, we have 
    \begin{equation*}
        \left\vert \sum_{i=0}^{e_q} \lambda_i \left(W^{\aux,i}_q - \mathbb{E}[W^{\aux,i}_q]\right) \right\vert\le c_1^{\aux}\sqrt{\frac{(\kappa^{\aux})^2\log^3(T/\delta)}{1+n_{0}^{\aux} + e_q}} + c_2^{\aux} \frac{\kappa^{\aux}\log^2(T/\delta)}{1+n_{0}^{\aux}+e_q},
    \end{equation*}
    where $c_1^{\aux}$ and $c_2^{\aux}$ are some positive universal constants.
\end{lemma}

%% file: appendix/proof_thm1.tex
\section{Analysis: Gap-independent Regret Bound}\label{appendix:regret_proof}
In this section, we present the detailed proof of Theorem \ref{thm:regret_main}. Before proceeding, we first rewrite the update of policy Q-function by unrolling the updates of temporary Q-functions (\ie, equations \eqref{eq:Q_circ_recursion_update} and \eqref{eq:Q_dagger_recursion_update}) as 
\begin{align}
    Q_h^{t}(s_h^t, a_h^t) &= \eta_{\ell^{n_h^t},h}\max_{j\in[J]}\Big\{\widetilde{Q}^{j,t}_{h}(s_{h}^t, a_{h}^t)\Big\} + (1-\eta_{\ell^{n_h^t},h})\cdot\widetilde{Q}_{h}^{\aux,t}(s_{h}^t, a_{h}^t) \label{eq:policy_Q_with_t}\\
    &= r_{h}(s_h^t, a_h^t)+ \eta_{\ell^{n_h^t},h} \max_{j\in[J]}\Big\{\sum_{i=0}^{n^t_h} W_{j,n^t_h}^{i} \tildeV_{h+1}^{\ell^{i}}(s_{h+1}^{\ell^{i}})\Big\} \nonumber\\
    &\quad +  (1-\eta_{\ell^{n_h^t},h})\cdot \max_{j\in[J]}\Big\{ \sum_{i=0}^{e_{q-1}} W_{j,q-1}^{\aux,i} \tildeV_{h+1}^{\aux,\ell^{\aux,i}_{q-1}}(s_{h+1}^{\ell^{\aux,i}_{q-1}})\Big\}, \label{eq:policy_Q_recursion}
\end{align}
for each $(t,h)\in[T]\times[H]$, where in the notations we omit the dependency of $q$ regarding the step $h$ and the episode $t$ for simplicity and let $n^{t}_h$ be the number of visits of the state-action pair $(s_h^t, a_h^t)$ before episode $t$ at step $h$ such that $\ell^{n_h^t}$ be the index of the last visit of the state-action pair $(s_h^t, a_h^t)$ at step $h$. The properties and the concentration inequality of the aggregated weights (\ie, $W_{j,n_h^t}$ and $W_{j,q-1}^{\aux}$), proved in Appendix \ref{appendix:weight_property} and \ref{appendix:concentration}, will play a crucial role in the following analysis.

\subsection{Proof of Theorem \ref{thm:regret_main}}
To control the total regret, we first present the following lemma regarding the optimism of the policy Q-function $Q_h^{t}$ for any $(t,h,s,a)\in[T]\times[H]\times\S\times\A$, where the detailed proof is postponed to Appendix \ref{appendix:optimism_proof}. 
\begin{lemma}[Optimism]\label{lemma:optimism}
    Consider $\delta\in (0,1)$.  Assume that $J = \lceil{c\cdot\log(SAHT/\delta)}\rceil$, $\kappa^{\aux} = c\cdot(\log(SAH/\delta) + \log(T))$, and $n_0^{\aux} = \lceil c\cdot\log(T)\cdot \kappa^{\aux} \rceil$, where $c$ is some universal constant. Let the initialized value function $V_h^0 = 2(H-h+1)$ and the mixing rate $\eta_{t,h} \in (0,1)$ for every $(t,h)\in \{0,1,\ldots,T\}\times[H]$. Then,  for any $(t,h,s,a)\in[T]\times[H]\times\S\times\A$, with probability at least $1-\delta$, the following event holds
    \begin{equation*}
        0 < (1-\eta_{\ell^{n_h^t},h})\cdot Q^{\star}_{h}(s,a) \le  Q_{h}^t (s,a).
    \end{equation*}
\end{lemma}
Before proceeding, we first define the value function induced by the policy Q-function $Q_h^{t}$ as
\begin{equation}\label{eq:policy_V_function}
    V_h^{t}(s_h^t) \defeq Q_h^{t}(s_h^t,\pi_h^t(s_h^t)) = \max_{a\in\A} Q_h^{t}(s_h^t,a).
\end{equation}
 Note that Lemma \ref{lemma:optimism} implies that for every $t\in[T]$,
\begin{align}
    V_1^{\star}(s_1^t) = \max_a Q^{\star}_{1}(s_1^t,a) &\le \frac{1}{1-\eta_{\ell^{n_h^t},h}}\max_a Q^{t}_{1}(s_1^t,a)\nonumber\\
    &= \frac{1}{1-\eta_{\ell^{n_h^t},h}} Q^{t}_1(s_1^t,a_1^t) = \frac{1}{1-\eta_{\ell^{n_h^t},h}} V_1^{t}(s_1^t).\label{eq:optimism_imply}
\end{align}
Thus, we can decompose the total regret as
\begin{equation*}
    \mathrm{Regret}_T = \sum_{t=1}^T (V_1^{\star} - V_1^{\pi^t})(s_1^t) \le \sum_{t=1}^T (\frac{1}{1-\eta_{\ell^{n_1^t},1}} V_1^t - V_1^{\pi^t})(s_1^t) = \sum_{t=1}^T \delta_1^t +\sum_{t=0}^{T-1} \frac{\eta_{t,1}}{1-\eta_{t,1}} H,
\end{equation*}
where we denote the performance gap as $\delta_h^t = (V_h^t - V_h^{\pi^t})(s_h^t)$ for every $(t,h)\in[T]\times[H]$. 

Note that $Q_h^{\pi^t}(s_h^t,a_h^t) = V_h^{\pi^t}(s_h^t)$ . Then, for any fixed episode $t\in[T]$ and step $h\in[H]$, we decompose the performance gap $\delta_h^t$:
\begin{align*}
    \delta_h^t 
    &\le \left({Q}^t_h - Q^{\star}_h\right)(s_h^t,a_h^t) + \left(Q^{\star}_h - Q_h^{\pi_t}\right)(s_h^t,a_h^t)\\
    &\le \left({Q}^t_h - Q^{\star}_h\right)(s_h^t,a_h^t) +P_{h,s_h^t,a_h^t}\left(V_{h+1}^{\star} - V_{h+1}^{\pi^t}\right)\\
    &\le \left({Q}^t_h - Q^{\star}_h\right)(s_h^t,a_h^t) +\underset{=:\delta_{h+1}^t}{\underbrace{\left({V}_{h+1}^t - V_{h+1}^{\pi_t}\right)(s_{h+1}^t)}} - \underset{=:\xi_{h+1}^t}{\underbrace{\left({V}_{h+1}^t - V_{h+1}^{\star}\right)(s_{h+1}^t)}}\\
    &\quad + \underset{=:\tau_{h+1}^t}{\underbrace{P_{h,s_h^t,a_h^t}\left(V_{h+1}^{\star} - V_{h+1}^{\pi^t}\right) - \left(V_{h+1}^{\star} - V_{h+1}^{\pi^t}\right)(s_{h+1}^t)}},
\end{align*}
where the second inequality is from the Bellman equations \eqref{eq:bellman_consistency} and \eqref{eq:bellman_optimality}. Together with the update \eqref{eq:policy_Q_with_t},
\begin{align}
    \delta_h^t 
    & \le \eta_{\ell^{n_h^t},h} \underset{=:\zeta_h^{t}}{\underbrace{\max_{j\in [J]} \left\{\widetilde{Q}_{h}^{j,t}(s_{h}^t, a_{h}^t)- Q^{\star}_h (s_h^t,a_h^t) +H\right\}}}+ (1-\eta_{\ell^{n_h^t},h})\underset{=:\zeta_h^{\aux,t}}{\underbrace{\left\{\widetilde{Q}_{h}^{\aux,t}(s_{h}^t, a_{h}^t) -  Q^{\star}_h (s_h^t,a_h^t)\right\}}}\nonumber\\
    & \quad + \delta_{h+1}^t - \xi_{h+1}^t + \tau_{h+1}^t -\eta_{\ell^{n_h^t},h}H. \label{eq:delta_decompose}
\end{align}

The main idea of the proof is to show that the performance gap $\delta_h^t$ can be upper bounded by some quantities from the next step $h+1$, and correspondingly, the total regret can be controlled by rolling out the performance gap over all episodes and steps.

Next, the following lemmas present a recursive bound for $\zeta_h^t$ and $\zeta_h^{\aux,t}$, respectively. The proofs are deferred to Appendix \ref{appendix:proof_lemma_recursion_circ} and \ref{appendix:proof_lemma_recursion_dagger}.

\begin{lemma}[Recursive bound for $\zeta_h^{t}$]\label{lemma:recursion_circ} Consider $\delta\in (0,1)$. For any $i\in[m]$, let $\alpha_{m}^{0}\defeq \prod_{k=1}^m \frac{n_0+k-1}{H+n_0+k}$ and $\alpha_m^{i} \defeq \frac{H+1}{H+n_0+i} \prod_{k=i+1}^m \frac{n_0+k-1}{H+n_0+k}$,  where $m= n_h^t(s_h^t,a_h^t)$. Then, for any $(t,h,s,a)\in[T]\times[H]\times\S\times\A$, with probability $1-\delta$, we have
    \begin{align}\label{eq:bounding_zeta_circ}
       0\le \zeta_h^{t} &\le 2\alpha_{m}^0 V_{h+1}^0 + \sum_{i=1}^m \alpha_m^i \left(\left(\tildeV_{h+1}^{\ell^{i}}- V_{h+1}^{\star}\right)(s_{h+1}^{\ell^{i}}) +H\right) + \tilde{b}_h^{t},
    \end{align}
    where $\tilde{b}_h^{t} = \widetilde{O}\left(\sqrt{\frac{(H+1)^3}{H+n_0+m}} + \frac{(H+1)^2}{H+n_0+m}\right)$.
\end{lemma}  

\begin{lemma}[Recursive bound for $\zeta_h^{\aux,t}$]\label{lemma:recursion_dagger} Consider $\delta\in (0,1)$.
For any $(t,h,s,a)\in[T]\times[H]\times\S\times\A$, we let $q = q_h^t(s,a)$ represent the index of the current stage and $e_{q-1}$ be the length of the prior stage for $q\ge 1$. Then for any $t\in[T]$, with probability at least $1-\delta$, we have
    \begin{equation}\label{eq:bounding_zeta_dagger}
        0\le\zeta_h^{\aux,t} \le \left(\sum_{i=0}^{e_{q-1}} \frac{1}{1+e_{q-1}}\left(\tildeV_{h+1}^{\aux,\ell^{\aux,i}_{q-1}} - V_{h+1}^{\star}\right)(s_{h+1}^{\ell^{\aux,i}_{q-1}}) + \tilde{b}^{\aux,t}_h\right)\cdot \mathbbm{1}\{q\ge 1\} +   V_{h+1}^0\cdot\mathbbm{1}\{q=0\},
    \end{equation}
    where $\tilde{b}^{\aux,t}_{h} =\widetilde{O}\left(\sqrt{\frac{(H+1)^2}{e_{q-1}}}\right)$.
\end{lemma}

We denote $\tilde{\xi}_h^{t} = \left(\tildeV_{h}^{t}- V_{h}^{\star}\right)(s_{h}^{t})+H$ and $\tilde{\xi}_h^{\aux,t} = \left(\tildeV_{h}^{\aux,t} - V_{h}^{\star}\right)(s_{h}^{t})$.
Combining Lemma \ref{lemma:recursion_circ} and  Lemma \ref{lemma:recursion_dagger} with \eqref{eq:delta_decompose}, we have that for any $(t,h,s,a)\in[T]\times[H]\times\S\times\A$ within the stage $q_h^t\ge 1$,
\begin{align*}
    \delta_h^t
    \le & 2\alpha_{n_h^t}^0V^0_{h+1} + \eta_{\ell^{n_h^t},h}\sum_{i=1}^{n_h^t} \alpha_{n_h^t}^i \tilde{\xi}_{h+1}^{\ell^i} + \frac{1-\eta_{\ell^{n_h^t},h}}{1+e_{q-1}}\sum_{i=0}^{e_{q-1}}\tilde{\xi}_{h+1}^{\aux,\ell^{\aux,i}_{q-1}} + \eta_{\ell^{n_h^t},h}\tilde{b}_h^{t} + (1-\eta_{\ell^{n_h^t},h})\tilde{b}_h^{\aux,t} \\
    &+ (\delta_{h+1}^t - \tilde{\xi}_{h+1}^t + \tau_{h+1}^t) -\eta_{\ell^{n_h^t},h}H,
\end{align*}
and during the initial stage $ q_h^t= 0$
\begin{equation*}
    \delta_h^t
    \le 2\alpha_{n_h^t}^0V^0_{h+1} + \eta_{\ell^{n_h^t},h}\sum_{i=1}^{n_h^t} \alpha_{n_h^t}^i \tilde{\xi}_{h+1}^{\ell^i} + \tilde{b}_h^{t} + (1-\eta_{\ell^{n_h^t},h})V_{h+1}^0 + (\delta_{h+1}^t - \tilde{\xi}_{h+1}^t + \tau_{h+1}^t)-\eta_{\ell^{n_h^t},h}H.
\end{equation*}

Define $b_h^t = \eta_{\ell^{n_h^t},h}\tilde{b}_h^{t} + (1-\eta_{\ell^{n_h^t},h})\cdot \tilde{b}_h^{\aux,t}\mathbbm{1}\{q_h^t\ge 1\}$.  
Thus, summing over $t$ from $1$ to $T$ leads to
\begin{align}
        \sum_{t=1}^T\delta_h^t 
      &  \le \sum_{t: q_h^t=0}^T \delta_h^t +\sum_{t: q_h^t \ge 1}^T \delta_h^t \nonumber \\
        &\le 2\sum_{t=1}^T\alpha_{n_h^t}^0 V_{h+1}^0 +   \sum_{t=1}^T\eta_{\ell^{n_h^t},h}\sum_{i=1}^m \alpha_{n_h^t}^i \tilde{\xi}_{h+1}^{\ell^i} +  \sum_{t=1}^T\frac{(1-\eta_{\ell^{n_h^t},h})\mathbbm{1}\{q_h^t\ge 1\}}{1+e_{q-1}}\sum_{i=0}^{e_{q-1}}\tilde{\xi}_{h+1}^{\aux,\ell^{\aux,i}_{q-1}} \nonumber\\
        & \quad+ \sum_{t=1}^T (b_h^{t} + \delta_{h+1}^t - \tilde{\xi}_{h+1}^t + \tau_{h+1}^t-\eta_{\ell^{n_h^t},h}H) + \sum_{t: q_h^t=0}^T  (1-\eta_{\ell^{n_h^t},h})V_{h+1}^0.\label{eq:sum_delta_over_t}
\end{align}
The first term on the right-hand-side of \eqref{eq:sum_delta_over_t} can be bounded by
\begin{equation}\label{eq:sum_delta_over_t_1}
    2\sum_{t=1}^T\alpha_{n_h^t}^0 V_{h+1}^0  = 2\sum_{(s,a)\in\S\times\A} \sum_{m=1}^{n_h^T(s,a)} \alpha_{m}^0 V_{h+1}^0 \le 2\sum_{(s,a)\in\S\times\A} \sum_{m=1}^{\infty} \alpha_{n_h^t}^0 V_{h+1}^0\le 2n_0SAV_{h+1}^0,
\end{equation}
where the last inequality is from
\begin{equation*}
    \sum_{m=1}^{\infty} \alpha_{n_h^t}^0 = \frac{n_0}{H+n_0} \left(1+\frac{n_0+1}{H+n_0+1}+\cdots\right) = \frac{n_0}{H+n_0} \frac{H+n_0}{H-1} \le \frac{n_0}{H-1}.
\end{equation*}
For any $(t,h)\in[T]\times[H]$, let the mixing rate be chosen as
\begin{equation}\label{eq:eta_choice}
    \eta_{t,h} \defeq \frac{1}{\sqrt{e_{q_h^t}}+1} = \frac{1}{\sqrt{(1+1/H)^{q_h^t}H}+1},
\end{equation}
which is non-increasing in $t$ along the visits to any fixed $(h,s,a)$. 
For the second term of \eqref{eq:sum_delta_over_t}, we have
\begin{align}
\sum_{t=1}^T \sum_{i=1}^{n_h^t} \eta_{\ell^{n_h^t},h}\,\alpha_{n_h^t}^{\,i}\,\tilde{\xi}_{h+1}^{\ell^i} 
&= \sum_{t'=1}^{T} \tilde{\xi}_{h+1}^{t'} \sum_{t=t'}^{T} \eta_{\ell^{n_h^t},h} 
   \sum_{i=1}^{n_h^t}\alpha_{n_h^t}^{\,i}\,\mathbbm{1}\{\ell^i = t'\} \nonumber\\
&= \sum_{t'=1}^{T} \tilde{\xi}_{h+1}^{t'} \sum_{t=t'}^{T} \eta_{\ell^{n_h^t},h}\,
   \alpha_{n_h^t}^{\,n_h^{t'}} \nonumber\\
&\le \sum_{t'=1}^{T} \tilde{\xi}_{h+1}^{t'}\Big(\max_{t\ge t'}\eta_{\ell^{n_h^t},h}\Big)
   \sum_{m=n_h^{t'}}^{\infty}\alpha_m^{\,n_h^{t'}} \nonumber\\
&\le \Big(1+\frac{1}{H}\Big)\sum_{t'=1}^{T} \eta_{\ell^{n_h^{t'}},h}\,\tilde{\xi}_{h+1}^{\,t'}, \label{eq:sum_delta_over_t_2}
\end{align}
where the penultimate inequality uses the non-increasing property of $\{\eta_{t,h}\}_{t=0}^{T}$ along the visits to any fixed $(h,s,a)$ and the last inequality follows from Proposition \ref{prop:aggregated_weight}. Note that $\tilde{\xi}_{h+1}^{t} \le 3H$ for any $t\in[T]$. Thus, 
\begin{align*}
    \sum_{t=1}^T \sum_{i=1}^{n_h^t} \eta_{\ell^{n_h^t},h}\,\alpha_{n_h^t}^{\,i}\,\tilde{\xi}_{h+1}^{\ell^i}  \lesssim   H\sum_{t=1}^{T} \eta_{\ell^{n_h^{t}},h}.
\end{align*}
To control the third term of \eqref{eq:sum_delta_over_t}, note that by \eqref{eq:eta_choice} ensures, the sequence $\{\eta_{t,h}\}$ is constant within each stage for any fixed $(h,s,a)$.
Following \citet{zhang2020almost}, we have
\begin{align}
    &\sum_{t=1}^T\sum_{i=0}^{e_{q_h^t-1}}\frac{(1-\eta_{\ell^{n_h^t},h})\mathbbm{1}\{q_h^t\ge 1\}}{1+e_{q_h^t-1}}\tilde{\xi}_{h+1}^{\aux,\ell^{\aux,i}_{q_h^t-1}} \nonumber\\
    &= \sum_{t=1}^T \frac{(1-\eta_{\ell^{n_h^t},h}) \mathbbm{1}\{q_h^t\ge 1\} }{1+e_{q_h^t-1}}\sum_{t'=1}^T \tilde{\xi}_{h+1}^{\aux,t'} \sum_{i=0}^{e_{q_h^t-1}} \1\{\ell^{\aux,i}_{q_h^t-1} = t'\} \nonumber\\
    &= \sum_{t'=1}^T \tilde{\xi}_{h+1}^{\aux,t'}  \sum_{t=1}^T \frac{(1-\eta_{\ell^{n_h^t},h}) \mathbbm{1}\{q_h^t\ge 1\} }{1+e_{q_h^t-1}}\sum_{i=0}^{e_{q_h^t-1}} \1\{\ell^{\aux,i}_{q_h^t-1} = t'\} \nonumber\\
    & \le  (1+\frac{1}{H})^2\sum_{t'=1}^T (1-\eta_{t',h}) \tilde{\xi}_{h+1}^{\aux, t'}, \label{eq:sum_delta_over_t_3}
\end{align}
where the proof of the last inequality is deferred to Appendix \ref{appendix:proof_eq_sum_delta_over_t_3}.
For the last term in the right-hand side of \eqref{eq:sum_delta_over_t},  $(t,h)\in[T]\times[H+1]$ with $q_h^t =  0$, we have  
\begin{equation}\label{eq:sum_delta_over_t_4}
   \sum_{t: q_h^t=0}^T  (1-\eta_{\ell^{n_h^t},h})V_{h+1}^0 \le SAH V_{h+1}^0 \le 2SAH^2.
\end{equation}
since there are at most $H$ episodes during the initial stage for any fixed $(h,s,a)$.
Moreover, it is easy to verify that $(1-\eta_{t',h})\tilde{\xi}_{h+1}^{\aux,t}\le \xi_{h+1}^t + \eta_{\ell^{n_{h+1}^t},{h+1}}2H$, for any $(t,h)\in[T]\times[H]$ with $q_h^t\ge 1$ since 
\begin{align}
    (1-\eta_{t,h}) \cdot \widetilde{V}_{h+1}^{\aux,t}(s_{h+1}^t) &\le \eta_{\ell^{n_{h+1}^t},h+1} \max_{j\in[J]}\big\{\widetilde{Q}^{j,t}_{h+1}(s_{h+1}^t, \pi_{h+1}^{\aux,t}(s_{h+1}^t))\big\} \nonumber\\
    &\quad + (1-\eta_{\ell^{n_{h+1}^t},{h+1}})\widetilde{Q}_{h+1}^{\aux,t}(s_{h+1}^t, \pi_{h+1}^{\aux,t}(s_{h+1}^t)) + \eta_{\ell^{n_{h+1}^t},{h+1}}\widetilde{V}_{h+1}^{\aux,t}(s_{h+1}^t) \nonumber\\
    &\le \eta_{\ell^{n_{h+1}^t},h+1}\max_{j\in[J]}\big\{\widetilde{Q}^{j,t}_{h+1}(s_{h+1}^t, \pi_{h+1}^{t}(s_{h+1}^t))\big\} \nonumber\\
    &\quad + (1-\eta_{\ell^{n_{h+1}^t},h+1}) \cdot\widetilde{Q}_{h+1}^{\aux,t}(s_{h+1}^t,  \pi_{h+1}^{t}(s_{h+1}^t)) +\eta_{\ell^{n_{h+1}^t},{h+1}}2H \nonumber\\
   &= Q_{h+1}^{t}(s_{h+1}^t,\pi_{h+1}^{t}(s_{h+1}^t)) + \eta_{\ell^{n_{h+1}^t},{h+1}}2H \nonumber\\
   &= V_{h+1}^t(s_{h+1}^t) +  \eta_{\ell^{n_{h+1}^t},{h+1}}2H, \label{eq:xi_combine}
\end{align}
where the first inequality follows from Line \ref{line:stage_value_update} in Algorithm \ref{alg:samplingq}, and the second inequality is due to the definition of the greedy policy $\pi^t$ and the first equation is from \eqref{eq:policy_Q_with_t}. 

By substituting \eqref{eq:sum_delta_over_t_1}-\eqref{eq:sum_delta_over_t_4} to \eqref{eq:sum_delta_over_t} and using \eqref{eq:xi_combine}, we have
\begin{align*}
    \sum_{t=1}^T\delta_h^t 
    &\lesssim SA H^2 +  H\sum_{t=1}^{T}\eta_{\ell^{n_h^{t}},h} + (1+\frac{1}{H})^2\sum_{t=1}^T (1-\eta_{\ell^{n_h^t},h})\tilde{\xi}_{h+1}^{\aux,t}+\sum_{t=1}^T (b_h^{t} + \delta_{h+1}^t - \xi_{h+1}^t + \tau_{h+1}^t)\\
    &\lesssim SA H^2 +  H\sum_{t=1}^{T} \left(\eta_{\ell^{n_h^{t}},h} +\eta_{\ell^{n_{h+1}^t},{h+1}}\right) + (1+\frac{1}{H})^2 \sum_{t=1}^T\xi_{h+1}^t + \sum_{t=1}^T (b_h^{t} + \delta_{h+1}^t - \xi_{h+1}^t + \tau_{h+1}^t) \\
    &\lesssim SA H^2  + H\sum_{t=1}^{T} \left(\eta_{\ell^{n_h^{t}},h} +\eta_{\ell^{n_{h+1}^t},{h+1}}\right) +(1+\frac{1}{H})^2 \sum_{t=1}^T\delta_{h+1}^t +  \sum_{t=1}^T (b_h^{t} + \tau_{h+1}^t),
\end{align*}
where the last line is from the fact that $\xi_{h+1}^t \le \delta_{h+1}^t$, since $V^{\star} \ge V^{\pi}$ for any policy $\pi$.  
Note that $(1+\frac{1}{H})^{2H}\le e^2$. Thus, by unrolling the above inequality until $h=1$, we obtain
\begin{equation}\label{eq:delta_sum_over_t}
    \sum_{t=1}^T \delta_1^t \le \widetilde{O}\left(SAH^3 + \sum_{t=1}^T\sum_{h=1}^H(b_h^{t} + \tau_{h+1}^t + \eta_{\ell^{n_h^t},h}H)\right).
\end{equation}

Recall that $b_h^t = \eta_{\ell^{n_h^t},h}\tilde{b}_h^{t} + (1-\eta_{\ell^{n_h^t},h})\cdot \tilde{b}_h^{\aux,t}\mathbbm{1}\{q_h^t\ge 1\}$. Before proceeding, we note that $\frac{e_{q}}{\sqrt{e_{q-1}}} \le 2\sqrt{e_{q-1}}$ for any $q\ge 1$ and we denote $Q_{h,s,a} = q_h^{T+1}(s,a)$ for any $(h,s,a)\in[H]\times\S\times\A$ such that $\sum_{(s,a,h)\in\S\times\A\times[H]} \sum_{q=1}^{Q_{h,s,a}} e_{q-1} \le TH$. Also, let $Q = \max_{(h,s,a)\in[H]\times\S\times\A} Q_{h,s,a} \le \frac{\log(T/H)}{\log(1+\frac{1}{H})}\le 4H\log(T/H)$, where the last inequality is due to the fact that $\log(1+\frac{1}{H})\ge \frac{1}{4H}$ for $H\ge 1$. Thus, one has
\begin{align*}
    \sum_{t=1}^T\sum_{h=1}^H \frac{\mathbbm{1}\{q_h^t\ge 1\}}{\sqrt{e_{q-1}}}
    &\le \sum_{(s,a,h)\in\S\times\A\times[H]} \sum_{q=1}^{Q_{h,s,a}} \frac{e_{q}}{\sqrt{e_{q-1}}} \le  2\sum_{(s,a,h)\in\S\times\A\times[H]} \sum_{q=1}^{Q_{h,s,a}}\sqrt{e_{q-1}}\\
    &\le  4\sqrt{SAH^2\log(T/H)}\sqrt{\sum_{(s,a,h)\in\S\times\A\times[H]}\sum_{q=1}^{Q_{h,s,a}} e_{q-1}} \\
    &\le O(1)\sqrt{SAH^2\log(T)\cdot TH} \le  \widetilde{O}(\sqrt{SAH^3T}),
\end{align*}
where the penultimate line is from the Cauchy-Schwarz inequality. 
In addition, we have
\begin{align*}
    \sum_{t=1}^T\sum_{h=1}^H \tilde{b}_h^{t}  &\le \widetilde{O}(1)  \sum_{(s,a,h)\in\S\times\A\times[H]} \sum_{m=1}^{n^T_h(s,a)}\sqrt{\frac{H^3}{m}}\\
    &\le \widetilde{O}(1) \sqrt{SAH\cdot T/SA} \sqrt{\sum_{(s,a,h)\in\S\times\A\times[H]} \sum_{m=1}^{\frac{T}{SA}} \frac{H^3}{m}} \\
    &\le \widetilde{O}(\sqrt{SAH^5T}),
\end{align*}
where the penultimate inequality holds since the left-hand side is maximized when $n_h^{T}(s,a) = \frac{T}{SA}$ for every $(h,s,a)\in[H]\times\S\times\A$. Thus, one has
\begin{align}\label{eq:sum_b}
    \sum_{h=1}^H \sum_{t=1}^T b_h^{t} &=  \sum_{h=1}^H \sum_{t=1}^T \left(\eta_{\ell^{n_h^t},h}\tilde{b}_h^{t}+  (1-\eta_{\ell^{n_h^t},h})\sum_{h=1}^H \sum_{t=1}^T\mathbbm{1}\{q_h^t\ge 1\} \cdot \tilde{b}_h^{\aux,t}\right) \le \widetilde{O}(\sqrt{SAH^5T}).
\end{align}

Moreover, we have
\begin{align*}
    \tau_{h+1}^t &= P_{h,s_h^t,a_h^t}\left(V_{h+1}^{\star} - V_{h+1}^{\pi^t}\right) - \left(V_{h+1}^{\star} - V_{h+1}^{\pi^t}\right)(s_{h+1}^t)
\end{align*}
 is a martingale-difference sequence with respect to the filtration $\F_{t,h}$ that contains all the random variables before the step $h+1$ at the $t$-th episode. By the Hoeffding's inequality, we have
\begin{equation}\label{eq:sum_tau}
    \left\vert \sum_{h=1}^H \sum_{t=1}^T\tau_{h+1}^t \right\vert \le \widetilde{O}(\sqrt{H^3T})
\end{equation}
with probability $1-\delta$. Note that 
\begin{align*}
    \sum_{t=1}^T\sum_{h=1}^H \eta_{\ell^{n_h^t},h} = \sum_{s,a,h}\sum_{q}^{Q_{h,s,a}} \frac{e_q}{\sqrt{e_q}+1} \le\sum_{s,a,h}\sum_q^{Q_{h,s,a}} \sqrt{e_q} &\lesssim \sqrt{SAHQ} \sqrt{\sum_{s,a,h}\sum_{q=0}^{Q_{h,s,a}}e_q} \\
    &\le \widetilde{O}(\sqrt{SAH^3T})
\end{align*}
where we use the fact $Q\le 4H\log(T/H)$ again.

The last part is to show the upper bound of $\sum_{t=0}^{T-1} \frac{\eta_{t,1}}{1-\eta_{t,1}} H$. 
By the choice of $\eta_{t,h}$ defined in \eqref{eq:eta_choice}, we have 
\begin{align*}
   \sum_{t=0}^{T-1} \frac{\eta_{t,1}}{1-\eta_{t,1}} \le \sum_{a\in\A}\sum_{q=0}^Q \frac{e_{q}}{\sqrt{e_{q}}}
    & \le \sum_{a\in\A}\sum_{q=1}^Q\sqrt{e_{q}}\\
    &\le \sqrt{AQ}\sqrt{\sum_{a\in\A}\sum_{q=1}^Q e_{q}}\\
    &\le \widetilde{O}(\sqrt{AHT}).
\end{align*}

Finally, substituting \eqref{eq:sum_b} and \eqref{eq:sum_tau} into \eqref{eq:delta_sum_over_t} and rescaling $\delta$ to $\delta/4$ complete the proof.

\input{appendix/optimistic.tex}

\input{appendix/recursion.tex}

\subsection{Proof of equation \eqref{eq:sum_delta_over_t_3}}\label{appendix:proof_eq_sum_delta_over_t_3}
Fix any episode $t'\in[T]$ that lies in stage $q-1\ge 0$ for the triple $(h,s,a)$.
Then in \eqref{eq:sum_delta_over_t_3} we have
 $$\sum_{i=0}^{e_{q_h^t-1}} \mathbbm{1}\{\ell_{q_h^t-1}^{\aux,i}= t'\} = 1$$ holds if and only if the episodes $t'$ and $t$ visit the same triple $(h,s,a)$ and the visit $(t,h,s,a)$ lies in the next stage $q$ of the triple $(h,s,a)$. Let $S_{t',h,s,a} = \big\{t\in[T]: \sum_{i=0}^{e_{q_h^t-1}} \1\{\ell^{\aux,i}_{q_h^t-1} = t'\}\big\}$ represent all the visits of $(h,s,a)$ in the same stage $q$ such that its cardinality is at most $e_{q}$. For any $t\in S_{t',h,s,a}$, the mixing rate remains the same.
 
 Then, we can then decompose
\begin{align*}
     &\sum_{t=1}^T \frac{(1-\eta_{\ell^{n_h^t},h})}{1+e_{q_h^t-1}}\sum_{i=0}^{e_{q_h^t-1}} \1\{\ell^{\aux,i}_{q_h^t-1} = t'\} \\
     & \le (1+\frac{1}{H})(1-\eta_{t',h}) \cdot \sum_{t=1}^T \frac{1}{e_{q_h^t-1}} \sum_{i=0}^{e_{q_h^t-1}} \1\{\ell^{\aux,i}_{q_h^t-1} = t'\}\\
     &\le (1+\frac{1}{H})^2(1-\eta_{t',h}),
\end{align*} 
where the first inequality is from \eqref{eq:eta_choice} such that,
\begin{align*}
    (1-\eta_{\ell^{n_h^t},h}) \le 1 - \eta_{t,h} = 1 -  \frac{1}{\sqrt{e_q} +1} = \frac{\sqrt{e_q}}{\sqrt{e_q} +1} &\le (1+\frac{1}{H})\frac{\sqrt{e_{q-1}}}{\sqrt{e_{q-1}} +1} = (1+\frac{1}{H})(1-\eta_{t',h}),
\end{align*}
and the second one follows the tailored choice of stage splitting such that $e_q/e_{q-1}\le 1+\frac{1}{H}$.

%% file: appendix/optimistic.tex
\subsection{Proof of Lemma \ref{lemma:optimism}}\label{appendix:optimism_proof}

For any $(s,a,h,t)\in\S\times\A\times[H]\times[T]$, we let $m = n^t_h(s,a)$ and $q = q_h^t(s,a)$ for simplicity. From \eqref{eq:policy_Q_recursion}, one has
\begin{equation}\label{eq:Q_bar_decompose}
 Q_h^t(s,a) = \eta_{\ell^{n_h^t},h}\max_{j\in[J]}\big\{\sum_{i=0}^{m} W_{j,m}^{i} \tildeV_{h+1}^{\ell^{i}}(s_{h+1}^{\ell^{i}}) + r_h(s,a)\big\} + (1-\eta_{\ell^{n_h^t},h})\widetilde{Q}_h^{\aux,t}(s,a),    
\end{equation}
where $\{\ell^i\}_{i=0}^m$ represent the episode index of the $i$-th visit of $(h,s,a)$ before $t$, and 
\begin{equation*}
    \widetilde{Q}_h^{\aux,t}(s,a) = 
    r_h(s,a)+\max_{j\in[J]} \left\{ \sum_{i=0}^{e_{q-1}} W_{j,q-1}^{\aux,i} \tildeV_{h+1}^{\aux,\ell^{\aux,i}_{q-1}}(s_{h+1}^{\ell^{\aux,i}_{q-1}})\right\}.
\end{equation*}

Before proceeding, we first claim that for any $(s,a,t,h)\in \S\times\A\times[T]\times[H]$, we have $\tildeQ_h^{\aux,t}(s,a) \ge Q_h^{\star}(s,a)$ if the following relationship holds
\begin{equation}\label{eq:claim_optimistic}
    \max_{j\in[J]} \left\{ \sum_{i=0}^{e_{q-1}} W_{j,q-1}^{\aux,i} V_{h+1}^{\star}(s_{h+1}^{\ell^{\aux,i}_{q-1}})\right\} \ge  P_{h,s,a}V_{h+1}^{\star},
\end{equation}
where we leave the detailed proof of this claim to the end of this subsection.

Next, the following lemma shows that \eqref{eq:claim_optimistic} holds with high probability, which implies 
\begin{align*}
     Q_h^t(s,a) \ge (1-\eta_{\ell^{n_h^t},h})\tildeQ_h^{\aux,t}(s,a) \ge (1-\eta_{\ell^{n_h^t},h})Q_h^{\star}(s,a)
\end{align*}
and completes the proof of Lemma \ref{lemma:optimism}. The detailed proof of the following lemma is postponed to Appendix \ref{appendix:proof_event_opt_optimism}.
\begin{lemma}\label{lemma:event_opt_optimism}
    Consider $\delta\in (0,1)$. Assume that $J = \lceil{c\cdot\log(SAHT/\delta)}\rceil$, $\kappa^{\aux} = c\cdot(\log(SAH/\delta) + \log(T))$, and $n_0^{{\aux}} = \lceil c\cdot\log(T)\cdot \kappa^{\aux} \rceil$, where $c>0$ is some universal constant. Let $V_h^0 = 2(H-h+1)$. Then,  for any $(s,a,h,t)\in \S\times\A\times[H]\times[T]$ with stage index $q\ge 1$, the equation \eqref{eq:claim_optimistic} holds with probability at least $1-\delta$.
\end{lemma}

\paragraph{Proof of the claim:} Assuming that \eqref{eq:claim_optimistic} holds,  we will first show by induction that 
$$\tildeQ_h^{\aux,t}(s,a) \ge Q_h^{\star}(s,a)$$
holds correspondingly, for any $(s,a,h,t)\in\S\times\A\times[H]\times[T]$. To begin with, when $h'=H+1$, $\tildeQ_{H+1}^{\aux,t}(s,a) =  Q_{H+1}^{\star} = 0$ holds naturally.  When $h'=h+1 \le H$, suppose that 
$\tildeQ_{h+1}^{\aux,t}(s,a) \ge  Q_{h+1}^{\star}(s,a)$, for any $(t,s,a)\in[T]\times\S\times\A$. By this hypothesis, we also have $\tildeV_{h+1}^{\aux,t}(s_{h+1}^{t}) = \max_{a}\tildeQ_{h+1}^{\aux,t}(s_{h+1}^{t},a) \ge \tildeQ_{h+1}^{\aux,t}(s_{h+1}^{t},\pi_h^{\star}(s_{h+1}^t)) \ge Q_{h+1}^{\star}(s_{h+1}^{t},\pi_h^{\star}(s_{h+1}^t)) = V_{h+1}^{\star}(s_{h+1}^{t})$. 
By induction, when $h' = h$, we have 
\begin{align*}
    \widetilde{Q}_{h}^{\aux,t}(s,a) &=  r_h(s,a)+\max_{j\in[J]} \left\{ \sum_{i=0}^{e_{q-1}} W_{j,q-1}^{\aux,i} \tildeV_{h+1}^{\aux,\ell^{\aux,i}_{q-1}}(s_{h+1}^{\ell^{\aux,i}_{q-1}})\right\}\\
    &\ge r_h(s,a)+ \max_{j\in[J]} \left\{ \sum_{i=0}^{e_{q-1}} W_{j,q-1}^{\aux,i} V^{\star}_{h+1}(s_{h+1}^{\ell^{\aux,i}_{q-1}})\right\}
\end{align*}
Thus, if \eqref{eq:claim_optimistic} holds, we have $\widetilde{Q}_{h}^{\aux,t}(s,a) \ge r_h(s,a)  + P_{h,s,a}V_{h+1}^{\star} = Q_h^{\star}(s,a)$ for any $(s,a,t,h)\in \S\times\A\times[T]\times[H]$, by the Bellman optimality equation \eqref{eq:bellman_optimality}.

%% file: appendix/recursion.tex
\subsection{Proof of Lemma \ref{lemma:recursion_circ}}\label{appendix:proof_lemma_recursion_circ}
For any $(t,h,s,a)\in[T]\times[H]\times\S\times\A$, denote $m= n_h^t(s,a)$ as the number of visits on $(h,s,a)$ before the $t$-th episode. Also, let $\alpha_{m}^{0}\defeq \prod_{k=1}^m \frac{n_0+k-1}{H+n_0+k}$ and $\alpha_m^{i} \defeq \frac{H+1}{H+n_0+i} \prod_{k=i+1}^m \frac{n_0+k-1}{H+n_0+k}$. 

From \eqref{eq:Q_circ_recursion_update} , for any $(t,h,s,a)\in[T]\times[H]\times\S\times\A$, we have
\begin{align*}
    \widetilde{Q}_{h}^{j,t}(s, a) 
    =  r_h(s,a) + \sum_{i=0}^{m}W_{j,m}^{i } \tildeV_{h+1}^{\ell^i}(s_{h+1}^{\ell^{i}}) \le r_h(s,a) + 2H \cdot \sum_{i=0}^{m}W_{j,m}^{i } \frac{\tildeV_{h+1}^{\ell^i}(s_{h+1}^{\ell^{i}})}{2H}.
\end{align*}

Note that $\frac{\tildeV_{h+1}^{\ell^i}(s_{h+1}^{\ell^{i}})}{2H} \le 1$ for any $i=0,\ldots,m$. Thus, we can apply Lemma \ref{lemma:concentration_aggregated_GD} and Proposition \ref{prop:aggregated_weight}
\begin{align*}
    \max_{j\in[J]} \widetilde{Q}_{h}^{j,t}(s, a) 
    &\le r_h(s,a) + 2H \left(\sum_{i=0}^{m}\E[W_{j,m}^{i}] \frac{\tildeV_{h+1}^{\ell^i}(s_{h+1}^{\ell^{i}})}{2H} + \frac{c_1}{2}\sqrt{\frac{(H+1)\kappa^2\log^3(2SAHTJ/\delta)}{H + n_0 + m}} \right.\\
    &\quad\left.+ \frac{c_2}{2} \frac{(H+1)\kappa\log^2(2SAHTJ/\delta)}{H+n_0+m}\right)\\
    &\le r_h(s,a) + \alpha_{m}^{0 } \tildeV_{h+1}^{0 } + \sum_{i=1}^{m}\alpha_m^i \tildeV_{h+1}^{\ell^{i} }(s_{h+1}^{\ell^{i}})\\
    &\quad + c_1\sqrt{\frac{(H+1)^3\kappa^2\log^3(2SAHTJ/\delta)}{H + n_0 + m}}  + c_2 \frac{(H+1)^2\kappa\log^2(2SAHTJ/\delta)}{H+n_0+m},
\end{align*}
with probability at least $1-\delta/2$, where $c_1, c_2$ are universal constants. By the Bellman optimality equation \eqref{eq:bellman_optimality}, we have
\begin{align}
    \zeta_h^{t} &\le \alpha_{m}^0 \tildeV_{h+1}^{0} + \sum_{i=1}^m \alpha_m^i \left(\left(\tildeV_{h+1}^{\ell^{i}}- V_{h+1}^{\star}\right)(s_{h+1}^{\ell^{i}}) + H\right) + \sum_{i=1}^{m}\alpha_m^i\left(V_{h+1}^{\star}(s_{h+1}^{\ell^{i}})-P_{h,s,a}V_{h+1}^{\star}\right) \nonumber\\
    &\quad+ c_1 \sqrt{\frac{(H+1)^3\kappa^2 \log^3(2SAHTJ/\delta)}{H+n_0+m}} + c_2\frac{(H+1)^2\kappa \log^2(2SAHTJ/\delta)}{H+n_0+m} + H - \sum_{i=1}^{m} \alpha_m^i  H \nonumber\\
    &\le 2\alpha_{m}^0 \tildeV_{h+1}^{0} + \sum_{i=1}^m \alpha_m^i \left(\left(\tildeV_{h+1}^{\ell^{i}}- V_{h+1}^{\star}\right)(s_{h+1}^{\ell^{i}}) + H\right) + \sum_{i=1}^{m}\alpha_m^i\left(V_{h+1}^{\star}(s_{h+1}^{\ell^{i}})-P_{h,s,a}V_{h+1}^{\star}\right)\nonumber \\
    &\quad+ c_1 \sqrt{\frac{(H+1)^3\kappa^2 \log^3(2SAHTJ/\delta)}{H+n_0+m}} + c_2\frac{(H+1)^2\kappa \log^2(2SAHTJ/\delta)}{H+n_0+m},\label{eq:zeta_h^t_intermediate}
 \end{align}
 where the last line uses the fact $H\le \tildeV_{h+1}^{0} $ and the equation \eqref{eqn:agg-sum-one} such that 
 \begin{equation*}
    H - \sum_{i=1}^{m} \alpha_m^i  H = \alpha_m^0 H.
 \end{equation*}

In addition, we denote $\F_i$ as the filtration containing all the random variables before the episode $\ell^i_h(s,a)$, such that $\alpha_m^i \left(V_{h+1}^{\star}(s_{h+1}^{\ell^{i}})-P_{h,s,a}V_{h+1}^{\star}\right)$ is a martingale difference sequence w.r.t. $\{\F_i\}_{i\ge 0}$ for any $i\le m$. 
Following \cite{jin2018q} and by Hoeffding's inequality and Proposition \ref{prop:aggregated_weight} (i.e., \eqref{eq:prop_max_exp_and_var}), with probability at least $1-\delta/2$, we have
\begin{align*}
    \left\vert \sum_{i=1}^{m}\alpha_m^i\left(V_{h+1}^{\star}(s_{h+1}^{\ell^{i}})-P_{h,s,a}V_{h+1}^{\star}\right) \right\vert  
    &\le c_3H\sqrt{\sum_{i=1}^{m}(\alpha_m^i)^2\log(2SAHT/\delta)} \\
    &\le c_3\sqrt{\frac{(H+1)^3\kappa}{H+n_0+m}\log(2SAHT/\delta)}
\end{align*}
for any $(t,h,s,a)\in[T]\times[H]\times\S\times\A$ and some universal constant $c_3>0$. Substituting the above inequality into \eqref{eq:zeta_h^t_intermediate} gives
\begin{align*}
    \zeta_h^{t} &\le 2\alpha_{m}^0 V_{h+1}^0 + \sum_{i=1}^m \alpha_m^i \left(\left(\tildeV_{h+1}^{\ell^{i}}- V_{h+1}^{\star}\right)(s_{h+1}^{\ell^{i}}) + H\right) \\
    &\quad +  (c_1 +c_3)\sqrt{\frac{(H+1)^3\kappa^2 \log^3(2SAHTJ/\delta)}{H+n_0+m}}  + c_2\frac{(H+1)^2\kappa \log^2(2SAHTJ/\delta)}{H+n_0+m}.
 \end{align*}
By letting $\tilde{b}_h^{t} = \widetilde{O}\left(\sqrt{\frac{(H+1)^3}{H+n_0+m}} + \frac{(H+1)^2}{H+n_0+m}\right)$, we complete the proof.

\subsection{Proof of Lemma \ref{lemma:recursion_dagger}}\label{appendix:proof_lemma_recursion_dagger}
Note that during the initial stage, for any $(t,h,s,a)\in[T]\times[H]\times\S\times\A$ within the stage $q_h^t(s,a)=0$, we have 
\begin{equation*}
    \tildeQ_h^{\aux,t}(s,a) - Q^{\star}_h(s,a) \le V^0_{h+1}.
\end{equation*}

From \eqref{eq:Q_dagger_recursion_update} and Lemma \ref{lemma:optimism}, for any $(t,h,s,a)\in[T]\times[H]\times\S\times\A$ with $q_h^t(s,a)\ge 1$, we have
\begin{align}\label{eq:bound_delta_aux}
    0\le \tildeQ_h^{\aux,t}(s,a) - Q^{\star}_h(s,a) &\le \max_{j\in[J]}\big\{ \sum_{i=0}^{e_{q-1}} W_{j,q-1}^{\aux,i} \tildeV_{h+1}^{\aux,\ell^{\aux,i}_{q-1}}(s_{h+1}^{\ell^{\aux,i}_{q-1}})\big\}  - P_{h,s,a}V^{\star}_{h+1}.
\end{align}

Then,  applying Lemma \ref{lemma:concentration_standard_Dir}  leads to
\begin{align*}
    &\sum_{i=0}^{e_{q-1}} W_{j,q-1}^{\aux,i} \tildeV_{h+1}^{\aux,\ell^{\aux,i}_{q-1}}(s_{h+1}^{\ell^{\aux,i}_{q-1}}) \\
    &= 2H \left(\sum_{i=0}^{e_{q-1}} W_{j,q-1}^{\aux,i} \frac{\tildeV_{h+1}^{\aux,\ell^{\aux,i}_{q-1}}(s_{h+1}^{\ell^{\aux,i}_{q-1}})}{2H}\right) \\
    &\le 2H \left(\frac{1}{1+e_{q-1}} \sum_{i=0}^{e_{q-1}} \frac{\tildeV_{h+1}^{\aux,\ell^{\aux,i}_{q-1}}(s_{h+1}^{\ell^{\aux,i}_{q-1}})}{2H} + c_1^{\aux} \sqrt{\frac{(\kappa^{\aux})^2\log^3(SAHTJ/\delta)}{e_{q-1}}} + c_2^{\aux} \frac{\kappa^{\aux}\log^2(SAHTJ/\delta)}{e_{q-1}} \right)\\
    &\le \frac{1}{1+e_{q-1}} \sum_{i=0}^{e_{q-1}} \tildeV_{h+1}^{\aux,\ell^{\aux,i}_{q-1}}(s_{h+1}^{\ell^{\aux,i}_{q-1}}) + 2H\left(c_1^{\aux} \sqrt{\frac{(\kappa^{\aux})^2\log^3(SAHTJ/\delta)}{e_{q-1}}} + c_2^{\aux} \frac{\kappa^{\aux}\log^2(SAHTJ/\delta)}{e_{q-1}}\right),
\end{align*}
for every $j\in [J]$ with probability at least $1-\delta/2$. 
Following the similar procedure in Appendix \ref{appendix:proof_lemma_recursion_circ} and by Hoeffding's inequality, with probability at least $1-\delta/2$, we have
\begin{align*}
    \left\vert \sum_{i=0}^{e_{q-1}}\frac{1}{1+e_{q-1}}\left(V_{h+1}^{\star}(s_{h+1}^{\ell^{\aux,i}_{q-1}})-P_{h,s,a}V_{h+1}^{\star}\right) \right\vert  
    &\le c_4\sqrt{\frac{H^2}{e_{q-1}}\log(2SAHT/\delta)}. \\
\end{align*}

Thus, we obtain that the following holds with probability $1-\delta$,
\begin{align*}
    &\tildeQ_h^{\aux,t}(s,a) - Q^{\star}_h(s,a) \\
    &\le \frac{1}{1+e_{q-1}}\sum_{i=0}^{e_{q-1}} \left(\tildeV_{h+1}^{\aux,\ell^{\aux,i}_{q-1}}- V_{h+1}^{\star}\right)(s_{h+1}^{\ell^{\aux,i}_{q-1}}) + \frac{1}{1+e_{q-1}}\sum_{i=0}^{e_{q-1}}\left(V_{h+1}^{\star}(s_{h+1}^{\ell^{\aux,i}_{q-1}})-P_{h,s,a}V_{h+1}^{\star}\right) \\
    &\quad +2H\left(c_1^{\aux} \sqrt{\frac{(\kappa^{\aux})^2\log^3(SAHTJ/\delta)}{e_{q-1}}} + c_2^{\aux} \frac{\kappa^{\aux}\log^2(SAHTJ/\delta)}{e_{q-1}}\right)\\
    &\le \frac{1}{1+e_{q-1}} \sum_{i=1}^{e_{q-1}} \left(\tildeV_{h+1}^{\aux,\ell^{\aux,i}_{q-1}} - V_{h+1}^{\star}\right)(s_{h+1}^{\ell^{\aux,i}_{q-1}}) + \tilde{b}^{\aux,t}_h,
\end{align*}
where $\tilde{b}^{\aux,t}_h =\widetilde{O}\left(\sqrt{\frac{H^2}{e_{q-1}}}\right)$.

Combining two cases of $q_h^t(s,a)\ge 1$ and $q_h^t(s,a)=0$ leads to
\begin{align*}
    \zeta_h^{\aux,t} &\le \mathbbm{1}\{q_h^t(s,a)\ge 1\}\left(\sum_{i=0}^{e_{q-1}} \frac{1}{1+e_{q-1}}\left(\tildeV_{h+1}^{\aux,\ell^{\aux,i}_{q-1}} - V_{h+1}^{\star}\right)(s_{h+1}^{\ell^{\aux,i}_{q-1}}) + \tilde{b}^{\aux,t}_h\right) + \mathbbm{1}\{q_h^t(s,a)=0\}\cdot V^0_{h+1},
\end{align*}
which completes the proof.
% ---------------------------------

\subsection{Proof of Lemma \ref{lemma:event_opt_optimism}}\label{appendix:proof_event_opt_optimism}

Following \cite{tiapkin2024model},  let $ \mathcal{E}^{\star}(\delta)$ be the event containing all $(t,h,s,a)\in[T]\times[H]\times\S\times\A$, such that  
\begin{equation}\label{eq:def_E_star}
    \mathcal{K}_{\inf}\left(\frac{1}{e_{q}} \sum_{i=1}^{e_{q}} \delta_{V_{h+1}^{\star}(s_{h+1}^{\ell^{\aux,i}_{q}})}, P_{h,s,a}V_{h+1}^{\star}\right) \le \frac{\beta^{\star}(\delta,e_{q})}{e_{q}},
\end{equation}
where $q = q_h^t(s,a)$ and $\beta^{\star}(\delta,n) \defeq \log(2SAH/\delta) + 3\log(\mathrm e\pi(2n+1))$. The following lemma shows that $\mathcal{E}^{\star}(\delta)$ holds with probability $1 - \frac{\delta}{2}$
\begin{lemma}[Lemma 4 in \cite{tiapkin2024model}]\label{lemma:event_model_error}
    Consider $\delta\in(0,1)$. With probability $1-\frac{\delta}{2}$, the following event holds 
    \begin{align}\label{eq:event_Delta_star}
        & \mathcal{K}_{\inf}\left(\frac{1}{e_{q}} \sum_{i=1}^{e_{q}} \delta_{V_{h+1}^{\star}(s_{h+1}^{\ell^i_q})}, P_{h,s,a}V_{h+1}^{\star}\right) \le \frac{\beta^{\star}(\delta,e_{q})}{e_{q}}, \quad\forall (t,h,s,a)\in[T]\times[H]\times\S\times\A,
    \end{align} 
    where $q=q_h^t(s,a)$ and $\beta^{\star}(\delta,e_q) \defeq \log(2SAH/\delta) + 3\log(e\pi(2e_{q}+1))$.
\end{lemma}

Consider some fixed $(t,h,s,a)\in[T]\times [H]\times\S\times\A$ within the stage $q = q_h^t(s,a)$. To construct an anti-concentration inequality bound of the weighted sum $ W_{j,q}^{\aux,i} V^{\star}_{h+1}(s_{h+1}^{\ell^{\aux,i}_{q}})$, we first recall the following two lemmas provided in \citet{tiapkin2024model}.
\begin{lemma}[Lemma 3 in \citet{tiapkin2024model}] \label{lemma:randql_standard_dirichlet}
    For any stage $q\ge 0$ and $j\in[J]$, the aggregated weights $W_{j,q}^{\aux}$ follows a standard Dirichlet distribution $\operatorname{Dir}(n_0^{\aux}/\kappa^{\aux},1/\kappa^{\aux},\ldots,1/\kappa^{\aux})$.
\end{lemma}
\begin{lemma}[Lemma 10 in \citet{tiapkin2024model}]\label{lemma:anti-concentration_dirichlet}
    For any $\alpha = (\alpha_0+1,\alpha_1,\ldots,\alpha_m)\in\mathbb{R}_{++}^{m+1}$, define $\overline{p} \in \Delta_m$ such that $\overline{p}(\ell) = \alpha_\ell / \overline{\alpha}, \, \ell = 0, \ldots, m$, where $\overline{\alpha} = \sum_{j=0}^m \alpha_j $. Also define a measure $\overline{\nu} = \sum_{i=0}^m \overline{p}(i) \cdot \delta_{f(i)}$.
    Let $\varepsilon \in (0, 1)$. Assume that $\alpha_0 \geq c_0 + \log_{17/16}(2(\overline{\alpha} - \alpha_0))$ for some universal constant $c_0$. Then for any $f : \{0, \ldots, m\} \to [0, b_0]$ such that $f(0) = b_0, \, f(j) \leq b \leq b_0/2, \, j \in [m],$ and any $\mu \in (0, b)$
    \[
    \mathbb{P}_{w \sim \operatorname{Dir}(\alpha)} [wf \geq \mu] \geq (1 - \varepsilon) \mathbb{P}_{g \sim \mathcal{N}(0, 1)} \left[ g \geq \sqrt{2 \overline{\alpha} \mathcal{K}_{\text{inf}}(\overline{\nu}, \mu)} \right].
    \]
\end{lemma}

According to Lemma \ref{lemma:randql_standard_dirichlet} and applying Lemma \ref{lemma:anti-concentration_dirichlet} with $\alpha_0 = n_0^{\aux}/\kappa^{\aux} -1 $, $\alpha_i = 1/\kappa^{\aux}, ~\forall i\in[e_q]$, $r_0 = 2$, $b_0 = 2(H-h+1)$, and $\overline{\nu}_q = \frac{n_0^{\aux}-\kappa^{\aux}}{e_q+n_0^{\aux}-\kappa^{\aux}} \delta_{V_{h+1}^{\star}(s_0)} + \sum_{i=1}^{e_q} \frac{1}{e_q+n_0^{\aux}-\kappa^{\aux}}\delta_{V_{h+1}^{\star}(s_{h+1}^{\ell^{\aux,i}_{q}})}$, we have that conditioned on the event $\mathcal{E}^{\star}(\delta)$ holds, 
\begin{align*}
    &\mathbb{P}\left( \sum_{i=0}^{e_q} W_{j,q}^{\aux,i} V_{h+1}^{\star}(s_{h+1}^{\ell^{\aux,i}_q}) \geq P_{h,s,a}V_{h+1}^{\star} \,\middle|\, \mathcal{E}^{\star}(\delta)\right)\\
    &\ge \frac{1}{2} \left( 1 - \Phi\left( \sqrt{\frac{2(e_q+n_0^{\aux}-\kappa^{\aux})\mathcal{K}_{\inf}(\overline{\nu}_q,P_{h,s,a}V_{h+1}^{\star})}{\kappa^{\aux}}} \right) \right) \ge \frac{1}{2} \left( 1 - \Phi\left( \sqrt{\frac{2\beta^{\star}(\delta, T)}{\kappa^{\aux}}} \right) \right).
\end{align*}
where $\Phi$ denotes the CDF of the standard normal distribution. Here the last inequality is from Lemma \ref{lemma:randql_lemma12} and Lemma \ref{lemma:event_model_error}.

Then, by selecting $\kappa^{\aux} = 2\beta^{\star}(\delta, T)$, we ensure a constant probability of optimism:
\[
\mathbb{P}\left( \sum_{i=0}^{e_q} W_{j,q}^{\aux,i} V_{h+1}^{\star}(s_{h+1}^{\ell^{\aux,i}_q}) \geq P_{h,s,a}V_{h+1}^{\star}\,\middle|\  \mathcal{E}^{\star}(\delta)\right) 
\geq \frac{1 - \Phi(1)}{2} \triangleq \gamma.
\]

Now, choosing $J = \left\lceil \log\left( \frac{2SAHT}{\delta} \right) / \log(1 / (1 - \gamma)) \right\rceil = \lceil c_J \cdot \log(2SAHT / \delta)\rceil$ ensures:
\begin{align*}
    &\mathbb{P}\left( \max_{j \in [J]} \left\{ \sum_{i=0}^{e_q} W_{j,q}^{\aux,i} V_{h+1}^{\star}(s_{h+1}^{\ell^{\aux,i}_q}) \right\} \geq P_{h,s,a}V_{h+1}^{\star} \,\middle|\, \mathcal{E}^{\star}(\delta) \right)\geq 1 - (1 - \gamma)^J \geq 1 - \frac{\delta}{2SAHT}.
\end{align*}

By applying a union bound over $(t,h,s,a)\in[T]\times [H]\times\S\times\A$ and taking expectation on $\mathcal{E}^{\star}(\delta)$, we conclude the proof.

%% file: appendix/anytime.tex
\section{Extending \myalg to Anytime Convergence}\label{appendix:anytime}

\subsection{\myalgany with anytime convergence guarantees}
In this section, we introduce \myalgany, a variant of Algorithm~\ref{alg:samplingq}, described in Algorithm~\ref{alg:samplingq_anytime}. Compared to Algorithm~\ref{alg:samplingq}, where the parameters $J,n_0^{\aux},\kappa^{\aux}$ are pre-fixed across episodes as functions of the total number of episodes, \myalgany avoids requiring the total number of episodes in advance and adaptively tunes these parameters using the stage-wise information (cf. Line \ref{line:stage_adapt_anytime} in Algorithm \ref{alg:samplingq_anytime}). 

\begin{algorithm}[!t]
    \caption{\myalgany}\label{alg:samplingq_anytime}
    \begin{algorithmic}[1]
    \REQUIRE Initial state $s_1$, optimistically-initial value $\{V^0_h\}$, inflation coefficient $\kappa\ge 0$, the number of prior transitions $n_0\ge 0$, ensemble size $J\ge 0$, mixing rate $\{\eta_{t,h}\}$, tunable constant $c'>0$ and confidence level $\delta\in(0,1)$.
    \STATE \textbf{Initialize:} $n_h(s, a), n_h^{\aux}(s,a), q_h(s,a) \gets 0$; $\tildeV_h(s), \tildeV^{\aux}_h(s) \gets V^0_h$; $\widetilde{Q}^{j}_{h}(s, a),\widetilde{Q}^{\aux}_{h}(s, a),\widetilde{Q}^{\aux,j}_{h}(s, a)  \gets r_h(s, a) + V^0_{h+1}$ for $(j,h,s,a)\in [J]\times[H]\times\S\times\A$.
    \FOR{$t \in [T]$}{
        \FOR{$h = 1,\ldots, H$}{     
            \STATE Play $a_{h} = \arg\max_{a\in\A} Q_h(s_h,a)$ and observe the next state $s_{h+1} \sim P_h(\cdot \vert s_{h}, a_{h})$.
            \STATE Set $m\gets n_h(s_{h}, a_{h})$, $m^{\aux}\!\gets\! n_h^{\aux}(s_h,a_h)$, $q\gets q_h(s_h,a_h)$ and $e_q \gets (1+1/H)^q H$.
             \STATE \textcolor{blue}{\texttt{/* Adapt parameters via the stage index $q$ of the current visit and its length $e_q$. \hfill */}}
            \STATE Set  $J^{\aux}_q \gets \lceil{c'\cdot\log(SAH(q+1)^2/\delta)}\rceil$, $\kappa_q^{\aux} \gets c'\cdot\left(\log(SAH/\delta) + \log(e_q)\right)$ and $n_{0,q}^{\aux} \gets \lceil c'\cdot\log(e_q)\cdot \kappa_q^{\aux} \rceil$. \label{line:stage_adapt_anytime}
            \STATE \textcolor{blue}{\texttt{/* Update temporary Q-ensembles via randomized learning rates. \hfill*/}}
            \FOR{$j = 1,\ldots, J$}{
            \STATE Sample $w^{j}_{m} \sim \B\left(\frac{H+1}{\kappa}, \frac{m + n_0}{\kappa} \right)$ and update $\tildeQ_{h}^{j}$ via \eqref{eq:Q_circ_update} .
            }
            \ENDFOR
            \FOR{$j = 1,\ldots, J^{\aux}_q$}{
            \STATE Sample $w^{\aux,j}_{m} \sim \B\left(\frac{1}{\kappa_q^{\aux}}, \frac{m^{\aux}+n_{0,q}^{\aux}}{\kappa_q^{\aux}}\right)$ and update $\tildeQ_{h}^{\aux,j}$ via \eqref{eq:Q_dagger_update}.
            }
            \ENDFOR
            \STATE \textcolor{blue}{\texttt{/* Update the agile policy Q-function by optimistic mixing.\hfill */}}
            \STATE Update the policy Q-function $Q_h$ via \eqref{eq:policy_Q_update}.
            \STATE \textcolor{blue}{\texttt{/* Update the policy with step-wise agility. \hfill */}}
            \STATE Update policy $\pi_h(s_h) \gets \arg\max_{a\in\A} Q_h(s_h,a)$. \label{line:policy_update}
            \STATE \textcolor{blue}{\texttt{/* Update $\tildeV_{h}$ optimistically. \hfill */}}
            \STATE Update $\tildeV_{h}(s_{h})\gets \max_{j\in[J]}\tildeQ^{j}_{h}(s_{h}, \pi_h(s_h))$.
            \STATE \textcolor{blue}{\texttt{/* Update visit counters. \hfill */}}            
            \STATE Update counter $n_h\!(s_{h}, a_{h})\!\!\gets\! n_h(s_{h}, a_{h})\!+\!1$ and $n_h^{\aux}(s_h,a_h) \gets n_h^{\aux}(s_h,a_h) + 1$.
            \STATE \textcolor{blue}{\texttt{/* At the end of the stage: update $\widetilde{Q}_h^{\aux}$, $\widetilde{V}_h^{\aux}$ and reset $n_h^{\aux}$, $\{\widetilde{Q}_h^{b,j}\}$.\hfill */}}
            \IF{$n_h^{\aux}(s_h,a_h) = \lfloor e_q\rfloor$ for the stage $q=q_h(s,a)$}{
                \STATE Update $\tildeQ_{h}^{\aux}(s_{h}, a_{h}) \!\gets\! \max_{j \in [J^{\aux}_q]} \widetilde{Q}^{\aux,j}_{h}(s_{h}, a_{h})$, $\pi^{\aux}_{h}(s_{h})\!\gets\arg\max_{a\in\A}\widetilde{Q}^{\aux}_{h}(s_{h}, a)$, and $\tilde{V}_h^{\aux}(s_h) \gets \widetilde{Q}^{\aux}_{h}(s_{h}, \pi^{\aux}_{h}(s_{h}))$. 
                \STATE Update $q_h\!(s_h,a_h)\!\gets\!\! q_h\!(s_h,a_h) + 1$ and $J^{\aux}_{q+1} \gets \lceil{c'\cdot\log(SAH(q+2)^2/\delta)}\rceil$.
                \STATE Reset $\widetilde{Q}_h^{\aux,j}(s_h,a_h) \gets r_h(s_h,a_h) +V^0_{h+1}$ for $j\in[J^{\aux}_{q+1}]$ and $n_h^{\aux}(s_h,a_h) \gets 0 $ . 
            }
            \ENDIF
        }
        \ENDFOR
    }
    \ENDFOR
    \end{algorithmic}
    \end{algorithm}

%%%%%%%%%%%%%%%%%%

We then establish an anytime-convergence guarantee in Theorem~\ref{thm:regret_anytime}, where the proof is deferred to the next subsection.

\begin{theorem}\label{thm:regret_anytime}
Consider $\delta\in (0,1)$ and $c'>0$ is a sufficiently large universal constant. Let the initialized value function $V_h^0 = 2(H-h+1)$ for any $h\in[H+1]$ and the mixing rate $\eta_{t,h} = \frac{1}{\sqrt{(1+1/H)^{q}H}+1}$ where $q = q_h^t(s_h^t,a_h^t)$ is the stage index for any $(t,h)\in[T]\times[H]$. Then, with probability at least $1-\delta$, Algorithm \ref{alg:samplingq_anytime} guarantees that
\begin{align*}
    \mathrm{Regret}_T \le \widetilde{O}\left(\sqrt{H^5SAT}\right).
\end{align*}
\end{theorem}

\subsection{Proof of Theorem \ref{thm:regret_anytime}}\label{appendix:proof_anytime}
The proof follows Appendix \ref{appendix:regret_proof}, except for the optimism step (i.e. Lemma \ref{lemma:optimism}), which depends on the appropriate choices of the prefixed parameters $J,n_0^{\aux},\kappa^{\aux}$ and eventually leads to non-anytime results. We therefore establish the optimism under the adaptive parameters $\{J_q^{\aux}\}$. $\{\kappa_q^{\aux}\}$, $\{n_{0,q}^{\aux}\}$, as follows.
\begin{lemma}[Optimism]\label{lemma:optimism_anytime}
    Consider $\delta\in (0,1)$. For each $(t,h,s,a)\in[T]\times[H]\times\S\times\A$, let $q\coloneqq q_h^t(s,a)$ denote the corresponding stage with length $e_q = (1+1/H)^{q}H$. Define $J_q^{\aux} = \lceil{c'\cdot\log(SAH(q+1)^2/\delta)}\rceil$, $\kappa_q^{\aux} = c'\cdot\left(\log(SAH/\delta) + \log(e_q)\right)$ and $n_{0,q}^{\aux} = \lceil c'\cdot\log(e_q)\cdot \kappa_q^{\aux} \rceil$, where $c'>0$ is a universal constant. Let the initialized value function $V_h^0 = 2(H-h+1)$ for any $h\in[H+1]$ and the mixing rate $\eta_{t,h} = \frac{1}{\sqrt{(1+1/H)^{q}H}+1}$ where $q = q_h^t(s_h^t,a_h^t)$ is the stage index for any $(t,h)\in[T]\times[H]$. Then,  with probability at least $1-\delta$, the following event holds
    \begin{equation*}
        (1-\eta_{t,h})\cdot Q^{\star}_{h}(s,a) \le  Q_{h}^t (s,a), \quad\forall (t,h,s,a)\in[T]\times[H]\times\S\times\A.
    \end{equation*}
\end{lemma}

\begin{proof}
From Appendix~\ref{appendix:optimism_proof} (i.e., the proof of Lemma \ref{lemma:optimism}), if the inequality \eqref{eq:claim_optimistic} holds, then $\tildeQ_h^{\aux,t}(s,a)\ge Q_h^{\star}(s,a)$, for any $(t,h,s,a)\in[T]\times[H]\times\S\times\A$. The following shows that \eqref{eq:claim_optimistic} holds with high probability under the adaptive parameters $\{J_q^{\aux}\},\{\kappa_q^{\aux}\},\{n_{0,q}^{\aux}\}$, which is analogous to Lemma \ref{lemma:event_opt_optimism}. 

To establish that \eqref{eq:claim_optimistic}  holds with high probability, we first adapt Lemma~\ref{lemma:randql_standard_dirichlet} to stage-wise parameters $\{\kappa_q^{\aux}\},\{n_{0,q}^{\aux}\}$, as these parameters remain fixed during the entire stage $q$.
\begin{lemma}[Modified version of Lemma 3 in \citet{tiapkin2024model}]
    For any $q\ge 0$ and $j\in[J_q^{\aux}]$, the aggregated weights $W_{j,q}^{\aux}= \left(W_{j,e_{q}}^{\aux,0}, W_{j,e_{q}}^{\aux,1},\ldots,W_{j,e_{q}}^{\aux,e_{q}}\right)$ as in \eqref{eq:W_aux_def}, where each randomized weight $w_{k,q}^{\aux,j}\sim \B(\frac{1}{\kappa_q^{\aux}},\frac{k+n_{0,q}^{\aux}}{\kappa_q^{\aux}}), \forall k\in\{0,\ldots,e_{q-1}\}$ using some fixed nonnegative parameters $\kappa_q^{\aux}, n_{0,q}^{\aux}$ at the stage $q$. Then, the aggregated weights $$W_{j,q}^{\aux} \sim \operatorname{Dir}(n_{0,q}^{\aux}/\kappa_q^{\aux},1/\kappa_q^{\aux},\ldots,1/\kappa_q^{\aux}).$$
\end{lemma}

Next, applying Lemma \ref{lemma:anti-concentration_dirichlet} with $\alpha_0 = n_{0,q}^{\aux}/\kappa_q^{\aux} -1 $, $\alpha_i = 1/\kappa_q^{\aux}, ~\forall i\in[e_q]$, $r_0 = 2$, $b_0 = 2(H-h+1)$, and $\overline{\nu}_q = \frac{n_{0,q}^{\aux}-\kappa_q^{\aux}}{e_q+n_{0,q}^{\aux}-\kappa_q^{\aux}} \delta_{V_{h+1}^{\star}(s_0)} + \sum_{i=1}^{e_q} \frac{1}{e_q+n_{0,q}^{\aux}-\kappa_q^{\aux}}\delta_{V_{h+1}^{\star}(s_{h+1}^{\ell^{\aux,i}_{q}})}$ leads to 
\begin{equation*}
    \mathbb{P}\left( \sum_{i=0}^{e_q} W_{j,q}^{\aux,i} V_{h+1}^{\star}(s_{h+1}^{\ell^{\aux,i}_q}) \geq  P_{h,s,a}V_{h+1}^{\star}\,\middle|\, \mathcal{E}^{\star}(\delta)\right)\\
    \ge \frac{1}{2} \left( 1 - \Phi\left( \sqrt{\frac{2(e_q+n_{0,q}^{\aux}-\kappa_q^{\aux})\mathcal{K}_{\inf}(\overline{\nu}_q,P_{h,s,a}V_{h+1}^{\star})}{\kappa_q^{\aux}}} \right) \right).
\end{equation*}
Conditioned on the event $\mathcal{E}^{\star}(\delta)$, Lemma \ref{lemma:randql_lemma12} and Lemma \ref{lemma:event_model_error} give
\begin{align*}
    &\mathbb{P}\left( \sum_{i=0}^{e_q} W_{j,q}^{\aux,i} V_{h+1}^{\star}(s_{h+1}^{\ell^{\aux,i}_q}) \geq  P_{h,s,a}V_{h+1}^{\star}\,\middle|\, \mathcal{E}^{\star}(\delta)\right)\ge \frac{1}{2} \left( 1 - \Phi\left( \sqrt{\frac{2\beta^{\star}(\delta, e_q)}{\kappa_q^{\aux}}} \right) \right).
\end{align*}
Choosing $\kappa_q^{\aux} = 2\beta^{\star}(\delta, e_q)$ yields a constant probability of optimism:
\[
\mathbb{P}\left( \sum_{i=0}^{e_q} W_{j,q}^{\aux,i} V_{h+1}^{\star}(s_{h+1}^{\ell^{\aux,i}_q}) \geq  P_{h,s,a}V_{h+1}^{\star}\,\middle|\  \mathcal{E}^{\star}(\delta)\right) 
\geq \frac{1 - \Phi(1)}{2} \triangleq \gamma.
\]

For any fixed $(h,s,a)\in [H]\times\S\times\A$, define 
\begin{align*}
 \mathcal{E}_{q,h,s,a} = \left\{\forall t\in[T]:  \max_{j \in [J_q^{\aux}]} \left\{ \sum_{i=0}^{e_q} W_{j,q}^{\aux,i} V_{h+1}^{\star}(s_{h+1}^{\ell^{\aux,i}_q}) \right\} \geq P_{h,s,a}V_{h+1}^{\star}, ~q_h^t(s,a)=q\right\},\quad\forall q\in\left\{0,\ldots, q_h^{T+1}(s,a)\right\}.
\end{align*}
With $J_q^{\aux} = \left\lceil \log\left( \frac{4SAH (q+1)^2}{\delta} \right) / \log(1 / (1 - \gamma)) \right\rceil = \lceil c_J \cdot \log(4SAH(q+1)^2 / \delta)\rceil$ and taking union bound, this gives the lower bound of the event $\mathcal{E}_{q,h,s,a}$ happened in stage $q$:
\begin{align*}
    \mathbb{P}\left( \mathcal{E}_{q,h,s,a}\,\middle|\, \mathcal{E}^{\star}(\delta) \right) &= \mathbb{P}\left( \max_{j \in [J_q^{\aux}]} \left\{ \sum_{i=0}^{e_q} W_{j,q}^{\aux,i} V_{h+1}^{\star}(s_{h+1}^{\ell^{\aux,i}_q}) \right\} \geq P_{h,s,a}V_{h+1}^{\star} \,\middle|\, \mathcal{E}^{\star}(\delta) \right)\\
    & \geq 1 - (1 - \gamma)^{J_q^{\aux}} \geq 1 - \frac{\delta}{4SAH(q+1)^2}.
\end{align*}
By applying a union bound over stages again, we have 
\begin{align*}
    \mathbb{P}\left(\forall t\in[T]:  \max_{j \in [J_q^{\aux}]} \left\{ \sum_{i=0}^{e_q} W_{j,q}^{\aux,i} V_{h+1}^{\star}(s_{h+1}^{\ell^{\aux,i}_q}) \right\} \geq P_{h,s,a}V_{h+1}^{\star}\right) 
    &= \mathbb{P}\left(\forall q\in[q_h^{T+1}(s,a)]:\mathcal{E}_{q,h,s,a} \,\middle|\, \mathcal{E}^{\star}(\delta) \right)\\
    &\ge 1 -  \frac{\delta}{4SAH}\sum_{q=0}^{q_h^{T+1}}\frac{1}{(q+1)^2}\\
    &\ge 1 -  \frac{\delta}{4SAH}\sum_{q=0}^{\infty}\frac{1}{(q+1)^2}\\
    & \ge  1 - \frac{\delta}{2SAH},
\end{align*} 
where the last line uses the fact that $\sum_{q=0}^{\infty} \frac{1}{(q+1)^2} = \frac{\pi^2}{6} < 2$ for any $q\ge 0$.
Finally, union-bounding over all $(h,s,a)\in[H]\times\S\times\A$ and taking expectation on $\mathcal{E}^{\star}(\delta)$ from Lemma \ref{lemma:event_model_error} conclude the proof.
\end{proof}

With Lemma \ref{lemma:optimism_anytime} in hand, the remainder of the analysis is identical to Appendix \ref{appendix:regret_proof}. In particular, Lemma \ref{lemma:concentration_standard_Dir} and \ref{lemma:recursion_circ} still hold with stage-wise parameters $\kappa_q^{\aux},n_{0,q}^{\aux}$ as long as they are fixed during each stage $q = q_h^t(s,a)\ge 0$ for $(t,h,s,a)\in[T]\times[H]\times\S\times\A$.

%% file: appendix/log_regret.tex
\section{Analysis: Gap-dependent Regret Bound} \label{appendix:log_regret}

    We begin by decomposing the total regret using the suboptimality gaps defined in Assumption \ref{assump:positive_gap}. Following \citet{yang2021q}, we obtain:
    \begin{align*}
        \mathrm{Regret}_T = \sum_{t=1}^T (V_1^{\star} - V_1^{\pi^t})(s_1^t) 
        &= \sum_{t=1}^T \left(V_1^{\star}(s_1^t) - Q_1^{\star}(s_1^t,a_1^t) + \left(Q_1^{\star} - Q_1^{\pi^t}\right)(s_1^t,a_1^t)\right)\\
        &= \sum_{t=1}^T\Delta_1(s_1^t,a_1^t) + \sum_{t=1}^T \mathbb{E}_{s_2^t\sim P_{1,s_1^t,a_1^t}} \left[\left(V_2^{\star} - V_2^{\pi^t}\right)(s_2^t) \right]\\
        &= \ldots =  \E\left[ \sum_{t=1}^T\sum_{h=1}^H \Delta_h(s_h^t,a_h^t)\mid a_h^t = \pi_h^t(s_h^t) \right],
    \end{align*}
    where the expectation is taken with respect to the underlying transition kernel.
    Before proceeding, we denote $n_h^t= n_h^t(s_h^t,a_h^t)$ and $q_h^t = q_h^t(s_h^t,a_h^t)$ for notational simplicity. Following the notations used in Appendix \ref{appendix:regret_proof}, we define $\alpha_{m}^{0}=\prod_{k=1}^m \frac{n_0+k-1}{H+n_0+k}$ and $\alpha_m^{i} \defeq \frac{H+1}{H+n_0+i} \prod_{k=i+1}^m \frac{n_0+k-1}{H+n_0+k}$ for any $i\in[m]$ and $m\in\mathbb{N}^{\star}$, and let $\tilde{\xi}_h^{t} = \left(\tildeV_{h}^{t}- V_{h}^{\star}\right)(s_{h}^{t})+H$ and $\tilde{\xi}_h^{\aux,t} = \left(\tildeV_{h}^{\aux,t} - V_{h}^{\star}\right)(s_{h}^{t})$.
    
    We first introduce the following lemma that characterizes the learning error of the Q-functions, which will be used to control the suboptimality gaps. The proof is deferred to Appendix \ref{appendix:lemma_event_learning_error}. 
    \begin{lemma}\label{lemma:event_learning_error}
        Let 
        $\mathcal{E} \defeq\Big\{\forall (t,h,s,a)\in [T]\times[H]\times\S\times\A: 0 \le \left(Q_h^t - Q_h^{\star}\right)(s,a) +\eta_{\ell^{n_h^t},h}H\le 2\alpha_{n_h^t}^0 V_{h+1}^0  + \eta_{\ell^{n_h^t},h} \sum_{i=1}^{n_h^t} \alpha_{{n_h^t}}^i \tilde{\xi}_{h+1}^{\ell^i} + \frac{1-\eta_{\ell^{n_h^t},h}}{1+e_{q_h^t-1}}\sum_{i=0}^{e_{q_h^t-1}}\tilde{\xi}_{h+1}^{\aux,\ell^{\aux,i}_{q_h^t-1}}\mathbbm{1}\{q_h^t\ge 1\} + B_h^t + \eta_{\ell^{n_h^t},h}H \Big\},
        $
        where  $B_h^t \le \linebreak c_B\left(\sqrt{\frac{H^3(\kappa^{\aux})^2 \log^3(SAHT^2)}{{n_h^t}}} + \frac{H^2 \kappa^{\aux} \log^2(SAHT^2)}{{n_h^t}} \right)+ V_{h+1}^0\cdot\mathbbm{1}\{q_h^t=0\}$ for some universal constant $c_B>0$. The event $\mathcal{E}$ holds with probability at least $1 - 1/T$.
    \end{lemma}

In addition, we define the operator $\clip[x|c]\defeq x\cdot \1\{x\ge c\}$ for some constant $c\ge 0$, which is commonly used in prior work \citep{simchowitz2019non,yang2021q,zheng2024gap}. By Lemma \ref{lemma:event_learning_error}, we have $V_h^{\star}(s_h^t) = \max_a Q^{\star}_{h}(s_h^t,a) \le \max_a Q^{t}_{h}(s_h^t,a) + \eta_{\ell^{n_h^t},h}H = Q^{t}_h(s_h^t,a_h^t)  + \eta_{\ell^{n_h^t},h}H $ such that  
\begin{align*}
\Delta_h(s_h^t,a_h^t) &= \clip[V_h^{\star}(s_h^t) - Q_h^{\star}(s_h^t,a_h^t) \mid \Delta_{\min}]\\
&\le \clip[\left(Q_h^t - Q_h^{\star}\right)(s_h^t,a_h^t)  + \eta_{\ell^{n_h^t},h}H \mid \Delta_{\min}].
\end{align*}

Thus, by definition, the expected total regret can be written as %\yc{shouldn't be equal in the last line}
\begin{align}
    \E[\mathrm{Regret}_T] &= \mathbb{P}(\mathcal{E})\cdot \E\left[\sum_{t=1}^T \sum_{h=1}^H \clip[\left(Q_h^t - Q_h^{\star}\right)(s_h^t,a_h^t) +\eta_{\ell^{n_h^t},h}H\mid \Delta_{\min}]\mid \mathcal{E}\right]\nonumber\\
    & + \mathbb{P}(\mathcal{E}^c)\cdot \E\left[\sum_{t=1}^T \sum_{h=1}^H \clip[\left(Q_h^t - Q_h^{\star}\right)(s_h^t,a_h^t)+\eta_{\ell^{n_h^t},h}H\mid \Delta_{\min}]\mid \mathcal{E}^c\right] \nonumber\\
    & \le (1-\frac{1}{T}) \E\left[\sum_{t=1}^T \sum_{h=1}^H \clip[\left(Q_h^t - Q_h^{\star}\right)(s_h^t,a_h^t)+\eta_{\ell^{n_h^t},h}H \mid \Delta_{\min}]\mid \mathcal{E}\right] + \frac{2}{T}\cdot TH^2. \label{eq:expect_total_regret}
\end{align}

Next, we control the first term in \eqref{eq:expect_total_regret} by categorizing the suboptimality gaps into different intervals. Specifically, we split the interval $[\Delta_{\min},H]$ into $N$ disjoint intervals, i.e., $\mathcal{I}_n \defeq [2^{n-1}\Delta_{\min},2^n \Delta_{\min}]$ for $n\in[N-1]$ and $\mathcal{I}_N \defeq[2^{N-1}\Delta_{\min},3H]$, where $N = \lceil\log_2(3H/\Delta_{\min})\rceil$. Denote the counter of state-action pair for each interval as $C_n \defeq \vert \left\{ (t,h): \left(\left(Q_h^t - Q_h^{\star}\right)(s_h^t,a_h^t) + \eta_{\ell^{n_h^t},h} H\right) \in \mathcal{I}_n \right\}\vert$. We then upper bound \eqref{eq:expect_total_regret} as follows:
\begin{align}\label{eq:expect_total_regret_with_counter}
    \E[\mathrm{Regret}_T] \le (1-\frac{1}{T}) \sum_{n=1}^N 2^{n} \Delta_{\min} C_n + 2H^2.
\end{align}

The following lemma shows that the counter is bounded in each interval, conditioned on event $\mathcal{E}$.
\begin{lemma}\label{lemma:bounded_counter}
    Under $\mathcal{E}$, we have that for every $n\in[N]$, $C_n \le O(\frac{H^6SA(\kappa^{\aux})^2 \log^3(SAHT)}{4^n\Delta_{\min}^2})$.
\end{lemma}
The proof is postponed to Appendix \ref{appendix:lemma_bounded_counter}. 
Thus, \eqref{eq:expect_total_regret_with_counter} becomes
\begin{equation*}
    \E[\mathrm{Regret}_T] \le O\left(\frac{H^6SA(\kappa^{\aux})^2 \log^3(SAHT)}{\Delta_{\min}}\right),
\end{equation*}
Noting that $\kappa^{\aux} = O\left(\log(SAHT)\right)$, we complete the proof.

\subsection{Proof of Lemma~\ref{lemma:event_learning_error}}\label{appendix:lemma_event_learning_error}
We begin by applying Lemma~\ref{lemma:optimism} with $\delta = \frac{1}{2T}$, which guarantees that for an ensemble size $J = \lceil c \cdot \log(4SAHT^2)\rceil$, it holds with probability at least $1 - \frac{1}{2T}$ that
\begin{equation*}
    Q_h^t(s,a)\ge (1-\eta_{\ell^{n_h^t},h})Q_h^{\star}(s,a)\ge Q_h^{\star}(s,a) - \eta_{\ell^{n_h^t},h}H,
\end{equation*}
for any $(t,h,s,a)\in [T]\times[H]\times\S\times\A$, where $c$ is a universal positive constant.

Recalling the definition of $Q_h^t$ from \eqref{eq:policy_Q_with_t} and applying  Lemma \ref{lemma:recursion_circ} and Lemma \ref{lemma:recursion_dagger}, we have, again with probability at least $1 - \frac{1}{2T}$,
\begin{align}
    &\left(Q_h^t - Q_h^{\star}\right)(s_h^t,a_h^t) + \eta_{\ell^{n_h^t},h}H\nonumber\\
    &= \eta_{\ell^{n_h^t},h}\max_{j\in [J]} \left\{\widetilde{Q}_{h}^{j,t}(s_{h}^t, a_{h}^t)- Q^{\star}_h (s_h^t,a_h^t)+ H\right\}+ (1-\eta_{\ell^{n_h^t},h}) \left\{\widetilde{Q}_{h}^{\aux,t}(s_{h}^t, a_{h}^t) -  Q^{\star}_h (s_h^t,a_h^t)\right\}\nonumber\\
    &\le 2\alpha_{n_h^t}^0 V_{h+1}^0 +  \eta_{\ell^{n_h^t},h} \sum_{i=1}^{n_h^t} \alpha_{{n_h^t}}^i \tilde{\xi}_{h+1}^{\ell^i} + \frac{1-\eta_{\ell^{n_h^t},h}}{1+e_{q_h^t-1}}\sum_{i=0}^{e_{q_h^t-1}}\tilde{\xi}_{h+1}^{\aux,\ell^{\aux,i}_{q_h^t-1}}\mathbbm{1}\{q_h^t\ge 1\}+ B_h^t.
\end{align}
Here,
\begin{align}
    B_h^t &\le  V_{h+1}^0\cdot\mathbbm{1}\{q_h^t=0\} + c_B \left(\sqrt{\frac{H^3\kappa^2 \log^3(SAHT^2)}{n_h^t}} + \frac{H^2\kappa \log^2(SAHT^2)}{n_h^t} \right. \nonumber\\
    & \left.\quad+ \sqrt{\frac{H^2(\kappa^{\aux})^2\log^3(SAHT^2)}{e_{q_h^t-1}}}\cdot \mathbbm{1}\{q_h^t\ge 1\} + \frac{H\kappa^{\aux} \log^2(SAHT^2)}{e_{q_h^t-1}}\cdot \mathbbm{1}\{q_h^t\ge 1\}\right), \label{eq:B_h}
\end{align}
where $c_B$ is a positive constant. To simplify the expression in \eqref{eq:B_h}, we   note 
 \begin{equation*}
        \frac{n_h^t}{e_{q_h^t-1}} \le \frac{\sum_{i=0}^{q_h^t} e_i}{e_{q_h^t-1}} = 1 + \frac{\sum_{i=0}^{q_h^t-2} e_i}{e_{q_h^t-1}} + \frac{e_{q_h^t}}{e_{q_h^t-1}} \le 2 +\frac{1}{H} + 4H \le 8H,
    \end{equation*}
    where the second inequality uses \citet[Lemma D.3]{zheng2024gap} and the choice of the stage length, \ie, $e_{{q_h^t}} = (1+\frac{1}{H})e_{q_h^t-1}$.

Thus, we could rewrite \eqref{eq:B_h} as
\begin{align}\label{eq:B_h_simplified}
    B_h^t &\le  c_B \left(\sqrt{\frac{H^3(\kappa^{\aux})^2 \log^3(SAHT^2)}{n_h^t}} + \frac{H^2\kappa^{\aux} \log^2(SAHT^2)}{n_h^t} \right) + V_{h+1}^0\cdot\mathbbm{1}\{q_h^t=0\}.
\end{align}

\subsection{Proof of Lemma \ref{lemma:bounded_counter}}\label{appendix:lemma_bounded_counter}

We first partition each interval $\mathcal{I}_n$ according to the step index $h$. Specifically, for every $n \in [N]$ and $h \in [H]$, we define:
\begin{align}
    w^t_{n,h} &\defeq \mathbbm{1} \left\{ \left(\left(Q_h^t - Q_h^{\star}\right)(s_h^t,a_h^t)+ +\eta_{\ell^{n_h^t},h}H\right) \in \mathcal{I}_n \right\}, \quad \forall t \in [T], \label{eq:def_w_n_h_t} \\
    C_{n,h} &\defeq \sum_{t=1}^T w^t_{n,h}. \nonumber 
\end{align}
Note that $w_{n,h}^t \in \{0, 1\}$ for all $t$, since it is an indicator function. 
By definition, we have $C_n = \sum_{h=1}^H C_{n,h}$, and furthermore, for every $n \in [N]$ and $h \in [H]$,
\begin{align}\label{eq:C_n_h_bound}
    2^{n-1}\Delta_{\min} C_{n,h} &\le  \sum_{t=1}^T w_{n,h}^t \left(\left(Q_h^t - Q_h^{\star}\right)(s_h^t,a_h^t) + \eta_{\ell^{n_h^t},h}H\right).
\end{align}

To control the right-hand side of \eqref{eq:C_n_h_bound}, we provide the following lemma, which upper bounds the weighted sum. The proof is postponed to Appendix \ref{appendix:lemma_weighted_sum_error}.
\begin{lemma}\label{lemma:weighted_sum_error}
    Let $\eta_{t,h} = 1/e_{q_h^t}$ for every $(t,h) \in [T] \times [H]$.
    Under event $\mathcal{E}$, for any $h \in [H]$ and $n \in [N]$, the weights $\{w^t_{n,h}\}_{t=1}^T$ defined in \eqref{eq:def_w_n_h_t} satisfy:
    \begin{align*}
        \sum_{t=1}^T w_{n,h}^t \left(\left(Q_h^t - Q_h^{\star}\right)(s_h^t,a_h^t) + \eta_{\ell^{n_h^t},h}H\right)
        &\le  O\left(\sqrt{C_{n,h} SAH^5(\kappa^{\aux})^2 \log^3(SAHT)}\right) \\
        &\quad + O\left(SA H^{3}\kappa^{\aux}\log\!\bigl(SAHT\bigr)  \log({1+C_{n,h}}) \right).
    \end{align*}
\end{lemma}

Combining \eqref{eq:sum_opt_additional_terms} and Lemma \ref{lemma:weighted_sum_error} with \eqref{eq:C_n_h_bound}, we obtain the following bound:
\begin{equation*}
    C_{n,h}\le O\left(\frac{H^5SA(\kappa^{\aux})^2 \log^3(SAHT)}{4^n\Delta_{\min}^2}\right).
\end{equation*}
Summing over $h \in [H]$, we conclude:
\begin{equation*}
    C_n =  \sum_{h=1}^H C_{n,h}\le O\left(\frac{H^6SA(\kappa^{\aux})^2 \log^3(SAHT)}{4^n\Delta_{\min}^2}\right).
\end{equation*}

\subsection{Proof of Lemma \ref{lemma:weighted_sum_error}}\label{appendix:lemma_weighted_sum_error}
Under event $\mathcal{E}$, we recall the following upper bound by Lemma \ref{lemma:event_learning_error}:
\begin{align}\label{eq:learning_error_upper}
    &\left(Q_h^t - Q_h^{\star}\right)(s_h^t,a_h^t) + \eta_{\ell^{n_h^t},h}H\nonumber\\
    &\le 2\alpha_{n_h^t}^0 V_{h+1}^0 +  \eta_{\ell^{n_h^t},h} \sum_{i=1}^{n_h^t} \alpha_{{n_h^t}}^i \tilde{\xi}_{h+1}^{\ell^i} + \underset{T_{h}^t}{\underbrace{\frac{1-\eta_{\ell^{n_h^t},h}}{1+e_{q_h^t-1}}\sum_{i=0}^{e_{q_h^t-1}}\tilde{\xi}_{h+1}^{\aux,\ell^{\aux,i}_{q_h^t-1}}\mathbbm{1}\{q_h^t\ge 1\}}}+ B_h^t +\eta_{\ell^{n_h^t},h}H,
\end{align}
where 
\begin{align*}
    B_h^t \;\le\;&
    c_B\underset{B_{h,1}^t}{\underbrace{%
        \sqrt{\frac{H^{3}\kappa^{2}\log^{3}\!\bigl(8SAHT^{2}\bigr)}{n_h^t}}%
    }}
    \;+\;
    c_B\underset{B_{h,2}^t}{\underbrace{%
        \frac{H^{2}\kappa^{\aux}\log\!\bigl(8SAHT^{2}\bigr)}{n_h^t}%
    }}
    + V_{h+1}^0\cdot\mathbbm{1}\{q_h^t=0\}.
    \end{align*}

Then, our target is to control 
\begin{align}\label{eq:total_error_with_w}
    &\sum_{t=1}^T w_{n,h}^t \left(\left(Q_h^t - Q_h^{\star}\right)(s_h^t,a_h^t) + \eta_{\ell^{n_h^t},h}H\right)\nonumber\\
    & \le 2\sum_{t=1}^T w_{n,h}^t \alpha_{n_h^t}^0V_{h+1}^0 + \sum_{t=1}^T w_{n,h}^t \eta_{\ell^{n_h^t},h} \sum_{i=1}^{n_h^t} \alpha_{{n_h^t}}^i \tilde{\xi}_{h+1}^{\ell^i} + \sum_{t=1}^T w_{n,h}^t T_h^t + \sum_{t=1}^T w_{n,h}^t B_h^t.
\end{align}
 For the first term on the right-hand side of \eqref{eq:total_error_with_w}, we follow the equation \eqref{eq:sum_delta_over_t_1} and obtain:
\begin{equation*}
    2\sum_{t=1}^T w_{n,h}^t \alpha_{n_h^t}^0V_{h+1}^0\le \frac{2n_0SAV_{h+1}^0 }{H-1}\le 4n_0SA.
\end{equation*}
For the second term on the right-hand side of \eqref{eq:total_error_with_w}, we follow the equation \eqref{eq:sum_delta_over_t_2} and obtain:
\begin{equation} \label{eq:sum_opt_additional_terms}
    \sum_{t=1}^T w_{n,h}^t \eta_{\ell^{n_h^t},h} \sum_{i=1}^{n_h^t} \alpha_{{n_h^t}}^i \tilde{\xi}_{h+1}^{\ell^i}  \lesssim H\sum_{t=1}^{T} \eta_{\ell^{n_h^{t}},h} \lesssim H \sum_{s,a}\sum_{q=1}^Q \frac{e_{q_h^t}}{e_{q_h^t}} \lesssim SAH^2\log(T),
\end{equation}
where the last inequality uses the fact $Q\le 4H\log(T)$.
For the third term  on the right-hand side of \eqref{eq:total_error_with_w}, we apply the similar arguments in \eqref{eq:sum_delta_over_t_3} and Appendix \ref{appendix:proof_eq_sum_delta_over_t_3},
Then, we have
\begin{align*}
     \sum_{t=1}^T w_{n,h}^t T_h^t &\le \sum_{t=1}^T w_{n,h}^t  (1-\eta_{\ell^{n_h^t},h})\sum_{i=0}^{e_{q_h^t-1}}\frac{\mathbbm{1}\{q_h^t\ge 1\}}{1+e_{q_h^t-1}}\tilde{\xi}_{h+1}^{\aux,\ell^{\aux,i}_{q_h^t-1}}\\
     &\le  (1+\frac{1}{H})\sum_{t'=1}^T (1-\eta_{t',h}) \tilde{\xi}_{h+1}^{\aux,t'} \cdot \sum_{t=1}^T \frac{w_{n,h}^t}{1+e_{q_h^t-1}} \sum_{i=0}^{e_{q_h^t-1}} \1\{\ell^{\aux,i}_{q_h^t-1} = t'\}\\
    &=  (1+\frac{1}{H})\sum_{t'=1}^T (1-\eta_{t',h}) \tilde{\xi}_{h+1}^{\aux,t'} \cdot \bar{w}_1^{t'},
    %  &\le  \sum_{t=1}^T \bar{w}_{1}^t \cdot \left({V}_{h+1}^t - V_{h+1}^{\star}\right)(s_{h+1}^t),
\end{align*}
where $ \bar{w}_1^{t'} = \sum_{t=1}^T \frac{w_{n,h}^t}{1+e_{q_h^t-1}} \sum_{i=0}^{e_{q_h^t-1}} \1\{\ell^{\aux,i}_{q_h^t-1} = t'\}$ for every $t'\in[T]$. Note that $\{\bar{w}_{1}^t\}$ is some sequence satisfying  $0\le\bar{w}^t_{1}\le (1+\frac{1}{H}),\forall t\in[T]$ and $\sum_{t=1}^T \bar{w}^t_{1} = C_{n,h}$, following the similar arguments in Lemma 4.3 in \cite{yang2021q}. From \eqref{eq:xi_combine} and \eqref{eq:sum_opt_additional_terms}, we further have
\begin{align*}
     \sum_{t=1}^T w_{n,h}^t T_h^t &\lesssim (1+\frac{1}{H}) \sum_{t=1}^T \left(\bar{w}_1^t\xi_{h+1}^t+ \eta_{\ell^{n_{h+1}^t},{h+1}}H\right)+ SAH^2\log(T).
\end{align*}
Note that $V_{h+1}^{\star}(s_{h+1}^t) \ge Q_{h+1}^{\star}(s_{h+1}^t, a_{h+1}^t)$ by the definition of the optimal policy in \eqref{eq:value_func_optimal}, and $V_{h+1}^t(s_{h+1}^t) = Q_{h+1}^t(s_{h+1}^t, a_{h+1}^t)$ by \eqref{eq:policy_V_function}. Thus, 
\begin{align*}
     \sum_{t=1}^T w_{n,h}^t T_h^t  \lesssim (1+\frac{1}{H})  \sum_{t=1}^T \bar{w}_{1}^t \left(\left({Q}_{h+1}^t - Q_{h+1}^{\star}\right)(s_{h+1}^t,a_{h+1}^t) + \eta_{\ell^{n_{h+1}^t},{h+1}} H\right) + SAH^2\log(T).
\end{align*}

We then consider the bound of $\sum_{t=1}^T \bar{w}^t_{i} B_h^t$ for any sequence $\{\bar{w}^t_{i}\}$ which satisfies $0\le\bar{w}^t_{i}\le (1+\frac{1}{H})^i,\forall t\in[T]$ and $\sum_{t=1}^T \bar{w}^t_{i} = C_{n,h}$ for any $i\in\{0,\ldots, H-h\}$. Before proceeding, we first define $C_{n,h,s,a} \defeq \sum_{m=1}^{n_h^T(s,a)} \bar{w}_{i}^{\ell^m}$.
Similar to the steps in Appendix \ref{appendix:regret_proof}, we bound $\sum_{t=1}^T \bar{w}^t_{i} B_h^t$ by controlling each term separately.  To begin with,
\begin{align}
    \sum_{t=1}^T \bar{w}^t_{i} B_{h,1}^t &\le \sqrt{H^3(\kappa^{\aux})^2 \log^3(8SAHT^2)} \sum_{t=1}^T \frac{ \bar{w}^t_{i}}{\sqrt{n_h^t}}\nonumber\\
    &\lesssim \sqrt{H^3(\kappa^{\aux})^2 \log^3(SAHT^2)} \sum_{(s,a)\in\S\times\A} \sum_{m=1}^{n^T_h(s,a)}\frac{\bar{w}_{i}^{\ell^m}}{\sqrt{m}}\nonumber\\
    &\lesssim \sqrt{H^3(\kappa^{\aux})^2 \log^3(SAHT^2)}  \sum_{(s,a)\in\S\times\A} \sum_{m=1}^{ \lceil C_{n,h,s,a} /(1+1/H)^i\rceil }\frac{(1+1/H)^i}{\sqrt{m}}  \nonumber\\
    &\lesssim \sqrt{ H^3(\kappa^{\aux})^2 \log^3(SAHT^2)} \sqrt{SA\sum_{(s,a)\in\S\times\A} C_{n,h,s,a} \cdot (1+1/H)^i} \nonumber\\
    &\lesssim \sqrt{ C_{n,h} SAH^3(\kappa^{\aux})^2 (1+1/H)^i\log^3(SAHT^2)},\label{eq:bound_B_1}
\end{align}
where the third inequality holds since the left-hand side is maximized when the first $\lceil C_{n,h,s,a} /(1+1/H)^i\rceil$ occupy the smallest indices, the penultimate line uses Cauchy-Schwarz inequality, and the last line is due to the fact that $\sum_{s,a}C_{n,h,s,a} \le C_{n,h}$.

Similarly,
\begin{align}
    \sum_{t=1}^T  \bar{w}^t_{i} B_{h,2}^t &\le \sum_{t=1}^T  \bar{w}^t_{i} \frac{H^{2}\kappa^{\aux}\log\!\bigl(8SAHT^{2}\bigr)}{n_h^t}\nonumber\\
    &\lesssim  H^{2}\kappa^{\aux}\log\!\bigl(SAHT^{2}\bigr) \sum_{t=1}^T \frac{\bar{w}^t_{i}}{n_h^t}\nonumber\\
    &\lesssim H^{2}\kappa^{\aux}\log\!\bigl(SAHT^{2}\bigr) \sum_{(s,a)\in\S\times\A} \sum_{m=1}^{ \lceil C_{n,h,s,a} /(1+1/H)^i\rceil}\frac{(1+1/H)^i}{m}\nonumber\\
    &\lesssim H^{2}\kappa^{\aux}\log\!\bigl(SAHT^{2}\bigr)(1+1/H)^i\sum_{(s,a)\in\S\times\A} \log(1+C_{n,h,s,a})\nonumber\\
    &\lesssim SA H^{2}\kappa^{\aux}\log\!\bigl(SAHT^{2}\bigr) (1+1/H)^i \log(1+C_{n,h}) .\label{eq:bound_B_2}
\end{align}
Also, $\sum_{t=1}^T \bar{w}^t_{i}  V_{h+1}^0\cdot\mathbbm{1}\{q_h^t=0\} \le SAH^2 (1+1/H)^i$, as there are at most $SA$ state-action paris and each can be visited at most $e_0 = H$ times in stage 0. Combining with \eqref{eq:bound_B_1} and \eqref{eq:bound_B_2}, one has 
\begin{align}
     \sum_{t=1}^T \bar{w}^t_{i} B_{h}^t \lesssim &\sqrt{ C_{n,h} SAH^3(\kappa^{\aux})^2 (1+1/H)^i\log^3(SAHT^2)} \nonumber\\
     &+ SA H^{2}\kappa^{\aux}\log\!\bigl(SAHT^{2}\bigr) (1+1/H)^i\log(1+C_{n,h}) + SAH^2(1+1/H)^i. \label{eq:weighted_B_i}
\end{align}

Observe that $\{w_{n,h}^t\}$ directly satisfies $0\le w_{n,h}^t \le 1$ and $\sum_{t=1}^T w_{n,h}^t = C_{n,h}$. Thus, we can apply \eqref{eq:weighted_B_i} with $i=0$  and obtain
\begin{align*}
    &\sum_{t=1}^T w_{n,h}^t  \left(\left(Q_h^t - Q_h^{\star}\right)(s_h^t,a_h^t) + \eta_{\ell^{n_h^t},h}H\right) \\
    &\lesssim  \sqrt{ C_{n,h} SAH^3(\kappa^{\aux})^2 \log^3(SAHT^2)}+ SA H^{2}\kappa^{\aux}\log\!\bigl(SAHT^{2}\bigr) \log(1+C_{n,h}) \\
    & \quad+ (1+\frac{1}{H})\sum_{t=1}^T \bar{w}^t_{1} \left(\left({Q}_{h+1}^t - Q_{h+1}^{\star}\right)(s_{h+1}^t,a_{h+1}^t) + \eta_{\ell^{n_{h+1}^t},{h+1}} H\right).
\end{align*}

We now recursively unroll the Q-value difference over future steps and utilize the fact that $(1+\frac{1}{H})^H\le \mathrm e$, yielding:
\begin{align*}
    &\sum_{t=1}^T w_{n,h}^t  \left(\left(Q_h^t - Q_h^{\star}\right)(s_h^t,a_h^t) + \eta_{\ell^{n_h^t},h}H\right) \\
    &\lesssim \sum_{i=0}^{H-h} \left((1+\frac{1}{H})^{2i} \sqrt{ C_{n,h} SAH^3(\kappa^{\aux})^2 \log^3(SAHT^2)} \right.\\
    &\quad \quad\quad\left.+ SA H^{2}\kappa^{\aux}\log\!\bigl(SAHT^{2}\bigr) (1+\frac{1}{H})^{2i} \log(1+C_{n,h}) \right)\\
    &\lesssim H\left(\sqrt{ C_{n,h} SAH^3(\kappa^{\aux})^2\log^3(SAHT^2)}+ SA H^{2}\kappa^{\aux}\log\!\bigl(SAHT^{2}\bigr) \log(1+C_{n,h})\right),
\end{align*}
which completes the proof.

%% file: bibfileRL.bib
@article{lai1985asymptotically,
  title={Asymptotically efficient adaptive allocation rules},
  author={Lai, Tze Leung and Robbins, Herbert},
  journal={Advances in Applied Mathematics},
  volume={6},
  number={1},
  pages={4--22},
  year={1985},
  publisher={Academic Press}
}

@article{li2021breaking,
  title={Breaking the sample complexity barrier to regret-optimal model-free reinforcement learning},
  author={Li, Gen and Shi, Laixi and Chen, Yuxin and Gu, Yuantao and Chi, Yuejie},
  journal={Advances in Neural Information Processing Systems},
  volume={34},
  pages={17762--17776},
  year={2021}
}

@inproceedings{yang2021q,
  title={Q-learning with Logarithmic Regret},
  author={Yang, Kunhe and Yang, Lin and Du, Simon},
  booktitle={International Conference on Artificial Intelligence and Statistics},
  pages={1576--1584},
  year={2021},
  organization={PMLR}
}

@article{bai2019provably,
  title={Provably efficient {Q}-learning with low switching cost},
  author={Bai, Yu and Xie, Tengyang and Jiang, Nan and Wang, Yu-Xiang},
  journal={arXiv preprint arXiv:1905.12849},
  year={2019}
}

@inproceedings{strehl2006pac,
  title={{PAC} model-free reinforcement learning},
  author={Strehl, Alexander L and Li, Lihong and Wiewiora, Eric and Langford, John and Littman, Michael L},
  booktitle={Proceedings of the 23rd international conference on Machine learning},
  pages={881--888},
  year={2006}
}

@article{zhang2020almost,
  title={Almost Optimal Model-Free Reinforcement Learning via Reference-Advantage Decomposition},
  author={Zhang, Zihan and Zhou, Yuan and Ji, Xiangyang},
  journal={Advances in Neural Information Processing Systems},
  volume={33},
  year={2020}
}

@article{sutton1988learning,
  title={Learning to predict by the methods of temporal differences},
  author={Sutton, Richard S},
  journal={Machine learning},
  volume={3},
  number={1},
  pages={9--44},
  year={1988},
  publisher={Springer}
}

@article{mnih2015human,
  title={Human-level control through deep reinforcement learning},
  author={Mnih, Volodymyr and Kavukcuoglu, Koray and Silver, David and Rusu, Andrei A and Veness, Joel and Bellemare, Marc G and Graves, Alex and Riedmiller, Martin and Fidjeland, Andreas K and Ostrovski, Georg},
  journal={Nature},
  volume={518},
  number={7540},
  pages={529--533},
  year={2015},
  publisher={Nature Publishing Group}
}

@inproceedings{jin2018q,
  title={Is {Q}-learning provably efficient?},
  author={Jin, Chi and Allen-Zhu, Zeyuan and Bubeck, Sebastien and Jordan, Michael I},
  booktitle={Advances in Neural Information Processing Systems},
  pages={4863--4873},
  year={2018}
}

@book{puterman2014markov,
  title={Markov decision processes: discrete stochastic dynamic programming},
  author={Puterman, Martin L},
  year={2014},
  publisher={John Wiley \& Sons}
}

@article{watkins1989learning,
  title={Learning from delayed rewards},
  author={Watkins, C. J. C. H.},
  journal={PhD thesis, King's College, University of Cambridge},
  year={1989}
}

@inproceedings{domingues2021episodic,
  title={Episodic reinforcement learning in finite {MDP}s: Minimax lower bounds revisited},
  author={Domingues, Omar Darwiche and M{\'e}nard, Pierre and Kaufmann, Emilie and Valko, Michal},
  booktitle={Algorithmic Learning Theory},
  pages={578--598},
  year={2021},
  organization={PMLR}
}


%% file: references.bib
@inproceedings{tiapkin2024model,
title={Model-free Posterior Sampling via Learning Rate Randomization},
  author={Tiapkin, Daniil and Belomestny, Denis and Calandriello, Daniele and Moulines, Eric and Munos, Remi and Naumov, Alexey and Perrault, Pierre and Valko, Michal and M{\'e}nard, Pierre},
booktitle={Thirty-seventh Conference on Neural Information Processing Systems},
year={2023},
}

@article{tiapkin2022optimistic,
  title={Optimistic posterior sampling for reinforcement learning with few samples and tight guarantees},
  author={Tiapkin, Daniil and Belomestny, Denis and Calandriello, Daniele and Moulines, {\'E}ric and Munos, Remi and Naumov, Alexey and Rowland, Mark and Valko, Michal and M{\'e}nard, Pierre},
  journal={Advances in Neural Information Processing Systems},
  volume={35},
  pages={10737--10751},
  year={2022}
}

@article{wong1998generalized,
  title={Generalized Dirichlet distribution in Bayesian analysis},
  author={Wong, Tzu-Tsung},
  journal={Applied Mathematics and Computation},
  volume={97},
  number={2-3},
  pages={165--181},
  year={1998},
  publisher={Elsevier}
}

@article{pinelis1994optimum,
  title={Optimum bounds for the distributions of martingales in Banach spaces},
  author={Pinelis, Iosif},
  journal={The Annals of Probability},
  pages={1679--1706},
  year={1994},
  publisher={JSTOR}
}

@article{dann2021provably,
  title={A provably efficient model-free posterior sampling method for episodic reinforcement learning},
  author={Dann, Christoph and Mohri, Mehryar and Zhang, Tong and Zimmert, Julian},
  journal={Advances in Neural Information Processing Systems},
  volume={34},
  pages={12040--12051},
  year={2021}
}

@inproceedings{hong2024q,
title={Q-{SFT}: Q-Learning for Language Models via Supervised Fine-Tuning},
author={Joey Hong and Anca Dragan and Sergey Levine},
booktitle={The Thirteenth International Conference on Learning Representations},
year={2025},
url={https://openreview.net/forum?id=v4MTnPiYXY}
}

@article{osband2013more,
  title={({M}ore) efficient reinforcement learning via posterior sampling},
  author={Osband, Ian and Russo, Daniel and Van Roy, Benjamin},
  journal={Advances in Neural Information Processing Systems},
  volume={26},
  year={2013}
}

@article{osband2016deep,
  title={Deep exploration via bootstrapped DQN},
  author={Osband, Ian and Blundell, Charles and Pritzel, Alexander and Van Roy, Benjamin},
  journal={Advances in Neural Information Processing Systems},
  volume={29},
  year={2016}
}

@inproceedings{strens2000bayesian,
  title={A Bayesian framework for reinforcement learning},
  author={Strens, Malcolm},
  booktitle={International Conference on Machine Learning},
  pages={943--950},
  year={2000}
}

@article{agrawal2017optimistic,
 author = {Agrawal, Shipra and Jia, Randy},
 journal = {Advances in Neural Information Processing Systems},
 title = {Optimistic posterior sampling for reinforcement learning: worst-case regret bounds},
 volume = {30},
 year = {2017}
}

@inproceedings{osband2016generalization,
  title={Generalization and exploration via randomized value functions},
  author={Osband, Ian and Van Roy, Benjamin and Wen, Zheng},
  booktitle={International Conference on Machine Learning},
  pages={2377--2386},
  year={2016},
  organization={PMLR}
}

@article{mnih2013playing,
  title={Playing atari with deep reinforcement learning},
  author={Mnih, Volodymyr and Kavukcuoglu, Koray and Silver, David and Graves, Alex and Antonoglou, Ioannis and Wierstra, Daan and Riedmiller, Martin},
  journal={arXiv preprint arXiv:1312.5602},
  year={2013}
}

@inproceedings{fortunato2017noisy,
title={Noisy Networks For Exploration},
author={Meire Fortunato and Mohammad Gheshlaghi Azar and Bilal Piot and Jacob Menick and Matteo Hessel and Ian Osband and Alex Graves and Volodymyr Mnih and Remi Munos and Demis Hassabis and Olivier Pietquin and Charles Blundell and Shane Legg},
booktitle={International Conference on Learning Representations},
year={2018},
}

@inproceedings{agrawal2021improved,
  title={Improved worst-case regret bounds for randomized least-squares value iteration},
  author={Agrawal, Priyank and Chen, Jinglin and Jiang, Nan},
  booktitle={Proceedings of the AAAI Conference on Artificial Intelligence},
  volume={35},
  number={8},
  pages={6566--6573},
  year={2021}
}

@inproceedings{zanette2020frequentist,
  title={Frequentist regret bounds for randomized least-squares value iteration},
  author={Zanette, Andrea and Brandfonbrener, David and Brunskill, Emma and Pirotta, Matteo and Lazaric, Alessandro},
  booktitle={International Conference on Artificial Intelligence and Statistics},
  pages={1954--1964},
  year={2020},
  organization={PMLR}
}

@inproceedings{
zheng2024gap,
title={Gap-Dependent Bounds for Q-Learning using Reference-Advantage Decomposition},
author={Zhong Zheng and Haochen Zhang and Lingzhou Xue},
booktitle={The Thirteenth International Conference on Learning Representations},
year={2025},
}

@article{simchowitz2019non,
  title={Non-asymptotic gap-dependent regret bounds for tabular mdps},
  author={Simchowitz, Max and Jamieson, Kevin G},
  journal={Advances in Neural Information Processing Systems},
  volume={32},
  year={2019}
}

@article{zhang2022feel,
  title={Feel-good thompson sampling for contextual bandits and reinforcement learning},
  author={Zhang, Tong},
  journal={SIAM Journal on Mathematics of Data Science},
  volume={4},
  number={2},
  pages={834--857},
  year={2022},
  publisher={SIAM}
}

@article{moradipari2023improved,
  title={Improved Bayesian regret bounds for Thompson sampling in reinforcement learning},
  author={Moradipari, Ahmadreza and Pedramfar, Mohammad and Shokrian Zini, Modjtaba and Aggarwal, Vaneet},
  journal={Advances in Neural Information Processing Systems},
  volume={36},
  pages={23557--23569},
  year={2023}
}

@article{hao2022regret,
  title={Regret bounds for information-directed reinforcement learning},
  author={Hao, Botao and Lattimore, Tor},
  journal={Advances in Neural Information Processing Systems},
  volume={35},
  pages={28575--28587},
  year={2022}
}

@article{song2019efficient,
  title={Efficient model-free reinforcement learning in metric spaces},
  author={Song, Zhao and Sun, Wen},
  journal={arXiv preprint arXiv:1905.00475},
  year={2019}
}

@inproceedings{melo2007q,
  title={Q-learning with linear function approximation},
  author={Melo, Francisco S and Ribeiro, M Isabel},
  booktitle={International Conference on Computational Learning Theory},
  pages={308--322},
  year={2007},
  organization={Springer}
}

@misc{rlberry,
    author = {Domingues, Omar Darwiche and Flet-Berliac, Yannis and Leurent, Edouard and M{\'e}nard, Pierre and Shang, Xuedong and Valko, Michal},
    doi = {10.5281/zenodo.5544540},
    month = {10},
    title = {{rlberry - A Reinforcement Learning Library for Research and Education}},
    url = {https://github.com/rlberry-py/rlberry},
    year = {2021}
}

@book{gupta2004handbook,
  title={Handbook of beta distribution and its applications},
  author={Gupta, Arjun K and Nadarajah, Saralees},
  year={2004},
  publisher={CRC press}
}

@inproceedings{wei2020model,
  title={Model-free reinforcement learning in infinite-horizon average-reward markov decision processes},
  author={Wei, Chen-Yu and Jahromi, Mehdi Jafarnia and Luo, Haipeng and Sharma, Hiteshi and Jain, Rahul},
  booktitle={International conference on machine learning},
  pages={10170--10180},
  year={2020},
  organization={PMLR}
}

@inproceedings{tang2025deep,
  title={Deep reinforcement learning for robotics: A survey of real-world successes},
  author={Tang, Chen and Abbatematteo, Ben and Hu, Jiaheng and Chandra, Rohan and Mart{\'\i}n-Mart{\'\i}n, Roberto and Stone, Peter},
  booktitle={Proceedings of the AAAI Conference on Artificial Intelligence},
  volume={39},
  number={27},
  pages={28694--28698},
  year={2025}
}

@article{osband2018randomized,
  title={Randomized prior functions for deep reinforcement learning},
  author={Osband, Ian and Aslanides, John and Cassirer, Albin},
  journal={Advances in neural information processing systems},
  volume={31},
  year={2018}
}

@article{ishfaq2024more,
  title={More efficient randomized exploration for reinforcement learning via approximate sampling},
  author={Ishfaq, Haque and Tan, Yixin and Yang, Yu and Lan, Qingfeng and Lu, Jianfeng and Mahmood, A Rupam and Precup, Doina and Xu, Pan},
  journal={arXiv preprint arXiv:2406.12241},
  year={2024}
}
